%% file: actor_critic_arxiv.tex
\documentclass[english]{article}
\usepackage{geometry}
\usepackage{graphicx} 
\usepackage[skins,theorems]{tcolorbox}
\tcbset{highlight math style={enhanced,colframe=red,colback=white}}

\usepackage[unicode=true,
 bookmarks=false,
 breaklinks=false,pdfborder={0 0 1},colorlinks=false]
 {hyperref}

\hypersetup{colorlinks,citecolor=blue,filecolor=blue,linkcolor=blue,urlcolor=blue}

\usepackage{styles/algorithm}
\usepackage{styles/algorithmic}

\usepackage{microtype,float}
\usepackage{graphicx}
\usepackage{subfigure}
\usepackage{booktabs} 

\usepackage{authblk}
\usepackage{graphics}
\usepackage{natbib}
\usepackage{amsmath, amssymb,float,booktabs}

\usepackage{enumitem}

\input{macros}

\definecolor{ballblue}{rgb}{0.13, 0.67, 0.8}
\definecolor{darkred}{RGB}{139,0,0}
\usepackage{soul}

\allowdisplaybreaks

\usepackage[textsize=tiny]{todonotes}

\theoremstyle{plain}

\title{Actor-Critics Can Achieve Optimal Sample Efficiency}

\author{Kevin Tan$^*$}
\author{Wei Fan\footnote{Equal contribution.}}
\author{Yuting Wei}
\affil{
  Department of Statistics and Data Science\\
  The Wharton School, University of Pennsylvania}

\begin{document}

\maketitle

\begin{abstract}
    Actor-critic algorithms have become a cornerstone in reinforcement learning (RL), leveraging the strengths of both policy-based and value-based methods. 
    Despite recent progress in understanding their statistical efficiency, no existing work has successfully learned an $\epsilon$-optimal policy with a sample complexity of $O(1/\epsilon^2)$ trajectories with general function approximation when strategic exploration is necessary. 
    We address this open problem by introducing a novel actor-critic algorithm that attains a sample-complexity of $O(dH^5 \log|\mathcal{A}|/\epsilon^2 + d H^4 \log|\mathcal{F}|/ \epsilon^2)$ trajectories, and accompanying $\sqrt{T}$ regret when the Bellman eluder dimension $d$ does not increase with $T$ at more than a $\log T$ rate. 
    Here, $\mathcal{F}$ is the critic function class, $\mathcal{A}$ is the action space, and $H$ is the horizon in the finite horizon MDP setting. Our algorithm integrates optimism, off-policy critic estimation targeting the optimal Q-function, and rare-switching policy resets. 
    We extend this to the setting of Hybrid RL, showing that initializing the critic with offline data yields sample efficiency gains compared to purely offline or online RL. Further, utilizing access to offline data, we provide a \textit{non-optimistic} provably efficient actor-critic algorithm that only additionally requires $N_{\off} \geq c_{\off}^*dH^4/\epsilon^2$ in exchange for omitting optimism, where $c_{\off}^*$ is the single-policy concentrability coefficient and $N_{\off}$ is the number of offline samples. This addresses another open problem in the literature. We further provide numerical experiments to support our theoretical findings.
\end{abstract}


\setcounter{tocdepth}{2}
\tableofcontents

\section{Introduction}

Actor-critic algorithms have emerged as a foundational approach in reinforcement learning (RL), mitigating the deficiencies of both policy-based and value-based approaches \citep{sutton1998, mnih2016asynchronousmethodsdeepreinforcement,haarnoja2018softactorcriticoffpolicymaximum}. 
These methods combine two components: an actor, which directly learns and improves the policy, and a critic, which evaluates the policy's quality. 
Given their popularity, significant recent progress has been made in understanding their theoretical underpinnings and statistical efficiency, especially in the presence of function approximation \citep{cai2024provablyefficientexplorationpolicy, zhong2023theoreticalanalysisoptimisticproximal, sherman2024rateoptimalpolicyoptimizationlinear, liu2023optimisticnaturalpolicygradient} -- which is required in real-world applications with prohibitively large state and action spaces.

However, much existing work \citep{abbasiyadkori2019politex,
neu2017unifiedviewentropyregularizedmarkov,
liu2023neuralproximaltrustregionpolicy,
bhandari2022globaloptimalityguaranteespolicy,
agarwal2021policygradient,cen2022fast, gaur2024closinggapachievingglobal} on the convergence of actor-critic methods requires assumptions on the reachability of the state-action space or on the coverage of the sampled data. \citet{liu2023optimisticnaturalpolicygradient} remark that this implies that the state-space is either well-explored or easy to explore. This allows the agent to bypass the need to actively explore the state-action space, making learning significantly easier.\footnote{When coverage/reachability assumptions are made, the linear convergence of policy-based methods \citep{lan2022policymirrordescentreinforcement, xiao2022convergenceratespolicygradient, yuan2023linearconvergencenaturalpolicy, chen2022offpolicylinearpolicy} and the gradient domination lemma \citep{kumar2024improvedsamplecomplexityglobal, mei2022globalconvergenceratessoftmax} enable the natural actor-critic algorithm to learn an $\epsilon$-optimal policy within $1/\epsilon^2$ samples if one can access the exact NPG update, although vanilla policy gradient methods can take (super) exponential time to converge \cite{li2021softmax}.} Therefore, these approaches analyze actor-critic methods from an optimization perspective and do not address the problem of exploration \citep{efroni2020optimisticpolicyoptimizationbandit} -- a salient problem that we seek to tackle, hence the need for strategic exploration. 

Without reachability assumptions, policy gradient methods struggle due to a lack of strategic exploration.\footnote{To illustrate, this was not solved until \citet{agarwal2020pcpgpolicycoverdirected}, who required $1/\epsilon^{11}$ samples to do so in linear MDPs.} A recent line of work utilizes optimism to address this. \citet{efroni2020optimisticpolicyoptimizationbandit, wu2022nearlyoptimalpolicyoptimization} and \citet{cai2024provablyefficientexplorationpolicy} achieve $\sqrt{T}$ regret within the settings of tabular and linear mixture MDPs respectively, with \citet{wu2022nearlyoptimalpolicyoptimization} attaining the minimax-optimal rate.  
Still, \citet{zhong2023theoreticalanalysisoptimisticproximal, liu2023optimisticnaturalpolicygradient} remark that these analyses do not generalize to more general MDPs due to the need to cover an exponentially growing policy class. Optimism is not the only way to solve this issue. For instance, \citet{huang2024sublinearregretclasscontinuoustime} utilize Gaussian perturbations to perform exploration in the linear-quadratic control problem. However, naive, direction-unaware perturbation-based exploration comes with a cost, achieving only $T^{3/4}$ regret in this case.

Within linear MDPs, \citet{sherman2024rateoptimalpolicyoptimizationlinear} and \citet{cassel2024warmupfreepolicyoptimization} have very recently been able to obtain the optimal rate of $\sqrt{T}$ regret or $1/\epsilon^2$ sample complexity. They do so via methodological advancements (specific to linear MDPs) that let them overcome the growing policy class issue.
However, the problem is unresolved with general function approximation -- the best known algorithm from \citet{liu2023optimisticnaturalpolicygradient} requires at least $1/\epsilon^3$ samples, increasing to $1/\epsilon^4$ when the policy class grows exponentially.


\paragraph{An open problem.} \textit{No actor-critic algorithm with general function approximation is currently known to achieve $1/\epsilon^2$ sample complexity or $\sqrt{T}$ regret} in this more challenging setting where strategic exploration is necessary. \citet{zhong2023theoreticalanalysisoptimisticproximal} and \citet{liu2023optimisticnaturalpolicygradient} remark that  a way forward to achieve the desired $1/\epsilon^2$ sample complexity remains unclear, and raise the open problem:

\begin{center}
    \textit{Can actor-critic or policy optimization algorithms achieve $1/\epsilon^2$ sample complexity or $\sqrt{T}$ regret with general function approximation and when strategic exploration is necessary?}
\end{center}

\begin{table*}[t!]
    \centering
    \scalebox{0.69}{
    \begin{tabular}{c|c|c|c}
    \toprule
        Algorithm & Sample Complexity & Regret & Remarks \tabularnewline
        \toprule
        
        \citet{agarwal2020pcpgpolicycoverdirected} &  $ d^3H^{15}\log|\gA|/\epsilon^{11}$ &
        None &
        \tabularnewline
        \citet{zanette2021cautiouslyoptimisticpolicyoptimization} & $H^4\log|\gA|/\epsilon^2 + d^3H^{13}\log|\gA|/\epsilon^3$ &
        $\sqrt{H^4\log|\gA|\log|\gA|T} + \sqrt{d^3H^{13}T}$& \centering \vline
        \tabularnewline
        \citet{zhong2023theoreticalanalysisoptimisticproximal} & $d^3H^8\log|\gA|/\epsilon^4 + d^5H^4/\epsilon^2$ &
        $\sqrt{d^3H^8\log|\gA|T} + \sqrt{d^5H^4T}$&Linear MDPs only
        \tabularnewline
        \citet{sherman2024rateoptimalpolicyoptimizationlinear}& $d^4H^7\log|\gA|/\epsilon^2$ & $\sqrt{d^4H^7\log|\gA|T}$ & \centering \vline
        \tabularnewline
        \citet{cassel2024warmupfreepolicyoptimization} & 
        $dH^5\log|\gA|/\epsilon^2 + d^3H^4/\epsilon^2$
        & $\sqrt{dH^5\log|\gA|T} + \sqrt{d^3H^4T}$
        & 
        \tabularnewline
        \toprule
        \citet{zhou2023offlinedataenhancedonpolicy} & $(\log|\gA|+C_{\text{npg}}^3\wedge C_{\off}^6)\log|\gF|H^{14}/\epsilon^6$& Linear & Requires offline data
        \tabularnewline
        \citet{liu2023optimisticnaturalpolicygradient} & $d\log|\gA|\log|\gF|H^6/\epsilon^3$& None & ``Good'' case only, $1/\epsilon^4$ generally 
        
        \tabularnewline
        DOUHUA (Algorithm \ref{alg:DOUHUA})& $\tcbhighmath[colframe=ballblue]{H^4\log|\gA|/{\epsilon^2} + dH^4\log|\gF|/\epsilon^2}$ & $\tcbhighmath[colframe=ballblue]{\sqrt{H^4\log|\gA|T} + \sqrt{dH^4\log|\gF|T}}$ & ``Good'' case only, vacuous generally
        
        \tabularnewline
        NORA (Algorithm \ref{alg:NORA}) & $\tcbhighmath[colframe=red]{dH^5\log|\gA|/{\epsilon^2} + dH^4\log|\gF|/{\epsilon^2}}$ & $\tcbhighmath[colframe=red]{\sqrt{dH^5\log|\gA|T} + \sqrt{dH^4\log|\gF|T}}$ & 
        Holds generally
        \tabularnewline
        \bottomrule
    \end{tabular}
    }
    \caption{Comparison of our work to existing literature. Algorithm \ref{alg:NORA} achieves the best known bound for actor-critic methods with general function approximation, and resolves an open problem on whether $\sqrt{T}$ regret or $1/\epsilon^2$ sample complexity is achievable in this scenario.}
    \label{tab:bounds-ac}
\end{table*}

\subsection{This paper}
We resolve this open problem in the affirmative. As a warm-up, we first consider an easy case -- where the complexity of the class of policies considered by the policy optimization procedure does not increase exponentially with the number of (critic) updates.\footnote{It was previously considered in \citet{zhong2023theoreticalanalysisoptimisticproximal} that the log-covering number of the policy class increases in the number of actor and critic updates.  
We sharpen this bound to the number of critic updates in Lemma \ref{lem:b2-covering-number-policy}, which may be of independent interest.} Then, a simple modification to the GOLF algorithm of \citet{jin2021bellman} allows one to achieve a regret of $\sqrt{H^4\log|\gA|T} + \sqrt{dH^4 \log|\gF| \;T}$,
in line with the results of \citep{efroni2020optimisticpolicyoptimizationbandit, cai2024provablyefficientexplorationpolicy} for tabular and linear mixture MDPs respectively.

However, this is not the case in many practical scenarios -- for example, where one uses linear models, decision trees, neural networks, or even random forests for the critic. In this much harder setting, Algorithm \ref{alg:DOUHUA} cannot achieve sublinear regret. We address this by introducing an algorithm, NORA (Algorithm \ref{alg:NORA}), which leverages three crucial ingredients: (1) optimism, (2) off-policy learning, and (3) rare-switching critic updates that target $Q^*$ and accompanying policy resets. Algorithm \ref{alg:NORA} achieves $\sqrt{d H^4 \log|\gF| T} + \sqrt{dH^5 \log|\mathcal{A}|T}$ regret, requiring only a factor of $dH\log T$ more samples than Algorithm \ref{alg:DOUHUA} even in the best case for the latter.



\subsection{Extensions to hybrid RL}

\citet{zhou2023offlinedataenhancedonpolicy} use both offline and online data to bypass the need to perform strategic exploration in policy optimization. This corresponds to the setting of hybrid RL \citep{nakamoto2023calql, amortila2024harnessing, ren2024hybrid, wagenmaker2023leveraging}, where \citet{song2023hybrid} show that using both offline and online data allows one to achieve $\sqrt{T}$ regret without optimism. However, the claimed $\sqrt{T}$ regret bound in \citet{zhou2023offlinedataenhancedonpolicy} requires on-policy sampling of $O(T)$ samples per timestep that does not contribute to the regret, leading to a sample complexity of $1/\epsilon^6$. Their algorithm cannot achieve sublinear regret in the more common setup where each sample contributes to the regret. Furthermore, they require bounded occupancy measure ratios of the optimal policy to \textit{any} policy.

We demonstrate that these issues can be mitigated. Specifically, we extend our optimistic algorithm to leverage both offline and online data, and show that actor-critic methods can benefit from hybrid data and achieve the provable gains in sample efficiency as observed in \cite{li2023reward, tan2024natural, tan2024hybridreinforcementlearningbreaks}. We also provide a non-optimistic provably efficient actor-critic
algorithm that only additionally requires $N_{\off} \geq c_{\off}^*(\gF,\Pi)dH^4/\epsilon^2$ offline samples (with bounded single-policy concentrability) in exchange for omitting optimism. This, along with the result in Theorem \ref{thm:hybrid-nora-regret_bound}, shows that hybrid RL therefore allows for sample efficiency gains with optimistic algorithms and computational efficiency gains with non-optimistic algorithms.

\paragraph{Notation.} We write $\gN_A(\rho)$ for the $\rho$-covering number over a set $A$, and in a small abuse of notation use $\gN_{A,B}(\rho)$ to denote the $\rho$-covering number over the set $A \cup B$. We use standard big-O and little-O notation within this paper. $O(g(n))$ denotes the set of functions $f(n)$ such that there exist positive constants $C$ and $n_0$ such that $0 \leq f(n) \leq C\cdot g(n)$ for all $n \geq n_0$. When we say a function is $O(g(n))$, we say that it grows no faster than $g(n)$ asymptotically.
$o(g(n))$ denotes the same, except that the inequality $0 \leq f(n) < C\cdot g(n)$ is strict, and so grows strictly slower than $g(n)$ asymptotically. $\Omega(g(n))$ involves the inequality $0 \leq C\cdot g(n) \leq f(n)$, and so denotes the set of functions $f(n)$ that grow at least as fast as $g(n)$. Finally, $\Theta(g(n))$ denotes the set of functions that are both $O(g(n))$ and $\Omega(g(n))$.

\section{Problem setting}

\paragraph{Markov decision processes.}
This paper focuses on finite horizon, episodic MDPs, represented by a tuple
$$\mathcal{M} = (\mathcal{S},\mathcal{A}, H,\{\prob_h\}_{h=1}^{H} ,\{r_h\}_{h=1}^{H}),$$
where $\mathcal{S}$ is the state space, $\mathcal{A}$ is the action space, $H$ is the horizon, $r_h: \mathcal{S}\times\mathcal{A}\to[0,1]$ is the reward function at step $h$ and $\prob_h: \mathcal{S}\times\mathcal{A}\to\Delta(\mathcal{S})$ is the transition kernel for step $h$. A policy $\{\pi_h\}_{h=1}^{H}$ is a set of $H$ functions, where each $\pi_h:\mathcal{S}\to\Delta(\mathcal{A})$ maps from a state on step $h$ to a probability distribution on actions. Write $\Pi$ for the class of all policies, and $\pi(s)$ as shorthand for the random variable $a \sim \pi(\cdot | s)$. Given a policy $\{\pi_h\}_{h=1}^{H}$ and reward function $\{r_h\}_{h=1}^{H}$, the state value function is defined as 
\begin{align}
    V_h^{\pi}(s) = \mathbb{E}\bigg[\sum_{h'=h}^{H}r_{h'}(s_{h'},a_{h'})|s_h = s\bigg],
\end{align}
where the expectation is taken over the randomness of $a_{h'}\sim \pi_{h'}(s_{h'})$ and $s_{h'+1}\sim \prob_h(\cdot|s_{h'},a_{h'})$ for any $h'\geq h$. The action value, or Q function is defined as
\begin{align}
Q_h^{\pi}(s,a) = \mathbb{E}\bigg[\sum_{h'=h}^{H}r_{h'}(s_{h'},a_{h'})|s_h = s,a_h=a\bigg],
\end{align}
where the expectation is taken over the similar randomness of action and state transition, with the only difference that the action randomness is only random as $h'\geq h+1$.

Without loss of generality, write $s_1$ for the initial state. The optimal policy is $\pi^{\star} = \mathrm{argmax}_{\pi\in\Pi} V_1^{\pi}(s_1)$. Correspondingly, we denote $V^{\star} = V^{\pi^{\star}}$ and $Q^{\star} = Q^{\pi^{\star}}$ as the optimal value and Q-functions. The Bellman operator with respect to the greedy policy and any policy $\pi$ is given by
\begin{subequations}
\begin{align}
    \mathcal{T}_{h}Q_{h+1}(s,a) &= r_h(s,a) + \mathbb{E}_{s'\sim P_h}\Big[\max_{a'\in\mathcal{A}}Q_{h+1}(s',a')\Big],\\
    \text{and }~\mathcal{T}^\pi_{h}Q_{h+1}(s,a) &= r_h(s,a) + \mathbb{E}_{s'\sim P_h}\Big[Q_{h+1}(s',\pi_{h+1}(s'))\Big].
\end{align}
\end{subequations}
The optimal Q-function $Q^{\star}$ is uniquely determined as the solution to the Bellman equation: $Q_{h}^{\star}(s,a)=\mathcal{T}_hQ_{h+1}^{\star}(s,a)$. Our goal is typically to learn an $\epsilon$-optimal policy $\widehat{\pi}$, such that $V_1^*(s_1) - V^{\widehat{\pi}}_1(s_1) \leq \epsilon$, or to obtain sublinear regret over $T$ rounds while playing $(\pi^{(t)})_{t=1}^T$:

\begin{align}
   \text{Reg}(T) := \sum_{t=1}^T \left(V_1^*(s_1) - V^{\pi^{(t)}}_1(s_1)\right) = o(T).
\end{align}

\paragraph{RL with function approximation.} Under general function approximation, we approximate Q-functions with a function class $\gF = \{\gF_h\}_{h \in [H]}$, where each $f_h : \gS \times \gA \to [0,H]$. The Bellman error with regard to $f \in \gF$ is $f_h - \gT_h f_{h+1}$, and additionally with regard to $\pi \in \Pi$ as $f_h - \gT_h^\pi f_{h+1}$.
Additionally, we write $\pi_h^{f}(a|s) = \mathbbm{1}(a' \in \argmax_{a\in \gA} f_h(s,a))$ for the greedy policy that plays the best action under $f$. We make the following routine assumptions on the richness of $\gF$ \citep{jin2021bellman, xie2022role,rajaraman2020fundamental, rashidinejad2023bridging}:

\begin{aspt}[Realizability]
    The function class $\mathcal{F}$ is rich enough such that for all $h\in[H]$, the function class $\gF_h$ contains the optimal action value function $Q_h^{\star}$: $Q^*_h \in \gF_h$.
\end{aspt}

\begin{aspt}[Generalized Completeness]
    There exists an auxiliary function class $\gG = \gG_1 \times ...\times\gG_H$, where each $g_h \in \gG_h$ satisfies $g_h: \gS \times \gA \to [0, H]$, that is sufficiently rich such that it contains all Bellman backups of $f \in \gF$. 
\end{aspt}

This auxiliary function class is $(\gT^\Pi)^T \gF = \{\gT^{\pi^{(T)}} \cdot ... \cdot \gT^{\pi^{(1)}} f \mid f \in \gF, \pi^{(1)},...,\pi^{(T)} \in \Pi\}$ for Algorithm \ref{alg:DOUHUA}, and $\gT \gF = \{\gT f \mid f \in \gF\}$ for Algorithm \ref{alg:NORA}. The former is far larger than the latter, with exceptions that we highlight in Section \ref{sec:douhua-easy}. 
We write $\gN_\gF(\rho)$ for the $\rho$-covering number of a function class $\gF$.\footnote{The $\rho$-covering number of a class corresponds to the smallest cardinality of a set of points, such that every point in the class is at least $\rho$-close to some point in that set. See \cite{wainwright2019high}.} To learn $f \in \gF$ that approximates the Q-function of a policy $\pi$ (we say that $f$ targets $\pi$), it is common to minimize the temporal difference (TD) error over a dataset $\gD$, as an estimate of the Bellman error:
\begin{align}\label{eqn:TD-error}
    \gL_h^{(t,\pi)}(f_h,f_{h+1})  = \sum_{(s,a,r,s') \in \gD} (f_h(s,a) - r - f_{h+1}(s',\pi_{h+1}(s')))^2.
\end{align} 

\paragraph{Measures of complexity.} The complexity of online learning in the presence of general function approximation is governed by complexity measures such as the Bellman rank \citep{jiang2016bellmanrank}, which corresponds to the intrinsic dimension in tabular, linear, and linear mixture MDPs. Another is the Bellman eluder dimension \citep{jin2021bellman}, which subsumes the Bellman rank and additionally characterizes the complexity of kernel, neural, and generalized linear MDPs. We use the squared distributional version:
\begin{defn}[Squared Distributional Bellman Eluder dimension \citep{jin2021bellman, xiong2023generalframeworksequentialdecisionmaking}]
    Let $\gF$ be a function class. The distributional Bellman Eluder dimension is the largest $d$ such that there exist measures $\{d_h^{(1)}, \ldots, d_h^{(d-1)}\}$, $d_h^{(d)}$, Bellman errors $\{\delta_h^{(1)}, \ldots, \delta_h^{(d-1)}\}$, $\delta_h^{(d)}$, and some $\epsilon'\geq\epsilon$, such that for all $t=1,...,d$, 
    \begin{align} 
    \left|\mathbb{E}_{d_h^{(t)}}[\delta_{h}^{(t)}]\right|>\varepsilon^{(t)} \text{ and } \sqrt{\sum_{i=1}^{t-1}\left(\mathbb{E}_{d_h^{(i)}}[(\delta_{h}^{(t)})^2]\right)} \leq \varepsilon^{(t)}.
    \end{align}
\end{defn}

\cite{jin2021bellman} primarily consider two types of distributions: (1) distributions $\mathcal{D}_{\mathcal{F}}$ induced by greedy policies $\pi^f$, and Dirac delta measures over state-action pairs $\mathcal{D}_{\Delta}$. They suggest an RL problem has low Bellman eluder dimension if either variant is small. Instructive examples include tabular MDPs, where this corresponds to the cardinality of the state-action space, and linear MDPs, where this is the corresponding dimension. 

The sequential extrapolation coefficient (SEC) of \citet{xie2022role} subsumes the Bellman eluder dimension:
\begin{defn}[Sequential Extrapolation Coefficient (SEC)]\label{defn:SEC}
    \begin{align}
        &\mathsf{SEC}(\gF,\Pi,T) \coloneqq \max_{h \in [H]} \sup_{f^{(1)}, \ldots, f^{(T)} \in \gF} \sup_{\pi^{(1)}, \dots, \pi^{(T)} \in \Pi} \qquad\left\{\sum_{t=1}^T \frac{\E_{\pi^{(t)}}[f_h^{(t)} - \gT_h f_{h+1}^{(t)}]^2}{H^2 \vee \sum_{i=1}^{t-1} \E_{\pi^{(i)}}[(f_h^{(t)} - \gT_h f_{h+1}^{(t)})^2]}\right\}.
    \end{align}
\end{defn}

The $\mathsf{SEC}$ is always bounded by $d \log T$, but there exist MDPs that have a Bellman eluder dimension $d$ on the order of $\sqrt{T}$, but a constant SEC \citep{xie2022role}. 
We shall use these measures of complexity to characterize the regret. 
Algorithm \ref{alg:DOUHUA} scales with the SEC, which is more general and weaker than the Bellman eluder dimension. Algorithm \ref{alg:NORA}, as presented, has a switching cost that scales with the $\gD_\Delta$-type Bellman eluder dimension. While this can be weakened to the more general $\ell_2$ eluder condition of \cite{xiong2023generalframeworksequentialdecisionmaking} with nothing more than a change in notation, we present our results in the Bellman eluder dimension framework for familiarity and ease of presentation.


\paragraph{Policy optimization and actor-critic algorithms.}
Policy optimization approaches optimize directly in the space of policies, enabled by the policy gradient theorem \citep{sutton1998}: $\nabla_\pi V_1^\pi(s_1) = \E_\pi[Q_1^\pi(s,a)\nabla_\pi \log \pi(s,a)]$. This can be done with Monte-Carlo estimates of $Q_1^\pi(s,a)$ (the REINFORCE algorithm), or a learned estimate of $Q_1^\pi(s,a)$ called a critic (actor-critic methods).

However, vanilla policy gradient methods can converge very slowly in the worst case \cite{li2021softmax}. It is often preferable to use other optimization algorithms, such as a second-order method in natural policy gradient (NPG) \citep{kakade2001npg}, KL-regularized gradient ascent in trust region policy optimization (TRPO) from \citet{schulman2017trustregionpolicyoptimization}, or proximal policy optimization (PPO), which performs a similar, but easier to compute, update. 
These methods are closely related in the limit of small step-sizes, and are approximate versions of mirror ascent \citep{schulman2017trustregionpolicyoptimization,neu2017unifiedviewentropyregularizedmarkov, rajeswaran2017learning}.

One instance in which the NPG, TRPO, and PPO updates coincide is with softmax policies \citep{cai2024provablyefficientexplorationpolicy,cen2022fast,agarwal2021policygradient}: $\pi(a|s) = \exp(g(s,a))/\sum_{a \in \gA} \exp(g(s,a))$ for some function $g : \gS \times \gA \to \R$. In this case, the update has a closed form: 
\begin{equation}\pi_h^{(t+1)}(a|s)\propto \pi_h^{(t)}(a|s)\exp(\eta f_h^{(t)}(s,a)), \;f_h^{(t)} \in \gF.
\label{eqn:multiplicative-weights-update}
\end{equation}
As mirror ascent, this update is identical to the multiplicative weights or Hedge algorithm. Like \citep{cai2024provablyefficientexplorationpolicy, zhong2023theoreticalanalysisoptimisticproximal, liu2023optimisticnaturalpolicygradient}, we exploit this equivalence to prove our regret bounds. 


\section{Optimistic actor-critics -- The easy case}
\label{sec:douhua-easy}
We now present an optimistic actor-critic algorithm, DOUHUA (Algorithm \ref{alg:DOUHUA}), with provable guarantees under general function approximation. DOUHUA is a natural derivative of the GOLF algorithm of \citet{jin2021bellman} for actor-critic approaches\footnote{Or a completely off-policy version of \citet{liu2023optimisticnaturalpolicygradient}.}, with only two (very natural) changes. The critic targets $Q^{\pi^{(t)}}$ instead of $Q^*$ while performing optimistic planning at every $s,a$ pair, and we maintain a stochastic policy that is updated with Equation \ref{eqn:multiplicative-weights-update}. Learning an optimistic critic off-policy achieves sample efficiency by reusing data while exploring efficiently.\footnote{Similarly to \citet{cai2024provablyefficientexplorationpolicy} in linear mixture MDPs.}

\begin{algorithm}[t]
    \caption{Double Optimistic Updates for Heavily Updating Actor-critics (DOUHUA)}
    \begin{algorithmic}[1]
            \STATE {\bfseries Input:} Function class $\gF$.
            \STATE {\bfseries Initialize:} $\gF^{(0)} \leftarrow \gF$, $\gD_h^{(0)}\leftarrow \emptyset$, learning rate $\eta = \Theta(\sqrt{\log|\gA| H^{-2}T^{-1}})$, policy $\pi^{(1)} \; \propto \; 1$, confidence width $\beta=\Theta(\log (HT \mathcal{N}_{\gF, (\gT^{\Pi})^T \gF}(1/T) / \delta))$.
            \FOR{episode $t = 1, 2, \dots, T$}
            \STATE $f_h^{(t)}(s,a) \in \operatorname{argmax}_{f \in \mathcal{F}^{(t-1, \pi^{(t-1)})}} f_h\left(s,a\right) \forall s,a,h$.
            \STATE Play policy $\pi^{(t)}$ for one episode, update dataset $\mathcal{D}_h^{(t)}$. 
            \STATE Compute confidence set $\gF^{(t,\pi^{(t)})}$:
            \begingroup
            \addtolength\jot{-1ex}
            \begin{align*}
                &\mathcal{F}^{(t, \pi^{(t)})} \leftarrow\left\{f \in \mathcal{F}: \mathcal{L}_h^{(t, \pi^{(t)})}\left(f_h, f_{h+1}\right)-\min _{f_h^{\prime} \in \mathcal{F}_h} \mathcal{L}_h^{(t, \pi^{(t)})}\left(f_h^{\prime}, f_{h+1}\right) \leq H^2\beta \; \forall h\right\},\\
                &\mathcal{L}_h^{(t, \pi^{(t)})}\left(f, f'\right) \gets\sum_{\left(s, a, r, s^{\prime}\right) \in \mathcal{D}_h^{(t)}}\left(f(s, a)-r-f^{\prime}(s^{\prime}, \pi^{(t)}_{h+1}(s'))\right)^2.
            \end{align*}
            \endgroup
        \STATE Update $\pi_h^{(t+1)}(a|s) \propto \pi_h^{(t)}(a|s)\exp(\eta f_h^{(t)}(s,a))$.
        \ENDFOR
    \end{algorithmic}
\label{alg:DOUHUA}
\end{algorithm}

Algorithm \ref{alg:DOUHUA} maintains confidence sets $\gF^{(t, \pi^{(t)})}$ that contain all $f \in \gF$ where the TD error $\gL_h^{(t,\pi^{(t)})}$ with respect to the Bellman operator $\gT_h^{\pi^{(t)}}$ is a $H^2\beta$-additive approximation of the minimizer $\min_{f_h' \in \gF_h} \gL_h^{(t,\pi^{(t)})}(f_h',f_{h+1})$. Upon the start of each trajectory $t$, Algorithm \ref{alg:DOUHUA} maximizes among all functions in the confidence set to play 
$f_h^{(t)}(s,a) = \sup_{f_h \in \gF_h^{(t-1,\pi^{(t-1)})}} f_h(s,a)$ for all $h, s,a$. This is exactly in line with GOLF, except that the critic targets $Q^{\pi^{(t)}}$ instead of $Q^*$, and we perform optimistic planning with regard to every state and action like \citet{liu2023optimisticnaturalpolicygradient}. We then perform a mirror ascent update $\pi_h^{(t+1)} \propto \pi_h^{(t)}\exp(\eta f_h^{(t)})$ instead of playing the greedy policy $\pi_h^{f^{(t)}}$. This algorithm satisfies the following  regret/sample complexity bound: 

\begin{thm}[Regret Bound for DOUHUA]
    
Algorithm \ref{alg:DOUHUA} achieves the following regret with probability at least $1-\delta$:
\begin{align*}
    \mathrm{Reg}(T)
    = O\left(\sqrt{H^4 T \log|\gA|} + \sqrt{\beta H^4 T\mathsf{SEC}(\gF, \Pi, T)}\right),
\end{align*}
where $\beta = \Theta\left(\log \left(HT \mathcal{N}_{\gF, (\gT^{\Pi})^T \gF}(1/T) / \delta\right)\right)$. To learn an $\epsilon$-optimal policy, it therefore requires:
\begin{align*}N \geq \Omega\left(H^4 \log|\gA|/\epsilon^2 + \beta H^4 \mathsf{SEC}(\gF, \Pi, T)/\epsilon^2\right).\end{align*}
\label{thm:regret-bound-douhua}
\end{thm}

\begin{proof}[Proof sketch:]
We decompose the regret with Lemma \ref{lem:regret-decomp-douhua-ac}:
    \begin{align}\label{eqn:regret-decom-douhua}
    \operatorname{Reg}(T)
    = & \underbrace{\sum_{t=1}^T \sum_{h=1}^H \mathbb{E}_{\pi^*}\left[\left\langle f_h^{(t)}\left(s_h, \cdot\right), \pi_h^*\left(\cdot \mid s_h\right)-\pi_h^{(t)}\left(\cdot \mid s_h\right)\right\rangle\right]}_{\text{\textcolor{red}{Tracking error of $\pi^{(t)}$ w.r.t. $\pi^*$, bounded by mirror ascent arguments.}}}\nonumber \\ 
    &\qquad\underbrace{-\sum_{t=1}^T \sum_{h=1}^H\mathbb{E}_{\pi^*}\left[\left(f_h^{(t)}-\gT_h^{\pi^{*}}f_{h+1}^{(t)}\right)\left(s_h, a_h\right)\right]}_{\text{\textcolor{orange}{Negative Bellman error under $\pi^*$, bounded by optimism.}}}\nonumber \\
    &\qquad\underbrace{+ \sum_{t=1}^T \sum_{h=1}^H\mathbb{E}_{\pi^{(t)}}\left[\left(f_h^{(t)}-\gT_h^{\pi^{(t)}} f_{h+1}^{(t)}\right)\left(s_h, a_h\right)\right]}_{\text{\textcolor{blue}{Bellman error under current policy occupancy, bounded by critic error.}}}.
\end{align}

We will then bound each term in (\ref{eqn:regret-decom-douhua}) separately. By a mirror ascent argument in Lemma \ref{lem:mirror-descent-douhua}, the first term is bounded as
\begin{align}\label{eqn:regret-decom-douhua1}
    \sum_{t=1}^{T} \sum_{h=1}^H \mathbb{E}_{\pi^*}\left[\left\langle f_h^{(t)}\left(s_h, \cdot\right), \pi_h^*\left(\cdot \mid s_h\right)-\pi_h^{(t)}\left(\cdot \mid s_h\right)\right\rangle\right] \leq \eta H^3 T/2 + \frac{H \log |\gA|}{\eta},
\end{align}
Lemma \ref{lem:ac-loss-bounded-optimism-pi} establishes optimism: $f_h^{(t)} \geq \gT_h^{(t-1)} f_{h+1}^{(t)}$. So we see in Lemma \ref{lem:negative-bellman-decomp-douhua} that the second term is decomposed as
\begin{align}\label{eqn:regret-decom-douhua2-1}
    -\sum_{t=1}^T \sum_{h=1}^H \mathbb{E}_{\pi^*}\left[f_h^{(t)}-\mathcal{T}_h^{\pi^{*}} f_{h+1}^{(t)}\right] & \leq  \sum_{t=1}^T \sum_{h=1}^H \mathbb{E}_{\pi^{*}}\left[\left\langle f_{h+1}^{(t)}(s_{h+1},\cdot), \pi_{h+1}^{*} (\cdot | s_{h+1}) -\pi_{h+1}^{(t)}(\cdot | s_{h+1}) \right\rangle\right] + \eta H^3T\\
    \label{eqn:regret-decom-douhua2}
    & \leq 3\eta H^3T/2 + \frac{H \log |\gA|}{\eta},
\end{align}
where the the second inequality holds as the first term of~(\ref{eqn:regret-decom-douhua2-1}) is a copy of the first term in (\ref{eqn:regret-decom-douhua}), thus can be bounded by (\ref{eqn:regret-decom-douhua1}).
For the third term, which is the Bellman error under the $\pi^{(t)}$ occupancy measure, we use Lemma \ref{lem:occ-measure-regret-douhua} to bound it as
\begin{align}\label{eqn:regret-decom-douhua3}
    \sum_{t=1}^T \sum_{h=1}^H\mathbb{E}_{\pi^{(t)}}\left[\left(f_h^{(t)}-\gT^{\pi^{(t)}}_h f_{h+1}^{(t)}\right)\left(s_h, a_h\right)\right]\leq \sqrt{\beta H^4 T \mathsf{SEC}(\gF, \Pi, T)} + \eta H^3 T,
\end{align}
where $\beta = \Theta\left(\log\left(HT\mathcal{N}_{\gF, (\gT^{\Pi})^T \gF}(1/T)/\delta\right)\right)$. Plugging (\ref{eqn:regret-decom-douhua1}), (\ref{eqn:regret-decom-douhua2}) and (\ref{eqn:regret-decom-douhua3}) into (\ref{eqn:regret-decom-douhua}), we obtain that
\begin{align}
    \operatorname{Reg}(T) &\leq  \sqrt{\beta H^4 T \mathsf{SEC}(\gF, \Pi, T)} + 3\eta H^3 T + \frac{2H \log |\gA|}{\eta} \nonumber\\
    &= O\left(\sqrt{H^4 T \log|\gA|} + \sqrt{\beta H^4 T\mathsf{SEC}(\gF, \Pi, T)}\right),
\end{align}
where the second equality holds when we set $\eta = O(\sqrt{\log|\mathcal{A}|H^{-2}T^{-1}})$.
\end{proof}

\paragraph{A Vacuous Bound.} The form of the above regret bound is appealing at first glance. However, the upper bound on the log-covering number of the policy class increases linearly with the number of significant\footnote{One can use an $\epsilon$-net and collapse nearby critics when $\exp(\eta(f-f')) < \epsilon$, which can reduce the bound significantly.} critic updates by default, as we show below in Lemma \ref{lem:b2-covering-number-policy}.\footnote{It was previously thought in \citet{zhong2023theoreticalanalysisoptimisticproximal} that the log-covering number of the policy class increases in the number of actor and critic updates. We sharpen this bound to the number of critic updates in Lemma \ref{lem:b2-covering-number-policy}, which may be of independent interest.} 
\begin{lem}[Bound on Covering Number of Policy Class, Modified Lemma B.2 from \citet{zhong2023theoreticalanalysisoptimisticproximal}]
    Consider the policy class $\Pi^{(T)}$ obtained by performing $T$ updates of the mirror ascent update as in (\ref{eqn:multiplicative-weights-update}), 
    where the critics $f_h^{(t)} \in \gF$ are updated at times $t_1,...,t_K$. Then, the covering number of the policy class at time $T$ is bounded by 
    \begin{align}\gN_{\Pi_h^{(T)}}(\rho/2H) \leq \min\left\{\prod_{k=1}^{K}\gN_{\gF_h^{(t_k)}}(\rho^2/16\eta k H^2 T), \;\gN_{\gA}(\rho/2H)\right\}.\end{align}
    \label{lem:b2-covering-number-policy}
\end{lem}
The proof is deferred to Appendix \ref{app:ac-covering-number}. Within Algorithm \ref{alg:DOUHUA}, this implies that $\log \gN_{(\gT^\Pi)^{T}\gF}(\rho) = O(\min\{ T , H\gN_{\gA}(\rho)\}\log \gN_{\gF}(\rho))$ in general, making the bound in Theorem \ref{alg:DOUHUA} vacuous. 

\paragraph{A Good Case.} However, within certain circumstances, it is possible to have $\log \gN_{(\gT^\Pi)^{T}\gF}(\rho) = O(\log \gN_{\gF}(\rho))$. This happens, for instance, when there exists a low-dimensional representation of the sum of clipped Q-functions -- such as when the sum of clipped Q-functions is a scaled Q-function:
\begin{defn}[Closure Under Truncated Sums]
    $\gF$ is closed under truncated sums if for any $T \in \mathbb{N}$ and $f^{(1)},...,f^{(T)} \in \gF$, $T^{-1}\sum_{t=1}^T \min\{\max\{f^{(t)},0\},H\} \in \gF$.
    \label{defn:closure-truncated-sums}
\end{defn}

As such, this is an interesting case where the log-covering number of the policy class does not blow up, with downstream implications for the regret of Algorithm \ref{alg:DOUHUA}:
\begin{lem}[Policy Class Growth]
    Let $\gF$ be a function class that satisfies Definition \ref{defn:closure-truncated-sums}, i.e. it is closed under truncated sums. Then, $\log \gN_{(\gT^\Pi)^{T}\gF}(\rho) = O(\gN_{\gF}(\rho^2/16\eta TH^2))$.
    \label{lem:policy-class-growth-sums}
\end{lem}

\begin{cor}[Regret of DOUHUA, The Good Case]
    If $\gF$ is closed under truncated sums, then with probability at least $1-\delta$, Algorithm \ref{alg:DOUHUA} achieves a regret of:
$$\mathrm{Reg}(T)
    = O\left(\sqrt{H^4 T \log|\gA|}  + \sqrt{\beta H^4 T\mathsf{SEC}(\gF, \Pi, T)}\right),$$
where $\beta = \Theta(\log (HT \mathcal{N}_{\gF}(H/\log|\gA| T^2) / \delta))$. To learn an $\epsilon$-optimal policy, it therefore requires:
$$N \geq \Omega\left(H^4 \log|\gA|/\epsilon^2 + H^4 \beta \mathsf{SEC}(\gF, \Pi, T)/\epsilon^2\right).$$
\label{cor:good-case-douhua}
\end{cor}

We defer the proof of Lemma \ref{lem:policy-class-growth-sums} to Appendix \ref{app:closure-truncated-sums}. Algorithm \ref{alg:DOUHUA} then achieves a regret that aligns with the results of \citep{efroni2020optimisticpolicyoptimizationbandit, cai2024provablyefficientexplorationpolicy} for tabular and linear mixture MDPs, respectively.\footnote{Where both the SEC and Bellman eluder dimension are $SA$ (tabular) and $d$ (linear).}
However, closure under truncated sums is a strong condition that is not fulfilled by many function classes, although tabular classes fulfill it. It is not fulfilled by linear models due to the clipping operator, requiring \citet{sherman2024rateoptimalpolicyoptimizationlinear} and \citet{cassel2024warmupfreepolicyoptimization} to develop bespoke algorithms to get around this in the setting of linear MDPs.\footnote{The former use reward-agnostic exploration to warm-start the critic to avoid truncation. The latter shrink features to do the same. Both employ rare-switching updates for the bonus function but not the critic, avoiding the moving target issue.} In the case of trees, random forests, boosting, and neural networks, without further assumptions, one needs to increase the size of the function class -- perhaps even on the same order as the increase in Lemma \ref{lem:b2-covering-number-policy}.\footnote{For instance, with neural networks, averaging two neural networks of $12$ layers results in a larger neural network with $24$ layers, and averaging two random forests of $500$ trees results in a random forest of $1000$ trees. However, one can argue that the effective complexity does not increase too much in these cases, and as such we can expect to see sublinear regret with certain nonparametric function classes.}

This prompts us to explore the possibility of algorithmic modifications to Algorithm \ref{alg:DOUHUA} in order to achieve the optimal regret rates in more general settings where the log-covering number of the policy class may increase linearly in the number of critic updates. We do so in the next section.

\section{Optimistic actor-critics -- The hard case}

Given what we have seen in our analysis of Algorithm \ref{alg:DOUHUA}, can we simply modify Algorithm \ref{alg:DOUHUA} to include rare-switching critic updates? If we can perform only $O(dH\log T)$ critic updates as in \cite{xiong2023generalframeworksequentialdecisionmaking}, perhaps it may be possible to obtain a similar regret bound to Corollary \ref{cor:good-case-douhua}. 

\subsection{Challenges and algorithm design}

However, this is not the case. Intuitively, if the critic targets a rapidly changing $\pi^{(t)}$, the Bellman error with regard to the policy $\pi^{(t_{\last})}$ at the last update will not be close to the Bellman error with regard to $\pi^{(t)}$. Although the former is what the critic $f^{(t_{\last})}=f^{(t)}$ targets, we evaluate the latter when considering a switch at time $t$. Therefore, the critic updates far more often when we target $Q^{\pi^{(t)}}$ rather than $Q^*$ -- as it tries to hit a moving target. 

Furthermore, this results in insufficient optimism. In Lemma \ref{lem:ac-loss-bounded-optimism-pi}, we can only guarantee optimism with respect to the Bellman operator at the last critic update time $t_{\last}$, $f_h^{(t)} \geq \gT_h^{\pi^{(t_{\last})}}f_{h+1}^{(t)}$. But we require $f_h^{(t)} \geq \gT_h^{\pi^{(t)}}f_{h+1}^{(t)}$. Attempting to work with this form of limited optimism results in the tracking error relative to $\pi^*$ becoming
$$\sum_{t=1}^T\sum_{h=1}^H \E_{\pi^*}\left[\left\langle f_h^{(t)}(s',\cdot), \pi_h^*(\cdot | s') - \pi_h^{(t_{\last})}(\cdot | s')\right\rangle\right],$$
incurring linear regret. This is the shaded area in Figure \ref{fig:ac-illustration}.

\begin{figure}[tbp!]
    \centering
    \includegraphics[width=0.5\linewidth]{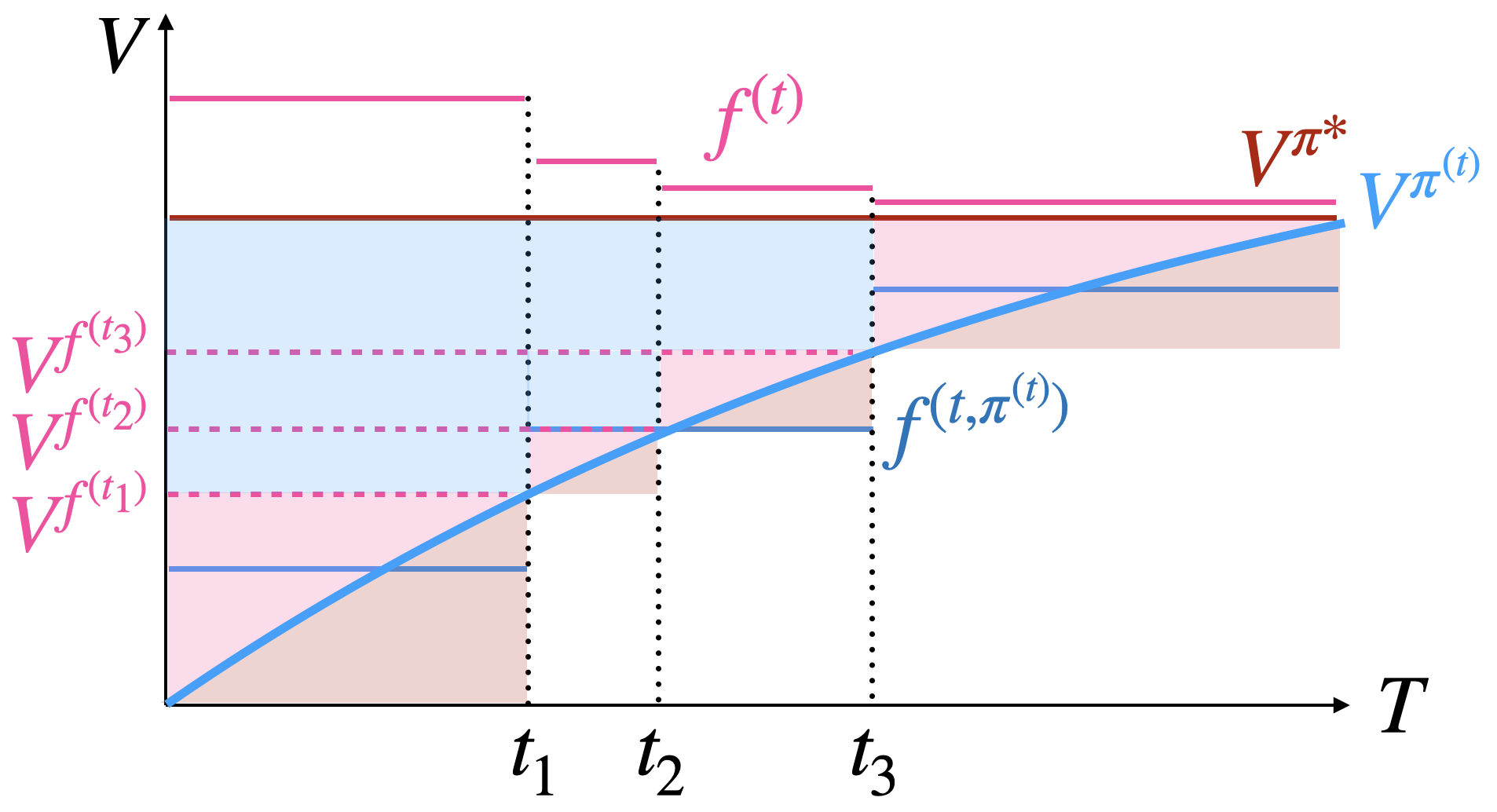}
    \caption{Illustration of tracking error in policy optimization, with a rare-switching critic $f^{(t)}$ that targets $\pi^*$. The \{blue and pink, pink\} area depicts the tracking error of \{$\pi^*$ \text{to} $\pi^{(t)}$, $\pi^{f^{(t)}}$ \text{to}  $\pi^{(t)}$\}. Both incur $\sqrt{T}$ regret. In contrast, $f^{(t,\pi^{(t)}})$ is a rare-switching critic that targets $\pi^{(t)}$, and so is insufficiently optimistic as $\pi^{(t)}$ changes. The blue, pink, and rust area depicts the tracking error of $\pi^{*}$ \text{to}  $\pi^{(t_{\last})}$ from insufficient optimism, which yields linear regret.}
    \label{fig:ac-illustration}
\end{figure}

\paragraph{A way forward.} Having the critic target $Q^*$ instead of $Q^{\pi^{(t)}}$ provides a solution. This ensures sufficient optimism, as we show in Lemma \ref{lem:ac-loss-bounded-optimism} that $f_h^{(t)} \geq \gT_h^{\pi^{(t)}}f_{h+1}^{(t)}$. Furthermore, we do not need to control $\log \mathcal{N}_{\left(\mathcal{T}^{\Pi}\right)^T \mathcal{F}}(\rho)$, as $\max_a Q_{h+1}(s,a)$ is a contraction and it suffices to control $\log \mathcal{N}_{\mathcal{T} \mathcal{F}}(\rho)$, which is often approximately $\log \mathcal{N}_{\mathcal{F}}(\rho)$. 

However, this introduces an additional term we need to control -- the deviation of the current policy $\pi^{(t)}$ from the greedy policy $\pi^{f^{(t)}}$, depicted as the pink area in Figure \ref{fig:ac-illustration}. This term is difficult to control, as $\pi^{f^{(t)}}$ changes with every critic update, and the actor requires sufficient time to catch up to the critic updates. 
To address this issue, we introduce policy resets at every critic update, allowing us to bound this with the standard mirror descent regret bound as in Lemma \ref{lem:policy-tracking-ft-t-nora}. 
The total tracking error scales with the number of critic updates -- which is then resolved by performing rare-switching critic updates in line with \citet{xiong2023generalframeworksequentialdecisionmaking}. 
Lastly, we increase the learning rate to reduce the regret incurred by the policy resets. This makes more aggressive policy updates to catch up with the sudden, rare, and large critic updates.

\paragraph{Summary. } Algorithm \ref{alg:NORA} combines \emph{optimism} for strategic exploration and \emph{off-policy learning} for sample efficiency. While rare-switching critic updates are a-priori appealing, they are unstable when the critic targets $\pi^{(t)}$, due to limited optimism and the challenge of tracking a moving policy (as we describe below, and further show in Appendix \ref{app:problem-target-q-pi}). These are resolved by \emph{targeting $\pi^*$} and re-introducing \emph{rare-switching critic updates} respectively. However, this introduces additional error, as the agent has to effectively unlearn the policy after each rare critic update. We therefore introduce \emph{policy resets} to the uniform policy after each critic update -- incurring minimal cost due to their rarity (as there are only $O(dH\log T)$ updates). A \emph{more aggressive learning rate} (by a factor of $\sqrt{dH\log T}$) mitigates some of the additional regret incurred, and can be seen as making more aggressive updates to make up for lost ground from policy resets. See Appendix \ref{app:proof-nora} for more details.

\begin{algorithm}[t]
    \caption{No-regret Optimistic Rare-switching Actor-critic (NORA)}
    \begin{algorithmic}[1]
            \STATE {\bfseries Input:} Function class $\gF$.
            \STATE {\bfseries Initialize:} $\gF^{(0)} \leftarrow \gF$, offline dataset $\gD_{\off}$, $\gD_h^{(0)}\leftarrow \emptyset, \forall h \in [H]$, $\eta = \Theta(\sqrt{d\log T \log|\gA| H^{-1}T^{-1}})$, $\pi^{(1)} \; \propto \; 1$, confidence width $\beta=\Theta(\log (HT^2 \mathcal{N}_{\gF, \gT \gF}(1/T) / \delta))$.
            \FOR{episode $t = 1, 2, \dots, T$}
            \STATE Set $f_h^{(t)}(s,a) \in \operatorname{argmax}_{f \in \mathcal{F}^{(t_{\text{last}})}} f_h\left(s,a\right) \forall s,a,h$.
            \STATE Play policy $\pi^{(t)}$ for one episode, obtain trajectory, update dataset $\mathcal{D}_h^{(t)}$. \IF{$\mathcal{L}_h^{(t)}(f_h^{(t)}, f_{h+1}^{(t)}) \geq \min_{f_h' \in \gF_h} \mathcal{L}_h^{(t)}(f_h', f_{h+1}^{(t)}) + 5H^2\beta$ for some $h$} 
            \STATE Compute confidence set $\gF^{(t)}$:
            \begingroup
            \addtolength\jot{-1ex}
            \begin{align*}
                &\mathcal{F}^{(t)} \leftarrow\left\{f \in \mathcal{F}: \mathcal{L}_h^{(t)}\left(f_h, f_{h+1}\right)-\min _{f_h^{\prime} \in \mathcal{F}_h} \mathcal{L}_h^{(t)}\left(f_h^{\prime}, f_{h+1}\right) \leq H^2\beta \; \forall h\right\},\\
                &\mathcal{L}_h^{(t)}\left(f, f'\right) \gets\sum_{\left(s, a, r, s^{\prime}\right) \in \mathcal{D}_h^{(t)}}\left(f(s, a)-r-\max_{a' \in \gA}f^{\prime}(s^{\prime}, a')\right)^2.
            \end{align*}
            \endgroup
            \STATE Reset policy $\pi^{(t)} \; \propto \; 1$. 
            \STATE Set $t_{\text{last}} \gets t$, $N_{\text{updates}}^{(t)} \gets N_{\text{updates}}^{(t-1)}+ 1$.
            \ELSE
            \STATE Set $N_{\text{updates}}^{(t)} \gets N_{\text{updates}}^{(t-1)}$, $\gF^{(t)} \gets \gF^{(t-1)}$.
            \ENDIF
        \STATE Update $\pi_h^{(t+1)}(a|s) \; \propto \; \pi_h^{(t)}(a|s)\exp(\eta f_h^{(t)}(s,a))$.
        \ENDFOR
    \end{algorithmic}
\label{alg:NORA}
\end{algorithm}



\subsection{Regret bound for NORA}

To control the switching cost with the framework of \citet{xiong2023generalframeworksequentialdecisionmaking}, we concern ourselves with function classes of low $D_\Delta$-type Bellman Eluder dimension \citep{jin2021bellman}:

\begin{aspt}[Bounded $D_{\Delta}$-type Bellman Eluder Dimension]
    Let $\mathcal{D}_{\Delta}:=\left\{\mathcal{D}_{\Delta, h}\right\}_{h \in[H]}$, where $\mathcal{D}_{\Delta, h}=\left\{\delta_{(s, a)}(\cdot) \mid s \in \mathcal{S}, a \in \mathcal{A}\right\}$. That is, we only consider distributions that are Dirac deltas on a single state-action pair. We assume that $d := d_{\text{BE}}(\gF, D_\Delta, 1/\sqrt{T}) < \infty$. 
\end{aspt}

This can be weakened to the more general $\ell_2$ eluder condition of \cite{xiong2023generalframeworksequentialdecisionmaking} with nothing more than a change in notation. However, we present our results in the Bellman eluder dimension framework for familiarity and ease of presentation. We now show the following for Algorithm \ref{alg:NORA}:

\begin{thm}[Regret Bound for NORA]
\label{thm:nora-regret_bound}
Algorithm \ref{alg:NORA} achieves the following regret with probability at least $1-\delta$:
\begin{align}
    \mathrm{Reg}(T)
    &= O\left(\sqrt{d H^5 T \log T \log|\gA|} + dH^3 \log T + \sqrt{\beta H^4 T  \mathsf{SEC}(\gF, \Pi, T)}\right).
\end{align}
where $\beta = \Theta(\log \left(HT^2 \mathcal{N}_{\gF,\gT \gF}(1/T) / \delta\right) )$.
This implies a sample complexity (ignoring lower-order terms) of 
\begin{align}N \geq \Omega\left(dH^5 \log T\log|\gA|/\epsilon^2 + H^4 \beta \mathsf{SEC}(\gF, \Pi, T)/\epsilon^2\right).\end{align}
\end{thm}

\begin{proof}[Proof sketch:]
By Lemma \ref{lem:regret-decomp-ac}, we have a slightly different regret decomposition than usual:
    \begin{align}\label{eqn:nora-regret-decom}
    \operatorname{Reg}(T)
    = & \underbrace{\sum_{t=1}^T \sum_{h=1}^H \mathbb{E}_{\pi^*}\left[\left\langle f_h^{(t)}\left(s_h, \cdot\right), \pi_h^*\left(\cdot \mid s_h\right)-\pi_h^{(t)}\left(\cdot \mid s_h\right)\right\rangle\right]}_{\text{\textcolor{red}{Tracking error of $\pi^{(t)}$ w.r.t. $\pi^*$, bounded by mirror ascent arguments.}}} \nonumber\\ 
    &\qquad\underbrace{-\sum_{t=1}^T \sum_{h=1}^H\mathbb{E}_{\pi^*}\left[\left(f_h^{(t)}-\gT_h^{\pi^{*}}f_{h+1}^{(t)}\right)\left(s_h, a_h\right)\right]}_{\text{\textcolor{orange}{Negative Bellman error under $\pi^*$, bounded by optimism.}}} \nonumber\\
    &\qquad\underbrace{+ \sum_{t=1}^T \sum_{h=1}^H\mathbb{E}_{\pi^{(t)}}\left[\left(f_h^{(t)}-\gT_h f_{h+1}^{(t)}\right)\left(s_h, a_h\right)\right]}_{\text{\textcolor{blue}{Bellman error under current policy occupancy, bounded by critic error.}}}\nonumber\\
    &\underbrace{+\sum_{t=1}^T \sum_{h=1}^H \mathbb{E}_{\pi^{(t)}}\left[\left\langle f_{h+1}^{(t)}(s', \cdot), \pi_{h+1}^{f^{(t)}}(\cdot | s')-\pi_{h+1}^{(t)}(\cdot | s')\right\rangle\right]}_{\text{\textcolor{teal}{Tracking error of $\pi^{(t)}$ w.r.t. $\pi^{f^{(t)}}$, bounded by mirror ascent and policy resets.}}}.
\end{align}

 We will then bound each term in (\ref{eqn:nora-regret-decom}) separately. We further decompose the first term as follows
 \begin{align}\label{eqn:nora-regret-decom-1}
           &\sum_{t=1}^T \sum_{h=1}^H \mathbb{E}_{\pi^{(t)}}\left[\left\langle f_h^{(t)}\left(s_{h+1}, \cdot\right), \pi_{h+1}^{\star}\left(\cdot \mid s_{h+1}\right)-\pi_{h+1}^{(t)}\left(\cdot \mid s_h\right)\right\rangle\right] \nonumber\\
        &= \sum_{k=1}^{K} \sum_{t=t_{k}+1}^{t_{k+1}} \sum_{h=1}^H \mathbb{E}_{\pi^{(t)}}\left[\left\langle f_h^{(t)}\left(s_h, \cdot\right), \pi_h^{\star}\left(\cdot \mid s_h\right)-\pi_h^{(t)}\left(\cdot \mid s_h\right)\right\rangle\right] \nonumber\\
        &= \sum_{k=1}^{K} \sum_{t=t_{k}+1}^{t_{k+1}-1} \sum_{h=1}^H \mathbb{E}_{\pi^{(t)}}\left[\left\langle f_h^{(t)}\left(s_h, \cdot\right), \pi_h^{\star}\left(\cdot \mid s_h\right)-\pi_h^{(t)}\left(\cdot \mid s_h\right)\right\rangle\right] \nonumber\\
        &\qquad+ \sum_{k=1}^{K}\sum_{h=1}^H \mathbb{E}_{\pi^{(t)}}\left[\left\langle f_h^{(t_k)}\left(s_h, \cdot\right), \pi_h^{\star}\left(\cdot \mid s_h\right)-\pi_h^{(t_k)}\left(\cdot \mid s_h\right)\right\rangle\right].
 \end{align}
For the first term of (\ref{eqn:nora-regret-decom-1}), following Lemma \ref{lem:mirror-descent-nora} and summing over all $t_!,\ldots,t_K$, we have
\begin{align}\label{eqn:nora-regret-decom-11}
    \sum_{k=1}^{K}\sum_{t=t_k+1}^{t_{k+1}-1} \sum_{h=1}^H \mathbb{E}_{\pi^*}\left[\left\langle f_h^{(t)}\left(s_h, \cdot\right), \pi_h^*\left(\cdot \mid s_h\right)-\pi_h^{(t)}\left(\cdot \mid s_h\right)\right\rangle\right] \leq \eta H^3 T/2 + \frac{KH \log |\gA|}{\eta} + KH^2,
\end{align}
 while for the second term, one may note that each inner product can be bounded by $2H$, therefore the summation can be bounded as
 \begin{align}\label{eqn:nora-regret-decom-12}
   \sum_{k=1}^{K}\sum_{h=1}^H \mathbb{E}_{\pi^{(t)}}\left[\left\langle f_h^{(t_k)}\left(s_h, \cdot\right), \pi_h^{\star}\left(\cdot \mid s_h\right)-\pi_h^{(t_k)}\left(\cdot \mid s_h\right)\right\rangle\right]\leq 2KH^2.
 \end{align}
 We apply this similar strategy on the fourth term of (\ref{eqn:nora-regret-decom}), which leads to the following decomposition
 \begin{align}\label{eqn:nora-regret-decom-4}
     &\sum_{t=1}^T \sum_{h=1}^H \mathbb{E}_{\pi^{(t)}}\left[\left\langle f_h^{(t)}\left(s_{h+1}, \cdot\right), \pi_{h+1}^{f^{(t)}}\left(\cdot \mid s_{h+1}\right)-\pi_{h+1}^{(t)}\left(\cdot \mid s_h\right)\right\rangle\right] \nonumber\\
        &= \sum_{k=1}^{K} \sum_{t=t_{k}+1}^{t_{k+1}-1} \sum_{h=1}^H \mathbb{E}_{\pi^{(t)}}\left[\left\langle f_h^{(t)}\left(s_h, \cdot\right), \pi_h^{f^{(t)}}\left(\cdot \mid s_h\right)-\pi_h^{(t)}\left(\cdot \mid s_h\right)\right\rangle\right] \nonumber\\
        &\qquad+ \sum_{k=1}^{K}\sum_{h=1}^H \mathbb{E}_{\pi^{(t)}}\left[\left\langle f_h^{(t_k)}\left(s_h, \cdot\right), \pi_h^{f^{(t)}}\left(\cdot \mid s_h\right)-\pi_h^{(t_k)}\left(\cdot \mid s_h\right)\right\rangle\right],
 \end{align}
 and the first term of (\ref{eqn:nora-regret-decom-4}) is bounded by Lemma \ref{lem:policy-tracking-ft-t-nora} with summing over all $t_1,\ldots,t_K$ as
 \begin{align}\label{eqn:nora-regret-decom-41}
     \sum_{k=1}^{K}\sum_{t=t_k+1}^{t_{k+1}-1} \sum_{h=1}^H \E_{d_h^{(t)}}\left[\left\langle f_{h+1}^{(t)}(s', \cdot) , \pi_{h+1}^{f^{(t)}}(\cdot | s') - \pi_{h+1}^{(t)} (\cdot | s')\right\rangle\right] \leq \eta H^3 T/2 + \frac{KH\log|\gA|}{\eta} + KH^2,
 \end{align}
 while the second term of (\ref{eqn:nora-regret-decom-4}) can be bounded similar to (\ref{eqn:nora-regret-decom-12}) as
  \begin{align}\label{eqn:nora-regret-decom-42}
   \sum_{k=1}^{K}\sum_{h=1}^H \mathbb{E}_{\pi^{(t)}}\left[\left\langle f_h^{(t_k)}\left(s_h, \cdot\right), \pi_h^{f^{(t)}}\left(\cdot \mid s_h\right)-\pi_h^{(t_k)}\left(\cdot \mid s_h\right)\right\rangle\right]\leq 2KH^2.
 \end{align}
  It only remains to bound the second term and third term of (\ref{eqn:nora-regret-decom}). We note that targeting $Q^*$ yields sufficient optimism to assert that $f_h^{(t)} \geq \gT_h^{\pi^{*}} f_{h+1}^{(t)}$ in Lemma \ref{lem:ac-loss-bounded-optimism}, and so we argue that the second term in (\ref{eqn:nora-regret-decom}) for the negative Bellman error under $\pi^*$ is nonpositive in Lemma \ref{lem:negative-bellman-decomp},
  \begin{align}\label{eqn:nora-regret-decom-2}
   -\sum_{t=1}^T \sum_{h=1}^H\mathbb{E}_{\pi^*}\left[\left(f_h^{(t)}-\gT_h^{\pi^{*}}f_{h+1}^{(t)}\right)\left(s_h, a_h\right)\right]\leq 0
  \end{align}
  The third term via a standard argument from \citet{xie2022role} in Lemma \ref{lem:occ-measure-regret-nora},
  \begin{align}\label{eqn:nora-regret-decom-3}
      \sum_{t=1}^T \sum_{h=1}^H\mathbb{E}_{\pi^{(t)}}\left[\left(f_h^{(t)}-\gT_h f_{h+1}^{(t)}\right)\left(s_h, a_h\right)\right] \leq \sqrt{\beta H^4 T\mathsf{SEC}(\gF, \Pi, T)}.
  \end{align}
Plugging (\ref{eqn:nora-regret-decom-1}) -- (\ref{eqn:nora-regret-decom-3}) into (\ref{eqn:nora-regret-decom}), we obtain that
\begin{align}
    \mathrm{Reg}(T) &\leq \eta H^3 T + \frac{2KH\log|\gA|}{\eta} + 6KH^2 + \sqrt{\beta H^4 T\mathsf{SEC}(\gF, \Pi, T)} \nonumber\\
    & \leq \eta H^3 T +  \frac{2dH^2\log T\log|\gA|}{\eta} + 6dH^3\log T +  \sqrt{\beta H^4 T\mathsf{SEC}(\gF, \Pi, T)}\nonumber\\
    & = 2\sqrt{2dH^5T\log T\log|\gA|} +6dH^3\log T +  \sqrt{\beta H^4 T\mathsf{SEC}(\gF, \Pi, T)}
\end{align}
 where the second inequality is obtained by bounding the number of critic updates by $dH\log T$ in Lemma \ref{lem:switch-cost} with the techniques of \citet{xiong2023generalframeworksequentialdecisionmaking}, and the last equality holds when we set the learning rate $\eta = \Theta(\sqrt{d\log T \log |\gA| H^{-1}T^{-1}})$. Here we conclude the proof.
\end{proof}

\paragraph{Quality of the regret bound.} 
Even compared to the ``good case'' for Algorithm \ref{alg:DOUHUA} in Corollary \ref{cor:good-case-douhua}, Algorithm \ref{alg:NORA} requires only $dH \log T$ more samples to learn an $\epsilon$-optimal policy. This is exactly the same as the switch cost, in line with what \citet{cassel2024warmupfreepolicyoptimization} observed for linear MDPs.

It is known that the Bellman eluder dimension may scale unfavorably with $T$ in rare cases \citep{xie2022role} where $d = d_{\text{BE}}(\gF, D_\Delta, 1/\sqrt{T}) = \Omega(\sqrt{T})$. Although we use the $\mathsf{SEC}$ of \citet{xie2022role} to bound the online regret whenever possible, as it is never larger than $O(d\log T)$, our result still depends on $d$ due to the switch cost. Still, function classes with low Bellman eluder dimension are ubiquitous, and so Algorithm \ref{alg:NORA} achieves $\sqrt{T}$ regret on a large class of problems including linear and kernel MDPs \citep{jin2021bellman}. We note that it is possible to weaken the dependence on the $d_\Delta$-type Bellman eluder dimension to the $\ell_2$ eluder condition of \cite{xiong2023generalframeworksequentialdecisionmaking} with nothing more than a change of notation, but we present our results in the language of the Bellman eluder dimension for ease of presentation.

\paragraph{Comparison with prior work.} The only other method we know of that claims to have achieved $\sqrt{T}$ regret with policy optimization and general function approximation is \citet{zhou2023offlinedataenhancedonpolicy}. However, they allow themselves to collect $\Omega(T)$ samples \textit{at every timestep $t$ without incurring any additional regret}. This incurs a sample complexity of $1/\epsilon^6$, while Algorithm \ref{alg:NORA} enjoys a $1/\epsilon^2$ gurarantee in comparison. As such, their regret bound is not sublinear in the more common setup where this sampling contributes to the regret. 
\citet{liu2023optimisticnaturalpolicygradient} achieve a result of $1/\epsilon^3$ by performing critic updates in batches. However, \citet{liu2023optimisticnaturalpolicygradient} do not account for the potential growth of the policy class in the $O(\sqrt{T})$ critic updates they make. Their result of $1/\epsilon^3$ sample complexity is therefore somewhat optimistic, and may very well be $1/\epsilon^4$.

\paragraph{Extension to other policy updates.} We are able to accommodate other policy optimization updates other than the multiplicative weights update, if they satisfy the following:

\begin{cor}
    If there exists some policy optimization oracle $\pi^{(t+1)} \gets \mathrm{PO}(\pi^{(t)}, f^{(t)}, \eta)$ and some $\mathrm{OPT}^*$ and $\mathrm{OPT}^f$ such that
    \begin{subequations}
    \begin{align}
        &\sum_{t=1}^T \sum_{h=1}^H \mathbb{E}_{\pi^*}\left[\left\langle f_h^{(t)}(s, \cdot), \pi_h^*(\cdot | s)-\pi_h^{(t)}(\cdot | s)\right\rangle\right] \leq \mathrm{OPT}^*,\\
    &\sum_{t=1}^T \sum_{h=1}^H \mathbb{E}_{\pi^{(t)}}\left[\left\langle f_h^{(t)}(s, \cdot), \pi_h^{f^{(t)}}(\cdot | s)-\pi_h^{(t)}(\cdot | s)\right\rangle\right] \leq \mathrm{OPT}^f,
    \end{align}
    \end{subequations}
    then Algorithm \ref{alg:NORA} with this policy update obtains a regret of
    \begin{equation}\mathrm{Reg}(T) = O\left(\sqrt{\beta H^4 T \mathsf{SEC}(\gF, \Pi, T)} + \mathrm{OPT}^* + \mathrm{OPT}^f\right).\end{equation}
\end{cor}

\subsection{Further comments on the design of NORA}

\paragraph{Targeting $Q^*$.} We are not the only ones who consider a critic that targets $Q^*$ instead of $Q^{\pi^{(t)}}$. By way of illustration, \citet{crites1994actorcriticqlearning} propose an actor-critic algorithm that mimics Q-learning via a critic that targets $Q^*$. Another example is the popular DDPG algorithm \citep{lillicrap2019continuouscontroldeepreinforcement} and its successor, TD3 \citep{fujimoto2018addressingfunctionapproximationerror}, that take turns to update an approximately greedy deterministic policy that approximately maximizes an estimate of the approximately greedy deterministic policy's Q-function (and so approximately targets $Q^*$), while performing stochastic exploration that tracks the deterministic policy over time. 

\paragraph{Similarity to deep deterministic policy gradients (DDPG).} Several aspects of the design of Algorithm \ref{alg:NORA} are reminiscent of the popular DDPG algorithm of \citet{lillicrap2019continuouscontroldeepreinforcement}. DDPG alternates between optimizing a deterministic policy $g_h^{(t)}(s)$ and a critic $f_h^{(t)}(s,a)$:
\begin{align}
    g_h^{(t)}(s) &\approx \argmax_{a\in\gA} f_h^{(t)}(s,a), \\
    f_h^{(t)}(s,a) &\approx \argmin \gL_h^{(t,g^{(t-k)})}(f_h,f_{h+1}^{(t-k)}),
    \label{eqn:ddpg}
\end{align}
while exploring according to a stochastic exploration policy that tracks the deterministic policy. In practice, the exploration policy $\pi_{\text{exp}}^{(t)}(a|s) \propto \mathbbm{1}(a = g^{(t)}(s)) + \gN_t$ is often given by the deterministic function perturbed by Gaussian noise. The critic is optimized according to slowly updating targets approximating the actor and critic from a few steps ago -- in a manner reminiscent of the rare-switching critic update in Algorithm \ref{alg:NORA}. As the deterministic policy approximates the greedy policy, and the approximately greedy policy is used when computing TD error targets, the DDPG critic approximately targets $Q^*$. This suggests that DDPG and its successor TD3 \citep{fujimoto2018addressingfunctionapproximationerror} may be useful backbones for adapting algorithmic insights garnered from the design and analysis of Algorithm \ref{alg:NORA} for practical RL.\footnote{This provides insight on why DDPG/TD3 are more sample-efficient than on-policy PPO. Algorithm \ref{alg:NORA} needs $1/\epsilon^2$ samples, while on-policy sampling needs at least $1/\epsilon^4$ \citep{liu2023optimisticnaturalpolicygradient}}

\section{Extension to hybrid RL}

In this section, we demonstrate that the benefits of having access to both offline and online data are two-fold with actor-critic algorithms. With optimism, actor-critic algorithms achieve sample efficiency gains. However, if one does not wish to use an optimistic algorithm (as optimism is difficult to implement in practice with deep function approximation), having access to offline data allows one to omit the use of optimism, achieving computational efficiency gains (e.g.~\cite{jin2021pessimism,li2024settling,rashidinejad2023bridging}). 



\subsection{Computational efficiency through optimism-free hybrid RL}

The benefits of hybrid RL extend beyond sample efficiency gains. \citet{song2023hybrid, amortila2024harnessing, zhou2023offlinedataenhancedonpolicy} show that it is possible for a hybrid RL algorithm to bypass the need for optimism altogether, as long as the offline dataset achieves sufficient coverage. In this spirit, we provide an algorithm, NOAH, with two variants in Algorithms \ref{alg:NOAH-pi} and \ref{alg:NOAH-star} that achieve $\sqrt{T}$ regret without using optimism. Algorithm \ref{alg:NOAH-pi} (NOAH-$\pi$) targets $\pi^{(t)}$, and follows the very natural procedure of performing a critic update via Fitted Q-Evaluation (FQE) \citep{munos2008finite} and an actor update in every episode. It therefore requires closure of the critic function class under truncated sums as in Definition \ref{defn:closure-truncated-sums} to control the growth of the policy class as in Lemma \ref{lem:policy-class-growth-sums}. Algorithm \ref{alg:NOAH-star} (NOAH-$*$), like NORA in Algorithm \ref{alg:NORA}, circumvents this by targeting $\pi^*$ and performing a rare-switching critic update. Both algorithms are fully off-policy, utilizing offline data and all collected online data without throwing any away. 

We use the following form of the single-policy concentrability coefficient, tweaked from that of \citet{zhan2022offline}, that has some resemblance to the transfer coefficient of \cite{song2023hybrid}:
\begin{defn}[Single-Policy Concentrability Coefficient]
    $$c_{\off}^*(\gF,\Pi) = \max_{h\in[H]}\sup_{f \in \gF} \sup_{\pi \in \Pi} \frac{\E_{\pi^*}\left[f_h-\gT_h^{\pi}f_{h+1} \right]^2}{\E_{\mu}[\left(f_h-\gT_h^{\pi}f_{h+1}\right)^2]}, \;\; c_{\off}^*(\gF) = \max_{h\in[H]}\sup_{f \in \gF} \frac{\E_{\pi^*}\left[f_h-\gT_h f_{h+1} \right]^2}{\E_{\mu}[\left(f_h-\gT_h f_{h+1}\right)^2]}.$$
    \label{defn:single-policy-concentrability}
\end{defn}
Notably, this is similar to the squared Bellman error variant of the single policy concentrability coefficient in the literature, with the exception that the square in the numerator is outside the expectation, and we have an additional supremum over policies that the Bellman operator is taken with respect to in the first definition. 

With the definition of single-policy concentrability coefficient, we provide a guarantee for this procedure below, but defer the proof to Appendices \ref{app:noah-pi-proofs} and \ref{app:noah-star-proofs}. At a high level, the proofs proceed in a similar way to Theorems \ref{thm:regret-bound-douhua} and \ref{thm:nora-regret_bound}, with the exception that the negative Bellman error under $\pi^*$ is bounded using the critic's Bellman error under the offline data.

\begin{algorithm}[t]
    \caption{Non-Optimistic Actor-critic with Hybrid RL targeting $\pi^{(t)}$ (NOAH-$\pi$)}
    \begin{algorithmic}[1]
            \STATE {\bfseries Input:} Function class $\gF$.
            \STATE {\bfseries Initialize:} $\gF^{(0)} \leftarrow \gF$, $\gD_h^{(0)}\leftarrow \emptyset, \forall h \in [H]$, $\eta = \Theta(\sqrt{\log|\gA| H^{-2}T^{-1}})$, $\pi^{(1)} \; \propto \; 1$, confidence width $\beta=\Theta(\log (HT^2 \mathcal{N}_{\gF, \gT \gF}(1/T) / \delta))$.
            \FOR{episode $t = 1, 2, \dots, T$}
            \STATE Play policy $\pi^{(t)}$ for one episode, obtain trajectory, update dataset $\mathcal{D}_h^{(t)}$. 
            \STATE Compute critic $f^{(t+1)}$ targeting $\pi^{(t)}$ via FQE:
            \begingroup
            \addtolength\jot{-1ex}
            \begin{align*}
                &f^{(t)} \leftarrow \argmin_{f \in \gF} \mathcal{L}_h^{(t,\pi^{(t)})}\left(f_h, f_{h+1}\right) \text{ for } h=H-1,...,1,\\
                &\mathcal{L}_h^{(t, \pi^{(t)})}\left(f, f'\right) \gets\sum_{\left(s, a, r, s^{\prime}\right) \in \mathcal{D}_h^{(t)}\cup\gD_{\off}}\left(f(s, a)-r-f^{\prime}(s^{\prime}, \pi_{h+1}^{(t)}(s^\prime))\right)^2.
            \end{align*}
            \endgroup
        \STATE Update $\pi_h^{(t+1)}(a|s) \; \propto \; \pi_h^{(t)}(a|s)\exp(\eta f_h^{(t)}(s,a))$.
        \ENDFOR
    \end{algorithmic}
\label{alg:NOAH-pi}
\end{algorithm}

\begin{algorithm}[t]
    \caption{Non-Optimistic Actor-critic with Hybrid RL targeting $\pi^*$ (NOAH-$*$)}
    \begin{algorithmic}[1]
            \STATE {\bfseries Input:} Function class $\gF$.
            \STATE {\bfseries Initialize:} $\gF^{(0)} \leftarrow \gF$, $\gD_h^{(0)}\leftarrow \emptyset, \forall h \in [H]$, $\eta = \Theta(\sqrt{d\log T \log|\gA| H^{-1}T^{-1}})$, $\pi^{(1)} \; \propto \; 1$, confidence width $\beta=\Theta(\log (HT^2 \mathcal{N}_{\gF, \gT \gF}(1/T) / \delta))$.
            \FOR{episode $t = 1, 2, \dots, T$}
            \STATE Play policy $\pi^{(t)}$ for one episode, obtain trajectory, update dataset $\mathcal{D}_h^{(t)}$. \IF{$\mathcal{L}_h^{(t)}(f_h^{(t)}, f_{h+1}^{(t)}) \geq \min_{f_h' \in \gF_h} \mathcal{L}_h^{(t)}(f_h', f_{h+1}^{(t)}) + 5H^2\beta$ for some $h$} 
            \STATE Compute critic $f^{(t+1)}$ via FQE:
            \begingroup
            \addtolength\jot{-1ex}
            \begin{align*}
                &f^{(t)} \leftarrow \argmin_{f \in \gF} \mathcal{L}_h^{(t)}\left(f_h, f_{h+1}\right) \text{ for } h=H-1,...,1,\\
                &\mathcal{L}_h^{(t)}\left(f, f'\right) \gets\sum_{\left(s, a, r, s^{\prime}\right) \in \mathcal{D}_h^{(t)}\cup\gD_{\off}}\left(f(s, a)-r-\max_{a' \in \gA}f^{\prime}(s^{\prime}, a')\right)^2.
            \end{align*}
            \endgroup
            \STATE Reset policy $\pi^{(t)} \; \propto \; 1$. 
            \STATE Set $t_{\text{last}} \gets t$, $N_{\text{updates}}^{(t)} \gets N_{\text{updates}}^{(t-1)}+ 1$.
            \ELSE
            \STATE Set $N_{\text{updates}}^{(t)} \gets N_{\text{updates}}^{(t-1)}$, $f^{(t+1)} \gets f^{(t)}$.
            \ENDIF
        \STATE Update $\pi_h^{(t+1)}(a|s) \; \propto \; \pi_h^{(t)}(a|s)\exp(\eta f_h^{(t)}(s,a))$.
        \ENDFOR
    \end{algorithmic}
\label{alg:NOAH-star}
\end{algorithm}

\begin{thm}[Regret Guarantee of Algorithm \ref{alg:NOAH-pi}]
    Let $c_{\off}^*(\gF)$ be the single-policy concentrability coefficients defined in Definition \ref{defn:single-policy-concentrability}. Algorithms \ref{alg:NOAH-pi} satisfy w.p. at least $1-\delta$:
    \begin{align}
        \mathrm{Reg}^\pi(T)
        = O\left(\sqrt{H^4 T \log|\gA|} + \sqrt{\beta^\pi H^4 c_{\off}^*(\gF,\Pi)T^2/N_{\off}} + \sqrt{\beta^\pi H^4 T\mathsf{SEC}(\gF, \Pi, T)}\right).
    \end{align}
    where  $\beta^\pi = \Theta\left(\log \left(HT \mathcal{N}_{\gF, (\gT^{\Pi})^T \gF}(1/T) / \delta\right)\right)$.
    \label{thm:hybrid-noah-pi}
\end{thm}

\begin{thm}[Regret Guarantee of Algorithm \ref{alg:NOAH-star}]
    Let  $c_{\off}^*(\gF,\Pi)$ be the single-policy concentrability coefficients defined in Definition \ref{defn:single-policy-concentrability}. Algorithms  \ref{alg:NOAH-star} respectively satisfy w.p. at least $1-\delta$:
    \begin{align}
        \mathrm{Reg}^*(T)
        = O\left(\sqrt{d H^5 T \log T \log|\gA|} + dH^3 \log T + \sqrt{\beta^* H^4 c_{\off}^*(\gF)T^2/N_{\off}} + \sqrt{\beta^* H^4 T  \mathsf{SEC}(\gF, \Pi, T)}\right),
    \end{align}
    where  $\beta^* = \Theta(\log \left(HT^2 \mathcal{N}_{\gF,\gT \gF}(1/T) / \delta\right) )$.
    \label{thm:hybrid-noah-star}
\end{thm}
As such, in exchange for omitting optimism, we simply incur an additional $\sqrt{\beta H^4 c_{\off}^*(\gF)T^2/N_{\off}}$ error term, which amounts to $\sqrt{T}$ regret as long as $N_{\off} = \Omega(T)$. 

Still, our result shows that the provable efficiency achieved without optimism by hybrid FQI-type methods \citep{song2023hybrid} also extends to actor-critic methods. At the same time, this also resolves the issue within \citet{zhou2023offlinedataenhancedonpolicy} of needing to collect $\Theta(T)$ samples at every timestep $t$. We therefore achieve a true $\sqrt{T}$ regret bound and a sample complexity of $1/\epsilon^2$, in contrast to the $1/\epsilon^6$ sample complexity of \citet{zhou2023offlinedataenhancedonpolicy}.

\subsection{Sample efficiency gains with hybrid RL}
We now extend Algorithm \ref{alg:NORA} to leverage both offline and online data, which aligns with the framework of \citet{tan2024natural}. The extension, found in Algorithm \ref{alg:NORA-long}, appends $N_{\off}$ offline samples to the online data at rounds $t=1,...,T$ and minimizes the TD error over the combined dataset when constructing the confidence sets. 

\begin{algorithm}[h]
    \caption{No-regret Optimistic Rare-switching Actor-critic (NORA) for Hybrid RL }
    \begin{algorithmic}[1]
            \STATE {\bfseries Input:} Offline dataset $\gD_{\off}$ which can be the empty set, samples sizes $T$, $N_{\off}$, function class $\gF$ and confidence width $\beta > 0$ 
            \STATE {\bfseries Initialize:} $\gF^{(0)} \leftarrow \gF$, $\gD_h^{(0)}\leftarrow \emptyset, \forall h \in [H]$, $\eta = \Theta(\sqrt{d\log T\log|\gA| H^{-1}T^{-1}})$, $\pi^{(1)} \; \propto \; 1$.
            \FOR{episode $t = 1, 2, \dots, T$}
            \STATE Select critic $f_h^{(t)}(s,a) :=\operatorname{argmax}_{f_h \in \mathcal{F}_h^{(t_{\text{last}})}} f_h\left(s,a\right)$ for all $s,a$.
            \STATE Play policy $\pi^{(t)}$ for one episode and obtain trajectory $(s_1^{(t)}, a_1^{(t)}, r_1^{(t)}), \dots, (s_H^{(t)}, a_H^{(t)}, r_H^{(t)})$.
                \STATE Update dataset $\mathcal{D}_h^{(t)} \leftarrow \mathcal{D}_h^{(t-1)} \cup\{(s_h^{(t)}, a_h^{(t)}, r_h^{(t)}, s_{h+1}^{(t)})\}, \forall h \in[H]$. 
            \IF{there exists some $h$ such that $\mathcal{L}_h^{(t)}(f_h^{(t)}, f_{h+1}^{(t)})-\min _{f_h^{\prime} \in \mathcal{F}_h} \mathcal{L}_h^{(t)}(f_h^{\prime}, f_{h+1}^{(t)}) \geq 5H^2\beta$} 
            \STATE Compute confidence set $\gF^{(t)}$:
            \vspace{-3mm}
            $$
                \mathcal{F}^{(t)} \leftarrow\left\{f \in \mathcal{F}: \mathcal{L}_h^{(t)}\left(f_h, f_{h+1}\right)-\min _{f_h^{\prime} \in \mathcal{F}_h} \mathcal{L}_h^{(t)}\left(f_h^{\prime}, f_{h+1}\right) \leq H^2\beta \quad \forall h \in[H]\right\},
            $$
            $$ \qquad \text{ where } \mathcal{L}_h^{(t)}\left(f, f'\right):=\sum_{\left(s, a, r, s^{\prime}\right) \in \mathcal{D}_h^{(t)} \cup \mathcal{D}_{\off, h}}\left(f(s, a)-r-\max_{a' \in \gA} f^{\prime}(s^{\prime}, a')\right)^2, \forall f \in \gF_{h}, f' \in \mathcal{F}_{h+1}.
            $$
            \STATE Reset policy $\pi^{(t)} \; \propto \; 1$. 
            \STATE Set $t_{\text{last}} := t$, increment number of updates $N_{\text{updates}}^{(t)} := N_{\text{updates}}^{(t-1)}+ 1$.
            \ELSE
            \STATE Set $N_{\text{updates}}^{(t)} := N_{\text{updates}}^{(t-1)}$, $\gF^{(t)} := \gF^{(t-1)}$.
            \ENDIF
           
        \STATE Select policy $\pi_h^{(t+1)} \; \propto \; \pi_h^{(t)}\exp(\eta f_h^{(t)})$.
        \ENDFOR
    \end{algorithmic}
\label{alg:NORA-long}
\end{algorithm}
Before we proceed, it is useful to introduce the partial all-policy concentrability coefficient from \citet{tan2024natural}.
\begin{defn}[Partial All-Policy Concentrability Coefficient]
    For a function class $\gF$ and a partition on the state-action space $\gX = \gS \times \gA \times [H]$, where we denote the offline and online partitions by $\gX_{\off}$ and $\gX_{\on}$, respectively, the partial all-policy concentrability coefficient is given by:
    \begin{align*}
    c_{\off}(\gF\mathbbm{1}_{\gX_{\off}}) 
    &\coloneqq \max_h \sup_{f \in \gF} \frac{\|(f_h - \gT_h f_{h+1}) \mathbbm{1}_{(\cdot, h) \in \gX_{\off}}\|_{2, d_h^\pi}^2}{\|(f_h - \gT_h f_{h+1})\mathbbm{1}_{(\cdot, h) \in \gX_{\off}}\|_{2, \mu_{h}}^2},
    \end{align*}
    where $\mathbbm{1}_{\gX_{\off}}$ denotes the indicator variable for whether $s,a\in \gX_{\on}.$
    \label{defn:partial-all-policy-concentrability}
\end{defn}

This yields the following guarantee:

\begin{thm}[Hybrid RL Regret Bound for NORA]
\label{thm:hybrid-nora-regret_bound}
Let $c_{\off}(\gF\mathbbm{1}_{\gX_{\off}})$ be the partial all-policy concentrability coefficient defined in Definition \ref{defn:partial-all-policy-concentrability} and let $\gX_{\off}, \gX_{\on}$ be an arbitrary partition over $\gX = \gS \times \gA \times [H]$. 
Algorithm \ref{alg:NORA-long} satisfies with probability at least $1-\delta$:
\begin{align}
    &\mathrm{Reg}(T)
    = \gO\left(\inf_{\gX_{\on}, \gX_{\off}} \left(\sqrt{\beta H^4c_{\off}(\gF\mathbbm{1}_{\gX_{\off}})T^2/N_{\off}} + \sqrt{\beta H^4 T \mathsf{SEC}(\gF \mathbbm{1}_{\gX_{\on}}, \Pi, T)} + \sqrt{dH^5T \log|\gA|}\right)\right),
\end{align}
where 
$\beta = C\log \left(HN \mathcal{N}_{\mathcal{\gF, \gT\gF}}(1/N) / \delta\right)$ for some constant $C$, $N = N_{\off} + T$, $\mathbbm{1}_{\gX_{\off}}, \mathbbm{1}_{\gX_{\on}}$ are indicator variables for whether $s,a\in\gX_{\off}$ or $\gX_{\on}$, and $c_{\off}(\gF\mathbbm{1}_{\gX_{\off}})$ is the partial all-policy concentrability coefficient \citep{tan2024natural}.
\end{thm}

We defer further details to Appendix \ref{app:hybrid-rl-nora}. On the high level, the critic error is split into an offline and an online term. One uses the coverage of the offline data to bound the former and online exploration for the latter. Algorithm \ref{alg:NORA-long} is completely unaware of the partition, but the regret bound optimizes over all partitions of the state-action space. Should the offline data be plentiful and of good coverage, the regret approximately becomes $\sqrt{\beta H^4 T \mathsf{SEC}(\gF \mathbbm{1}_{\gX_{\on}}, \Pi, T)} + \sqrt{dH^5T \log|\gA|}$ for some small $\gX_{\on}$, yielding improvements over Theorem \ref{thm:nora-regret_bound}.
This shows that actor-critic methods can benefit from hybrid data, achieving the provable gains in sample efficiency compared to offline-only and online-only learning observed
by \citet{li2023reward, tan2024hybridreinforcementlearningbreaks}.

This guarantee in Theorem \ref{thm:hybrid-nora-regret_bound} for Algorithm \ref{alg:NORA-long} is stronger than the result in Theorem \ref{thm:hybrid-noah-star} for Algorithm \ref{alg:NOAH-star}. As mentioned earlier, if the offline data is plentiful and enjoys good coverage, the regret in Theorem \ref{thm:hybrid-nora-regret_bound} primarily comes from the $\sqrt{dH^5 T \log|\gA|}$ cost of optimizing the actor as $\mathsf{SEC}(\gF \mathbbm{1}_{\gX_{\on}}, \Pi, T)$ is small if $\gX_{\on}$ is small.\footnote{Which can possibly be mitigated as well by performing offline policy optimization.} On the other hand, Algorithm \ref{alg:NOAH-star} incurs a term depending on $\mathsf{SEC}(\gF, \Pi, T)$, which can be much larger than $\mathsf{SEC}(\gF \mathbbm{1}_{\gX_{\on}}, \Pi, T)$.

\section{Numerical experiments}


We provide two numerical experiments to empirically verify our findings. Details on reproducing our findings are deferred to Appendix \ref{app:experiment_details}.

\paragraph{Optimism in linear MDPs.} The first experiment examines Algorithms \ref{alg:DOUHUA} and \ref{alg:NORA} in a linear MDP setting, in order to validate if they indeed achieve $\sqrt{T}$ regret in practice. Accordingly, we implement optimism with LSVI-UCB bonuses \citep{jin2019provably} instead of global optimism as in GOLF.

We compare our algorithms to a rare-switching\footnote{Implemented with the doubling determinant method used in \cite{he2023nearly}. This incurs at most a $dH\log T$ switching cost.} version of LSVI-UCB on a linear MDP tetris task \citep{tan2024hybridreinforcementlearningbreaks, tan2024natural}. The condition required for Algorithm \ref{alg:DOUHUA} to work holds here, as we do not clip the Q-function estimates. 
Figure~\ref{fig:median-reward} shows that Algorithm \ref{alg:DOUHUA} (surprisingly) performs better than LSVI-UCB, and Figure~\ref{fig:cumulative-regret} empirically illustrates that Algorithm \ref{alg:NORA} also achieves $\sqrt{T}$ regret even though it performs slightly worse than LSVI-UCB.

\begin{figure}[h]
    \centering
    \includegraphics[width=0.8\linewidth]{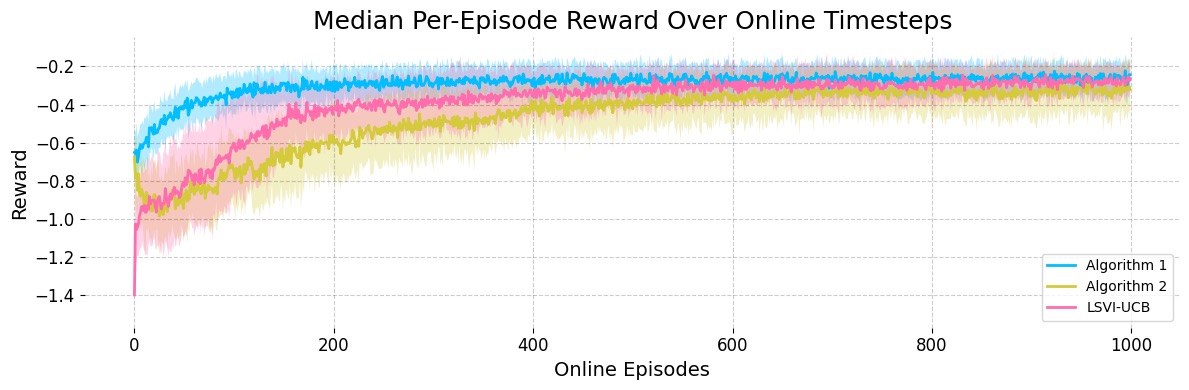}
    \caption{Per-episode reward of Algorithms 1 and 2, compared to a rare-switching version of LSVI-UCB (Jin et al., 2019) on a linear MDP tetris task. Here, the condition required for Algorithm 1 to work holds, as we do not clip the Q-function estimates. Algorithm 1 outperforms LSVI-UCB, and Alg. 2 catches up after some time. Results averaged over 30 trials.}
    \label{fig:median-reward}
\end{figure}

\begin{figure}[h]
    \centering
    \includegraphics[width=0.8\linewidth]{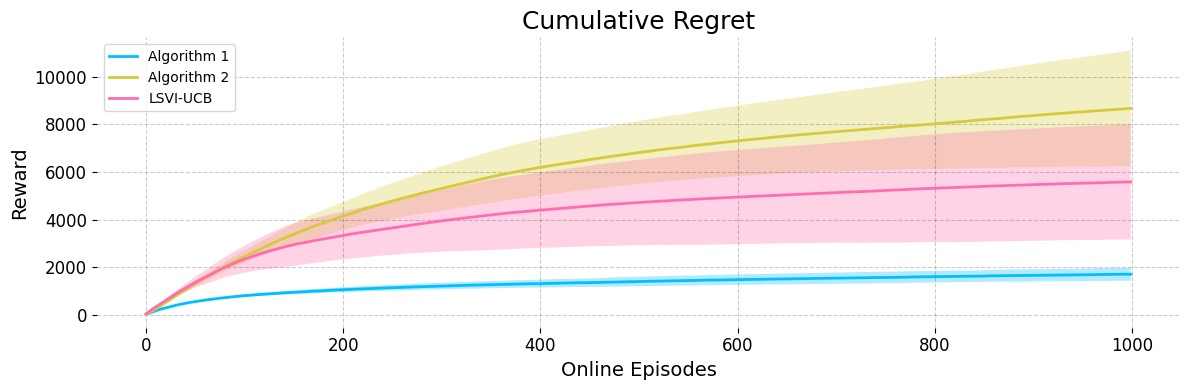}
    \caption{Cumulative regret of Algorithms 1 and 2 vs. LSVI-UCB (Jin et al., 2019) on Tetris. This setting favors Alg. 1, which outperforms LSVI-UCB; Alg. 2 also achieves $\sqrt{T}$ regret. Averaged over 30 trials.}
    \label{fig:cumulative-regret}
\end{figure}

\paragraph{Deep hybrid RL.} We compare variants of Algorithms \ref{alg:NOAH-pi} and \ref{alg:NOAH-star}, that we call Algorithms 1H and 2H respectively, to Cal-QL \citep{nakamoto2023calql}. The differences are as follows. Algorithms 1H and 2H employ offline pretraining of both the actor and the critic, and the actor employs the soft actor-critic (SAC) update from \cite{haarnoja2018softactorcriticoffpolicymaximum}. Additionally, Algorithm 2H omits the policy resets. Like Cal-QL, and unlike Algorithm 1H, Algorithm 2H takes a maximum over 10 randomly sampled actions from the policy to enhance exploration. This can be viewed as a computationally efficient version of optimism, or alternatively as approximately targeting $\pi^*$. 

Algorithm 2H outperforms Cal-QL, which in turn slightly outperforms Algorithm 1H. These results suggest that Algorithms 1H and 2H remain highly competitive, and may perform just as well as the state of the art in hybrid RL (Cal-QL) -- even without the use of pessimism. Our results support our theoretical findings that hybrid RL allows for computationally efficient actor-critic algorithms.

\begin{figure}[H]
    \centering
    \includegraphics[width=0.8\linewidth]{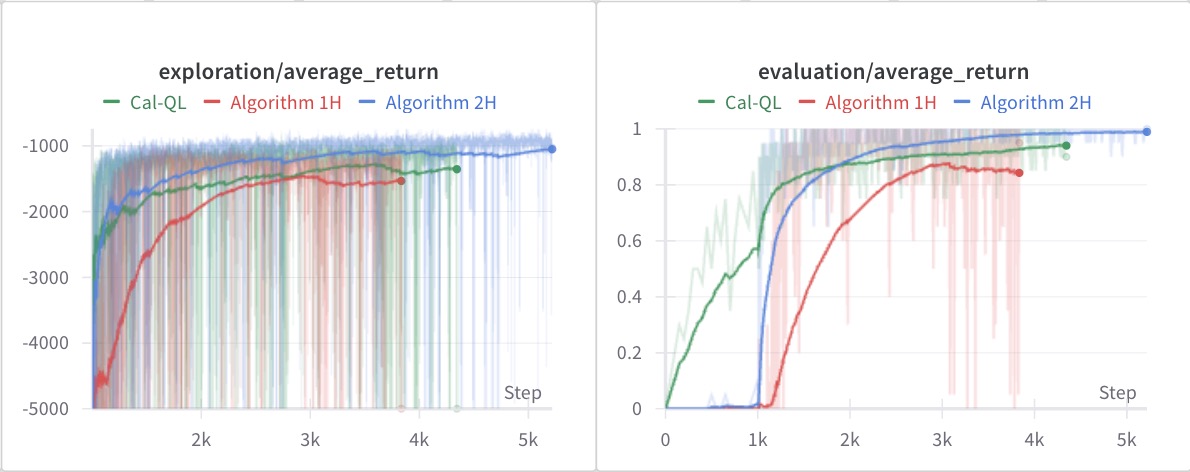}
    \caption{Hybrid RL experiment on the \texttt{antmaze-medium-diverse-v2} task. Comparison of Alg. 1H, Alg. 2H, and Cal-QL (Nakamoto et al., 2023) with offline pretraining. Alg. 2H uses approximate optimism via action sampling and outperforms both Cal-QL and Alg. 1H, despite no pessimism. Results suggest hybrid RL enables efficient exploration without pessimism. Evaluation plots show offline pretraining in the first 1000 steps. All plots are exponentially smoothed.}
    \label{fig:average-return}
\end{figure}

\section{Conclusions and future work}

To conclude, we resolve an open problem in the online RL literature by designing an actor-critic method NORA (Algorithm \ref{alg:NORA}) with general function approximation that achieves $1/\epsilon^2$ (and matching $\sqrt{T}$ regret) without making any reachability or coverage assumptions. This was achieved through several key ingredients in the algorithm design. By (1) performing optimistic exploration, (2) learning a critic off-policy without throwing away any data, (3) having the critic target $Q^*$ instead of $Q^{\pi^{(t)}}$ to ensure sufficient optimism and enable rare-switching critic updates, and (4) performing policy resets to control the deviation of the current policy from the greedy policy, we achieve the desired result in Theorem \ref{thm:nora-regret_bound}.

We resolve another open problem in the hybrid RL literature by providing a non-optimistic provably efficient actor-critic
algorithm that only additionally requires $N_{\off} \geq c_{\off}^*(\gF,\Pi)dH^4/\epsilon^2$ offline samples in exchange for omitting optimism. This, along with the result in Theorem \ref{thm:hybrid-nora-regret_bound}, shows that hybrid RL therefore allows for sample efficiency gains with optimistic algorithms and computational efficiency gains with non-optimistic algorithms. 

While Algorithm \ref{alg:NORA} is generally computationally inefficient, the insights gained are promising for empirical applications. 
Optimism is necessary for strategic exploration, and can be implemented via linear model bonuses \citet{agarwal2020pcpgpolicycoverdirected}, count-based exploration \citet{martin2017countbasedexplorationfeaturespace}, randomized value functions \citet{osband2019deepexplorationrandomizedvalue}, or random latent exploration \citet{mahankali2024randomlatentexplorationdeep}. An off-policy critic can be learned with the DDQN algorithm of \citet{vanhasselt2015deepreinforcementlearningdouble}, with the insight that one only needs to update when the TD error is high. Finally, an optimistic version of TD3 \citep{fujimoto2018addressingfunctionapproximationerror} with rare-switching critic updates and policy resets when necessary is a promising approach for a practical adaptation of Algorithm \ref{alg:NORA} in future work. Further numerical experiments and an extensive empirical study would constitute a very welcome contribution to the literature.

One may wish to incorporate bonus function classes as in \cite{agarwal2022voql, zhao2023nearlyoptimallowswitchingalgorithm}. With a rarely-updated bonus function class, updating the base critic every episode, one may be able to achieve sufficient optimism without targeting $Q^*$, but only if the sum of Q-functions is close to a scaled Q-function. 
Lastly, adapting the work of \citet{he2023nearly} to see if policy optimization can be minimax-optimal in linear MDPs, like \citet{wu2022nearlyoptimalpolicyoptimization} do for tabular MDPs, is a welcome contribution to the literature. 



\section*{Acknowledgements}

The authors are supported in part by the NSF grants CCF-2106778, CCF-2418156 and CAREER award DMS-2143215.

\appendix

\section{Proofs for Theorem \ref{thm:regret-bound-douhua}}
\label{app:proof-douhua}

We prove a regret bound for Algorithm \ref{alg:DOUHUA} here. This algorithm achieves sublinear regret only when the covering number of the function class does not increase linearly with the number of critic updates. 

\subsection{Regret decomposition}

We first work with the following regret decomposition below. This is the same regret decomposition as that of \citet{cai2024provablyefficientexplorationpolicy} and \citet{zhong2023theoreticalanalysisoptimisticproximal}, and we provide the proof for completeness. 

\begin{lem}[Regret Decomposition For $Q^\pi$-Targeting Actor-Critics]
The regret at time $T$ yields the following decomposition 
\begin{align*}
    \operatorname{Reg}(T)
    = & \sum_{t=1}^T \sum_{h=1}^H \mathbb{E}_{\pi^*}\left[\left\langle f_h^{(t)}\left(s_h, \cdot\right), \pi_h^*\left(\cdot \mid s_h\right)-\pi_h^{(t)}\left(\cdot \mid s_h\right)\right\rangle\right] -\sum_{t=1}^T \sum_{h=1}^H\mathbb{E}_{\pi^*}\left[\left(f_h^{(t)}-\gT_h^{\pi^{*}}f_{h+1}^{(t)}\right)\left(s_h, a_h\right)\right] \\
    &\qquad+ \sum_{t=1}^T \sum_{h=1}^H\mathbb{E}_{\pi^{(t)}}\left[\left(f_h^{(t)}-\gT_h^{\pi^{(t)}} f_{h+1}^{(t)}\right)\left(s_h, a_h\right)\right].
\end{align*}
\label{lem:regret-decomp-douhua-ac}
\end{lem}
\begin{proof}

    By adding and subtracting $f_1^{(t)}(s_1^{(t)}, \pi_1^{(t)}(s_1^{(t)}))$, we obtain
    \begin{align}\label{equ:regret-decompose}
        \text{Reg}(T) &= \sum_{t=1}^T\left(V_1^{*}(s_1^{(t)})-V_1^{\pi^{(t)}}(s_1^{(t)})\right) \nonumber\\
        &= \sum_{t=1}^T \left(V_1^{*}(s_1^{(t)})-f_1^{(t)}(s_1^{(t)}, \pi_1^{(t)}(s_1^{(t)}))\right) + \sum_{t=1}^T\left(f_1^{(t)}(s_1^{(t)}, \pi_1^{(t)}(s_1^{(t)}))-V_1^{\pi^{(t)}}(s_1^{(t)})\right) .
    \end{align}
    
    We can now apply the value difference lemma/generalized policy difference lemma in Lemma \ref{lem:policy-value-difference-lemma} \citep{cai2024provablyefficientexplorationpolicy, efroni2020optimisticpolicyoptimizationbandit} with $f^{(t)}$ as the Q-function, $\pi'=\pi^*$, and $\pi = \pi^{(t)}$ to find that
    \begin{align}\label{eqn:decomp-Qpi-1}
        &\sum_{t=1}^T \left(V_1^{*}(s_1^{(t)})-f_1^{(t)}(s_1^{(t)}, \pi_1^{(t)}(s_1^{(t)}))\right) \nonumber\\
        &= \sum_{t=1}^T \sum_{h=1}^H \mathbb{E}_{\pi^*}\left[\left\langle f_h^{(t)}\left(s_h, \cdot\right), \pi_h^*\left(\cdot \mid s_h\right)-\pi_h^{(t)}\left(\cdot \mid s_h\right)\right\rangle\right] -\sum_{t=1}^T \sum_{h=1}^H\mathbb{E}_{\pi^*}\left[\left(f_h^{(t)}-\gT_h^{\pi^{*}}f_{h+1}^{(t)}\right)\left(s_h, a_h\right)\right].
    \end{align}

    Another application of Lemma \ref{lem:policy-value-difference-lemma} with $\pi = \pi' = \pi^{(t)}$ and with $f^{(t)}$ as the Q-function yields
    \begin{align}\label{eqn:decomp-Qpi-2}
        &\sum_{t=1}^T\left(f_1^{(t)}(s_1^{(t)}, \pi_1^{(t)}(s_1^{(t)}))-V_1^{\pi^{(t)}}(s_1^{(t)})\right) \nonumber \\
        &= \sum_{t=1}^T \sum_{h=1}^H \mathbb{E}_{\pi^{(t)}}\left[\left\langle f_h^{(t)}\left(s_h, \cdot\right), \pi_h^{(t)}\left(\cdot \mid s_h\right)-\pi_h^{(t)}\left(\cdot \mid s_h\right)\right\rangle\right] + \sum_{t=1}^T \sum_{h=1}^H\mathbb{E}_{\pi^{(t)}}\left[\left(f_h^{(t)}-\gT_h^{\pi^{(t)}}f_{h+1}^{(t)}\right)\left(s_h, a_h\right)\right] \nonumber \\
        &= \sum_{t=1}^T \sum_{h=1}^H\mathbb{E}_{\pi^{(t)}}\left[\left(f_h^{(t)}-\gT_h^{\pi^{(t)}}f_{h+1}^{(t)}\right)\left(s_h, a_h\right)\right].
    \end{align}
    Plugging~(\ref{eqn:decomp-Qpi-1}) and~(\ref{eqn:decomp-Qpi-2}) into~(\ref{equ:regret-decompose}) concludes the proof.
\end{proof}

\subsection{Bounding the tracking error}

The first term is bounded by $\sqrt{H^4T\log|\gA|}$, as we see in the following lemma. This lemma is the standard mirror descent regret bound, and is a similar argument to lemmas found in \citet{liu2023optimisticnaturalpolicygradient} and \citet{cai2024provablyefficientexplorationpolicy}.

\begin{lem}[Mirror Descent Tracking Error for Algorithm \ref{alg:DOUHUA}]
    Let $\pi^{(1)} \propto 1$ and consider updating policies with respect to a set of Q-function estimates $f^{(1)},...,f^{(T)}$ by
    \begin{align}\label{eqn:mwu-update}\pi_{h+1}^{(t+1)}(\cdot | s') = \frac{\pi_{h+1}^{(t)}(\cdot | s') \exp(\eta f_{h+1}^{(t)}(s', \cdot))}{\sum_{a \in \gA} \pi_{h+1}^{(t)}(a | s') \exp(\eta f_{h+1}^{(t)}(s', a))} = Z^{-1}\pi_{h+1}^{(t)}(\cdot | s') \exp(\eta f_{h+1}^{(t)}(s', \cdot)).\end{align}
        The tracking error with respect to the optimal policy is then bounded by:
    $$\sum_{t=1}^{T} \sum_{h=1}^H \mathbb{E}_{\pi^*}\left[\left\langle f_h^{(t)}\left(s_h, \cdot\right), \pi_h^*\left(\cdot \mid s_h\right)-\pi_h^{(t)}\left(\cdot \mid s_h\right)\right\rangle\right] \leq \eta H^3 T/2 + \frac{H \log |\gA|}{\eta}.$$
    \label{lem:mirror-descent-douhua}
\end{lem}
\begin{proof}
Rearranging~(\ref{eqn:mwu-update}) yields
$$\eta f_{h+1}^{(t)}(s', \cdot) = \log Z + \log \pi_{h+1}^{(t+1)}(\cdot | s') - \log \pi_{h+1}^{(t)}(\cdot | s'),$$
where $\log Z$ is
$$\log Z = \log\left(\sum_{a \in \gA} \pi_{h+1}^{(t)}(a | s') \exp(\eta f_{h+1}^{(t)}(s', a))\right).$$

We can now bound, noting that $\sum_{a \in \gA} \left(\pi_{h+1}^{*}(\cdot | s') - \pi_{h+1}^{(t+1)}(\cdot | s')\right) = 0$, that
\begin{align}\label{eqn:mirror-alg1-tel1}
    &\left\langle \eta f_{h+1}^{(t)}(s', \cdot), \pi_{h+1}^{*}(\cdot | s') - \pi_{h+1}^{(t+1)}(\cdot | s')\right\rangle \nonumber\\
    &= \left\langle \log Z + \log \pi_{h+1}^{{(t+1)}}(\cdot | s') - \log \pi_{h+1}^{(t)}(\cdot | s'), \pi_{h+1}^{*}(\cdot | s') - \pi_{h+1}^{(t+1)}(\cdot | s')\right\rangle\nonumber \\ 
    &=\left\langle \log \pi_{h+1}^{{(t+1)}}(\cdot | s') - \log \pi_{h+1}^{(t)}(\cdot | s'), \pi_{h+1}^{*}(\cdot | s') - \pi_{h+1}^{(t+1)}(\cdot | s')\right\rangle\nonumber \\
    &= \text{KL}\left(\pi_{h+1}^{*}(\cdot | s')\; || \;\pi_{h+1}^{(t)}(\cdot | s')\right) - \text{KL}\left(\pi_{h+1}^{*}(\cdot | s')\; || \;\pi_{h+1}^{(t+1)}(\cdot | s')\right) - \text{KL}\left(\pi_{h+1}^{(t+1)}(\cdot | s')\; || \;\pi_{h+1}^{(t)}(\cdot | s')\right),
\end{align}
where the last line follows directly from Lemma~\ref{lem: KL} by setting $\pi=\pi_{h+1}^{(t+1)}(\cdot|s')$, $\pi_1 = \pi_{h+1}^{\star}(\cdot|s')$ and $\pi_2 = \pi_{h+1}^{(t)}(\cdot|s')$. As a result, we can bound the desired inner product as follows,
\begin{align}\label{eqn:mirror-alg1-tel2}
    &\left\langle \eta f_{h+1}^{(t)}(s', \cdot), \pi_{h+1}^{*}(\cdot | s') - \pi_{h+1}^{(t)}(\cdot | s')\right\rangle \nonumber \\
    &= \left\langle \eta f_{h+1}^{(t)}(s', \cdot), \pi_{h+1}^{*}(\cdot | s') - \pi_{h+1}^{(t+1)}(\cdot | s')\right\rangle - \left\langle \eta f_{h+1}^{(t)}(s', \cdot), \pi_{h+1}^{(t)}(\cdot | s') - \pi_{h+1}^{(t+1)}(\cdot | s')\right\rangle \nonumber\\
    &\leq \text{KL}\left(\pi_{h+1}^{*}(\cdot | s')\; || \;\pi_{h+1}^{(t)}(\cdot | s')\right) - \text{KL}\left(\pi_{h+1}^{*}(\cdot | s')\; || \;\pi_{h+1}^{(t+1)}(\cdot | s')\right) - \text{KL}\left(\pi_{h+1}^{(t+1)}(\cdot | s')\; || \;\pi_{h+1}^{(t)}(\cdot | s')\right) \nonumber\\
    &\qquad + \eta H ||\pi_{h+1}^{(t+1)}(\cdot | s') - \pi_{h+1}^{(t)}(\cdot | s')||_1,
\end{align}
where the last line follows directly from~(\ref{eqn:mirror-alg1-tel1}). Summing up the inner product bounded in (\ref{eqn:mirror-alg1-tel2}), we derive that
\begin{align}\label{eqn:mirror-alg1-tel3}
    &\sum_{t=1}^{T} \sum_{h=1}^H \mathbb{E}_{\pi^*}\left[\left\langle f_{h+1}^{(t)}(s', \cdot), \pi_{h+1}^{*}(\cdot | s') - \pi_{h+1}^{(t)}(\cdot | s')\right\rangle \right]\nonumber\\
    &= \frac{1}{\eta} \sum_{t=1}^{T} \sum_{h=1}^H \mathbb{E}_{\pi^*}\left[ \left\langle \eta f_{h+1}^{(t)}(s', \cdot), \pi_{h+1}^{*}(\cdot | s') - \pi_{h+1}^{(t)}(\cdot | s')\right\rangle\right] \nonumber\\
    &\leq \frac{1}{\eta} \sum_{t=1}^{T} \sum_{h=1}^H \mathbb{E}_{\pi^*}\left[\text{KL}\left(\pi_{h+1}^{*}(\cdot | s')\; || \;\pi_{h+1}^{(t)}(\cdot | s')\right) - \text{KL}\left(\pi_{h+1}^{*}(\cdot | s')\; || \;\pi_{h+1}^{(t+1)}(\cdot | s')\right) \right.\nonumber\\
    &\qquad\qquad\qquad \left.- \text{KL}\left(\pi_{h+1}^{(t+1)}(\cdot | s')\; || \;\pi_{h+1}^{(t)}(\cdot | s')\right) + \eta H ||\pi_{h+1}^{(t+1)}(\cdot | s') - \pi_{h+1}^{(t)}(\cdot | s')||_1\right].
\end{align}
Now we apply Pinsker's inequality and note that
\begin{align}\label{eqn:pinsker1}\text{KL}(\pi_{h+1}^{(t+1)}(\cdot|s')||\pi_{h+1}^{(t)}(\cdot|s'))\geq \Vert\pi_{h+1}^{(t+1)}(\cdot|s') - \pi_{h+1}^{(t)}(\cdot|s')\Vert_1^2/2,\end{align}
and plugging~(\ref{eqn:pinsker1}) into~(\ref{eqn:mirror-alg1-tel3}) yields that
\begin{align}
    &\sum_{t=1}^{T} \sum_{h=1}^H \mathbb{E}_{\pi^*}\left[\left\langle f_{h+1}^{(t)}(s', \cdot), \pi_{h+1}^{*}(\cdot | s') - \pi_{h+1}^{(t)}(\cdot | s')\right\rangle \right]\nonumber\\
    &\leq \frac{1}{\eta} \sum_{t=1}^{T}  \sum_{h=1}^H \mathbb{E}_{\pi^*}\left[ \text{KL}\left(\pi_{h+1}^{*}(\cdot | s')\; || \;\pi_{h+1}^{(t)}(\cdot | s')\right) - \text{KL}\left(\pi_{h+1}^{*}(\cdot | s')\; || \;\pi_{h+1}^{(t+1)}(\cdot | s')\right) \right.\nonumber \\
    &\qquad\qquad\qquad\qquad\qquad\qquad  \left. - ||\pi_{h+1}^{(t+1)}(\cdot | s') - \pi_{h+1}^{(t)}(\cdot | s')||_1^2/2 + \eta H ||\pi_{h+1}^{(t+1)}(\cdot | s') - \pi_{h+1}^{(t)}(\cdot | s')||_1 \right] \nonumber\\
    &\leq \frac{1}{\eta} \sum_{t=1}^{T} \sum_{h=1}^H \mathbb{E}_{\pi^*}\left[ \text{KL}\left(\pi_{h+1}^{*}(\cdot | s')\; || \;\pi_{h+1}^{(t)}(\cdot | s')\right) - \text{KL}\left(\pi_{h+1}^{*}(\cdot | s')\; || \;\pi_{h+1}^{(t+1)}(\cdot | s')\right) + \eta^2H^2/2\right],
\end{align}
where we use the fact that $\max_{x\in \R}\left\{-x^2/2 + \eta H x\right\} = \eta^2H^2/2$ in the last line. Continuing to simplify this expression yields
\begin{align}\label{eqn:mirror-alg1-tel4}
    &\sum_{t=1}^{T} \sum_{h=1}^H \mathbb{E}_{\pi^*}\left[\left\langle f_{h+1}^{(t)}(s', \cdot), \pi_{h+1}^{*}(\cdot | s') - \pi_{h+1}^{(t)}(\cdot | s')\right\rangle \right]\nonumber \\
    &\leq \sum_{t=1}^{T} \sum_{h=1}^H \left(\eta H^2/2 + \mathbb{E}_{\pi^*}\left[ \frac{\text{KL}\left(\pi_{h+1}^{*}(\cdot | s')\; || \;\pi_{h+1}^{(t)}(\cdot | s')\right) - \text{KL}\left(\pi_{h+1}^{*}(\cdot | s')\; || \;\pi_{h+1}^{(t+1)}(\cdot | s')\right)}{\eta} \right]\right)\nonumber \\
    &\leq \eta H^3 T/2 + \sum_{h=1}^H \frac{\text{KL}\left(\pi_{h+1}^{*}(\cdot | s')\; || \;\pi_{h+1}^{(1)}(\cdot | s')\right) - \text{KL}\left(\pi_{h+1}^{*}(\cdot | s')\; || \;\pi_{h+1}^{(T)}(\cdot | s')\right)}{\eta}\nonumber\\
    &\leq \eta H^3 T/2 + \frac{H \log |\gA|}{\eta},
\end{align}
where the last inequality follows from the fact that the KL-divergence is non-negative as well as noting that $\pi^{(1)}$ is the uniform policy, so that the KL divergence can be bounded as
\begin{align}\label{eqn:kl-unifotm}
    \text{KL}\left(\pi_{h+1}^{*}(\cdot | s')\; || \;\pi_{h+1}^{(1)}(\cdot | s')\right)& = \sum_{a\in\mathcal{A}}\pi_{h+1}^{*}(a | s')\log\pi_{h+1}^{*}(a | s') - \sum_{a\in\mathcal{A}}\pi_{h+1}^{*}(a | s')\log\pi_{h+1}^{(1)}(a | s')\nonumber\\
    &\leq - \sum_{a\in\mathcal{A}}\pi_{h+1}^{*}(a | s')\log\pi_{h+1}^{(1)}(a | s')\nonumber\\
    & = \log|\mathcal{A}|.
\end{align}

\end{proof}

\subsection{Asserting optimism}

\begin{lem}[Negative Bellman Error, Algorithm \ref{alg:DOUHUA}]
    Within Algorithm \ref{alg:DOUHUA}, it holds that
    $$-\sum_{t=1}^T \sum_{h=1}^H \mathbb{E}_{\pi^*}\left[f_h^{(t)}-\mathcal{T}_h^{\pi^{*}} f_{h+1}^{(t)}\right] \leq  \sum_{t=1}^T \sum_{h=1}^H \mathbb{E}_{\pi^{*}}\left[\left\langle f_{h+1}^{(t)}(s_{h+1},\cdot), \pi_{h+1}^{*} (\cdot | s_{h+1}) -\pi_{h+1}^{(t)}(\cdot | s_{h+1}) \right\rangle\right] + \eta H^3T.$$
    \label{lem:negative-bellman-decomp-douhua}
\end{lem}
\begin{proof}
    First, we decompose the negative bellman error as follows
    \begin{align}\label{eqn:decom-bellman}
        -\sum_{t=1}^T \sum_{h=1}^H \mathbb{E}_{\pi^*}\left[f_h^{(t)}-\mathcal{T}_h^{\pi^{*}} f_{h+1}^{(t)}\right] 
        & = \sum_{t=1}^T \sum_{h=1}^H \mathbb{E}_{\pi^*}\left[\mathcal{T}_h^{\pi^{(t)}} f_{h+1}^{(t)}- f_h^{(t)}\right] + \sum_{t=1}^T \sum_{h=1}^H \mathbb{E}_{\pi^*}\left[\mathcal{T}_h^{\pi^{*}} f_{h+1}^{(t)}- \mathcal{T}_h^{\pi^{(t)}} f_{h+1}^{(t)}\right]\nonumber\\
        & = \sum_{t=1}^T \sum_{h=1}^H \mathbb{E}_{\pi^*}\left[\mathcal{T}_h^{\pi^{(t-1)}} f_{h+1}^{(t)} - f_h^{(t)}\right] + \sum_{t=1}^T \sum_{h=1}^H \mathbb{E}_{\pi^*}\left[\mathcal{T}_h^{\pi^{(t)}} f_{h+1}^{(t)} - \mathcal{T}_h^{\pi^{(t-1)}} f_{h+1}^{(t)}\right]\nonumber\\
        &\qquad\qquad\qquad + \sum_{t=1}^T \sum_{h=1}^H \mathbb{E}_{\pi^*}\left[\mathcal{T}_h^{\pi^{*}} f_{h+1}^{(t)}- \mathcal{T}_h^{\pi^{(t)}} f_{h+1}^{(t)}\right].
    \end{align}

Here, using the result from Lemma~\ref{lem:ac-loss-bounded-optimism-pi} that $f_h^{(t)} \geq \gT_h^{(t-1)} f_{h+1}^{(t)}$, we note that the first term is less than 0. We employ Lemma \ref{lem:policy-difference} with $t_1=1$ and $t_2=T$ to bound the second term by $\eta H^3 T$.
For the third term, we note that
\begin{align}\label{eqn:decom-bellman2}
& \sum_{t=1}^T \sum_{h=1}^H \mathbb{E}_{\pi^*}\left[\mathcal{T}_h^{\pi^{*}} f_{h+1}^{(t)}- \mathcal{T}_h^{\pi^{(t)}} f_{h+1}^{(t)}\right]\nonumber\\
        &= \sum_{t=1}^T \sum_{h=1}^H \mathbb{E}_{\pi^*}\left[r_h(s_h, a_h) + f_{h+1}^{(t)}\left(s_{h+1}, \pi_{h+1}^{*}(s_{h+1})\right) - r_h(s_h, a_h) - f_{h+1}^{(t)}\left(s_{h+1}, \pi_{h+1}^{(t)}(s_{h+1})\right) \right] \nonumber\\
        &= \sum_{t=1}^T \sum_{h=1}^H \mathbb{E}_{\pi^{*}}\left[f_{h+1}^{(t)}\left(s_{h+1}, \pi_{h+1}^{*}(s_{h+1})\right) - f_{h+1}^{(t)}\left(s_{h+1}, \pi_{h+1}^{(t)}(s_{h+1})\right) \right] \nonumber\\
        &=  \sum_{t=1}^T \sum_{h=1}^H \mathbb{E}_{\pi^{*}}\left[\left\langle f_{h+1}^{(t)}(s_{h+1},\cdot), \pi_{h+1}^{*} (\cdot | s_{h+1}) -\pi_{h+1}^{(t)}(\cdot | s_{h+1}) \right\rangle\right].
    \end{align}
    Replacing the last term in (\ref{eqn:decom-bellman}) with (\ref{eqn:decom-bellman2}) concludes the proof.
\end{proof}

\subsection{Bounding the Bellman error under the learned policies}

\begin{lem}[Bellman Error Under Policy Occupancy, Algorithm \ref{alg:DOUHUA}]
    Within Algorithm \ref{alg:DOUHUA}, it holds that
\begin{align*}
    \sum_{t=1}^T \sum_{h=1}^H\mathbb{E}_{\pi^{(t)}}\left[\left(f_h^{(t)}-\gT^{\pi^{(t)}}_h f_{h+1}^{(t)}\right)\left(s_h, a_h\right)\right]\leq \sqrt{\beta H^4 T \mathsf{SEC}(\gF, \Pi, T)} + \eta H^3 T
\end{align*}
    with probability at least $1-\delta$, where $\beta = \Theta\left(\log \left(T G \mathcal{N}_{\mathcal{F},\left(\mathcal{T}^{\mathrm{II}}\right)^T \mathcal{F}}(1 / T) / \delta\right)\right)$. 
    \label{lem:occ-measure-regret-douhua}
\end{lem}
\begin{proof}
We first decompose the Bellman error
\begin{align}\label{eqn:bellman-learned}
    &\sum_{t=1}^T \sum_{h=1}^H\mathbb{E}_{\pi^{(t)}}\left[\left(f_h^{(t)}-\gT^{\pi^{(t)}}_h f_{h+1}^{(t)}\right)\left(s_h, a_h\right)\right] \nonumber\\
    &=\sum_{t=1}^T \sum_{h=1}^H\mathbb{E}_{\pi^{(t)}}\left[\left(f_h^{(t)}-\gT^{\pi^{(t-1)}}_h f_{h+1}^{(t)}\right)\left(s_h, a_h\right)\right] + \sum_{t=1}^T \sum_{h=1}^H\mathbb{E}_{\pi^{(t)}}\left[\left(\gT^{\pi^{(t-1)}}_h f_{h+1}^{(t)}-\gT^{\pi^{(t)}}_h f_{h+1}^{(t)}\right)\left(s_h, a_h\right)\right] \nonumber\\
    &=\sum_{t=1}^T \sum_{h=1}^H\mathbb{E}_{\pi^{(t)}}\left[\left(f_h^{(t)}-\gT^{\pi^{(t-1)}}_h f_{h+1}^{(t)}\right)\left(s_h, a_h\right)\right] + \eta H^3 T,
\end{align}
where the last line holds by Lemma \ref{lem:policy-difference}. We will further bound the first term by applying Cauchy-Schwarz inequality, and we will see that
\begin{align}\label{eqn:bellman-learned1}
&\sum_{t=1}^T \sum_{h=1}^H\mathbb{E}_{\pi^{(t)}}\left[\left(f_h^{(t)}-\gT^{\pi^{(t-1)}}_h f_{h+1}^{(t)}\right)\left(s_h, a_h\right)\right]\nonumber \\
    &= \sum_{t=1}^T \sum_{h=1}^H \mathbb{E}_{\pi^{(t)}}\left[\left(f_h^{(t)}-\gT^{\pi^{(t-1)}}_h f_{h+1}^{(t)}\right)\left(s_h, a_h\right)\right]\left(\frac{H^2 \vee \sum_{i=1}^{t-1}\mathbb{E}_{\pi^{(i)}}\left[\left(f_h^{(t)}-\gT^{\pi^{(t-1)}}_h f_{h+1}^{(t)}\right)^2\right]}{H^2 \vee \sum_{i=1}^{t-1}\mathbb{E}_{\pi^{(i)}}\left[\left(f_h^{(t)}-\gT^{\pi^{(t-1)}}_h f_{h+1}^{(t)}\right)^2\right]}\right)^{1/2}\nonumber \\
    &\leq \sqrt{\sum_{t=1}^T\sum_{h=1}^H \frac{\mathbb{E}_{\pi^{(t)}}\left[\left(f_h^{(t)}-\gT^{\pi^{(t-1)}}_h f_{h+1}^{(t)}\right)\right]^2}{H^2 \vee \sum_{i=1}^{t-1}\mathbb{E}_{\pi^{(i)}}\left[\left(f_h^{(t)}-\gT^{\pi^{(t-1)}}_h f_{h+1}^{(t)}\right)^2\right]}}\sqrt{\sum_{t=1}^T\sum_{h=1}^H H^2 \vee \sum_{i=1}^{t-1}\mathbb{E}_{\pi^{(i)}}\left[\left(f_h^{(t)}-\gT^{\pi^{(t-1)}}_h f_{h+1}^{(t)}\right)^2\right]}, 
\end{align}
where the last step follows from Cauchy-Schwarz inequality. Within the last inequality, the first term can be bounded by $H$ times the SEC of \citet{xie2022role}, by the very definition of the SEC in Definition~\ref{defn:SEC}. The second term is bounded by Lemma \ref{lem:ac-loss-bounded-optimism}, where $\beta = \Theta\left(\log \left(T G \mathcal{N}_{\mathcal{F},\left(\mathcal{T}^{\mathrm{II}}\right)^T \mathcal{F}}(1 / T) / \delta\right)\right)$. Putting these bounds into~(\ref{eqn:bellman-learned1}), we obtain that
\begin{align}\label{eqn:bellman-learned2}
    &\sum_{t=1}^T \sum_{h=1}^H\mathbb{E}_{\pi^{(t)}}\left[\left(f_h^{(t)}-\gT^{\pi^{(t-1)}}_h f_{h+1}^{(t)}\right)\left(s_h, a_h\right)\right] \leq \sqrt{H\mathsf{SEC}(\gF, \Pi, T)}\sqrt{\beta H^3 T}\leq \sqrt{\beta H^4 T\mathsf{SEC}(\gF, \Pi, T)}.
\end{align}
Plugging (\ref{eqn:bellman-learned2}) into (\ref{eqn:bellman-learned}), we finally obtain that
\begin{align}
    \sum_{t=1}^T \sum_{h=1}^H\mathbb{E}_{\pi^{(t)}}\left[\left(f_h^{(t)}-\gT^{\pi^{(t)}}_h f_{h+1}^{(t)}\right)\left(s_h, a_h\right)\right]\leq \sqrt{\beta H^4 T\mathsf{SEC}(\gF, \Pi, T)} + \eta H^3 T.
\end{align}
\end{proof}



\subsection{Auxiliary lemmas}

\subsubsection{Showing optimism for critics targeting $Q^{\pi^{(t)}}$}

We prove the following lemma in more generality than is needed for Algorithm \ref{alg:DOUHUA}, accomodating critic updates that are rarer than in every episode, with the aim to use the more general result in future sections. This can be thought of as an analogue of Lemma 15 in \citet{xie2022role}, which is also Lemmas 39 and 40 in \citet{jin2021bellman}.

\begin{lem}[Optimism and in-sample error control for critics targeting $Q^{\pi^{(t)}}$]
Let $t_{\last}$ be the time of the last critic update before episode $t$. Consider a critic targeting $Q^{\pi^{(t)}}$ as in Algorithm \ref{alg:DOUHUA}. With probability at least $1-\delta$, for all $t \in\left[T\right]$, we have that for all $h=1, \ldots, H$
\begin{align*}
&\text{(i) $\gT_h^{\pi^{(t_{\last})}} f_{h+1}^{{(t)}} \in \mathcal{F}_h^{(t)}$, and so $f_h^{(t)} \geq \gT_h^{\pi^{(t_{\last})}} f_{h+1}^{(t)}$}, \\
&\text{(ii) $\sum_{i=1}^{t-1} \mathbb{E}_{d^{\pi^{(i)}}}\left[\left(f_h^{(t)} - \gT_h^{\pi^{(t-1)}} f_{h+1}^{(t)}\right)^2\right] \leq O\left(H^2 \beta\right)$,}
\end{align*}
by choosing $\beta=c_1(\log [H T\mathcal{N}_{\mathcal{F},\left(\mathcal{T}^{\Pi}\right)^T \mathcal{F}}(\rho) / \delta])$ for some constant $c_1$.
\label{lem:ac-loss-bounded-optimism-pi}
\end{lem}

\begin{proof}
    We note that it does not hold that $Q^* \in \gF^{(t)}$, as our Bellman operator is given by the operator $\gT_h^{\pi^{(t)}}$ under policy $\pi^{(t)}$ and not the greedy policy under $f^{(t)}$. Furthermore, as we do not throw away samples not in the current batch as \citet{liu2023optimisticnaturalpolicygradient} do, we do not enjoy the same conditional independence of dataset and next-step value functions. We therefore take a different approach, of modifying the analysis of \citet{xiong2023generalframeworksequentialdecisionmaking} to the policy gradient setting in order to do so. 

    By Lemma \ref{lem:g1-ac-loss-bounded} applied to policy $\pi^{(t)}$, for any $h \in [H]$ and $t \in [T]$, we have with probability $1-\delta$ that
    \begin{align}0 \leq \gL_h^{(t, \pi^{(t)})}(\gT_h^{\pi^{(t)}}f_{h+1}^{(t)}, f_{h+1}^{(t)}) - \min_{f_h' \in \gF_h} \gL_h^{(t)}(f_h', f_{h+1}^{(t)}) \leq H^2\beta.\end{align}

    We construct the confidence sets only at timesteps $t_k$ where switches occur:
    \begin{align}\gF_h^{(t_{\text{last}}, \pi^{(t_{\last})})} := \left\{f \in \gF : \gL_h^{(t_{\text{last}}, \pi^{(t_{\last})})}(f_h, f_{h+1}) - \min_{f_h' \in \gF_h} \gL_h^{(t_{\text{last}}, \pi^{(t_{\last})})}(f_h', f_{h+1}) \leq H^2\beta \;\; \forall h \in [H] \right\}.\end{align}
    It must then follow that $\gT_h^{\pi^{(t_k)}}f_{h+1} \in \gF_h^{(t_k)}$ for any $f \in \gF$ and $t_k$. Now, we want to establish that 
    \begin{align}\gT_h^{\pi^{(t_{\text{last}})}} f_{h+1}^{(t)} \leq  \sup_{f_h \in \mathcal{F}_h^{(t_{\text{last}})}} f_h\left(s,a\right) = f_h^{(t)}.\end{align}

    Further recall that we have defined
    $$\gF_h^{(t_{\text{last}})} := \left\{f \in \gF : \gL_h^{(t_{\text{last}})}(f_h, f_{h+1}) - \min_{f_h' \in \gF_h} \gL_h^{(t_{\text{last}})}(f_h', f_{h+1}) \leq H^2\beta \;\; \forall h \in [H] \right\}.$$
    So we can apply Lemma \ref{lem:g1-ac-loss-bounded} on $t_{\text{last}}$, $f_{h+1}^{(t)}$, and $\pi^{(t)}$ to find that
    \begin{align}0 \leq \gL_h^{(t_{\text{last}})}(\gT_h^{\pi^{(t_{\text{last}})}}f_{h+1}^{(t)}, f_{h+1}^{(t)}) - \min_{f_h' \in \gF_h} \gL_h^{(t_{\text{last}})}(f_h', f_{h+1}^{(t)}) \leq H^2\beta.\end{align}
    It then holds that $$\gT_h^{\pi^{(t_{\text{last}})}}f_{h+1}^{(t)} \in \gF_h^{t_{\text{last}}} \text{, and so } \gT_h^{\pi^{(t_{\text{last}})}} f_{h+1}^{(t)} \leq  \sup_{f_h \in \mathcal{F}_h^{(t_{\text{last}})}} f_h\left(s,a\right) = f_h^{(t)}.$$

    The second result can now be shown by Lemma \ref{lem:g2-ac-loss-concentration-ub} using a similar argument to the proof of Theorem 1 in \citet{xiong2023generalframeworksequentialdecisionmaking}, which in turn takes inspiration from the proofs of Lemmas 39 and 40 in \citet{jin2021bellman}. We elaborate accordingly.

   Consider two cases, one where we perform an update at episode $t-1$ and one where we do not. If we perform an update at episode $t-1$, then by the construction of $\gF^{(t-1)}$ it must hold that
    $$\mathcal{L}_h^{(t-1, \pi^{(t-1)})}\left(f_h^{(t)}, f_{h+1}^{(t)}\right)-\min _{f_h^{\prime} \in \mathcal{F}_h} \mathcal{L}_h^{(t-1, \pi^{(t-1)})}\left(f_h^{\prime}, f_{h+1}^{(t)}\right) \leq H^2\beta, $$
    and by Lemma \ref{lem:g1-ac-loss-bounded} it must also hold that
    $$0 \leq \mathcal{L}_h^{(t-1, \pi^{(t-1)})}\left(\gT_h^{\pi^{(t-1)}}f_h^{(t)}, f_{h+1}^{(t)}\right)-\min _{f_h^{\prime} \in \mathcal{F}_h} \mathcal{L}_h^{(t-1, \pi^{(t-1)})}\left(f_h^{\prime}, f_{h+1}^{(t)}\right) \leq H^2\beta.$$
    One can then see that
    \begin{align}\mathcal{L}_h^{(t-1, \pi^{(t-1)})}\left(f_h^{(t)}, f_{h+1}^{(t)}\right) - \mathcal{L}_h^{(t-1, \pi^{(t-1)})}\left(\gT_h^{\pi^{(t-1)}}f_h^{(t)}, f_{h+1}^{(t)}\right) \leq 6H^2\beta.\end{align}

    The same holds for the other case where we do not perform an update at episode $t-1$. Observe that because we did not perform an update,  
    $$\mathcal{L}_h^{(t-1, \pi^{(t-1)})}\left(f_h^{(t-1)}, f_{h+1}^{(t-1)}\right)-\min_{f_h^{\prime} \in \mathcal{F}_h} \mathcal{L}_h^{(t-1, \pi^{(t-1)})}\left(f_h^{\prime}, f_{h+1}^{(t-1)}\right) \leq 5 H^2\beta.$$
    From Lemma \ref{lem:g1-ac-loss-bounded}, it also holds that
    $$\mathcal{L}_h^{(t-1, \pi^{(t-1)})}\left(\gT_h^{\pi^{(t-1)}}f_h^{(t-1)}, f_{h+1}^{(t-1)}\right)-\min_{f_h^{\prime} \in \mathcal{F}_h} \mathcal{L}_h^{(t-1, \pi^{(t-1)})}\left(f_h^{\prime}, f_{h+1}^{(t-1)}\right) \leq H^2\beta.$$
    Putting the above two statements together and using that $f^{(t)} = f^{(t-1)}$ again yields 
    \begin{align}
        &\mathcal{L}_h^{(t-1, \pi^{(t-1)})}\left(f_h^{(t)}, f_{h+1}^{(t)}\right) - \mathcal{L}_h^{(t-1, \pi^{(t-1)})}\left(\gT_h^{\pi^{(t-1)}}f_h^{(t)}, f_{h+1}^{(t)}\right) \nonumber\\
        &= \mathcal{L}_h^{(t-1, \pi^{(t-1)})}\left(f_h^{(t-1)}, f_{h+1}^{(t-1)}\right) - \mathcal{L}_h^{(t-1, \pi^{(t-1)})}\left(\gT_h^{\pi^{(t-1)}}f_h^{(t-1)}, f_{h+1}^{(t-1)}\right) \nonumber\\
        &\leq 6H^2\beta.
    \end{align}

    An application of Lemma \ref{lem:g2-ac-loss-concentration-ub} to both cases then yields in either case that 
    \begin{align}\sum_{i=1}^{t-1}\left(f_h^{(t)} - \gT_h^{\pi^{(t-1)}} f_{h+1}^{(t)}\right)^2(s_h^{(t)}, a_h^{(t)}) \leq 7H^2\beta,\end{align}
    and also that 
    \begin{align}\sum_{i=1}^{t-1}\E_{d_h^{(i)}}\left[\left(f_h^{(t)} - \gT_h^{\pi^{(t-1)}} f_{h+1}^{(t)}\right)^2\right] \leq 7H^2\beta.\end{align}
\end{proof}

\subsubsection{Bound on covering number of value function class (Proof of Lemma~\ref{lem:b2-covering-number-policy})}
\label{app:ac-covering-number}

    The proof is similar to that of Lemma B.2 in \citet{zhong2023theoreticalanalysisoptimisticproximal}, but we strengthen the result to show that the covering number increases only in the number of critic updates, not policy updates. As in \citet{zhong2023theoreticalanalysisoptimisticproximal}, let $\gF^{(t_k)}_{\rho^2/16\eta K H^2T, h}$ be a minimal $\rho^2/16\eta K H^2T$-net of $\gF^{(t_k)}_h$ for all $k$. So for any $\pi \propto \exp(\eta\sum_{i=1}^K (t_i - t_{i-1}) f_h^{(t_i)})$ with $ f^{(t_i)}_h \in \gF_h^{(t_i)}$, there exists some $\widehat{f}_h^{(t_i)} \in \gF^{(t_k)}_{\rho^2/16\eta K H^2T, h}$ so 
    $$\sup_{s,a} |{f}_h^{(t_i)}(s,a) - \widehat{f}_h^{(t_i)}(s,a)| \leq \frac{\rho^2}{16\eta K H^2T} \text{ for all $i \in [K]$.}$$

    It then holds that
    \begin{align}
        &\sup_{s,a} \left|\eta \sum_{i=1}^K (t_i - t_{i-1}) f_h^{(t_i)}(s,a) - \eta\sum_{i=1}^K (t_i - t_{i-1}) \widehat{f}_h^{(t_i)}(s,a)\right| 
        \leq \eta \sum_{i=1}^K (t_i - t_{i-1}) \sup_{s,a} |{f}_h^{(t_i)}(s,a) - \widehat{f}_h^{(t_i)}(s,a)|\nonumber \\
        &\qquad \leq \eta T \sum_{i=1}^K  \sup_{s,a} |{f}_h^{(t_i)}(s,a) - \widehat{f}_h^{(t_i)}(s,a)|  
        \leq \frac{\rho^2}{16H^2}. 
    \end{align}

    Now we invoke Lemma B.3 of \citet{zhong2023theoreticalanalysisoptimisticproximal}, provided as Lemma \ref{lem:b3-zhong-policy-q-ineq} for completeness, to show that 
    \begin{align}\sup_{s} ||\pi(\cdot | s) - \pi'(\cdot | s)||_1 \leq 2 \sqrt{\sup_{s,a} \left|\eta \sum_{i=1}^K (t_i - t_{i-1}) f_h^{(t_i)}(s,a) - \eta\sum_{i=1}^K (t_i - t_{i-1}) \widehat{f}_h^{(t_i)}(s,a)\right|} \leq \frac{\rho}{2H}.\end{align}

    It then holds that
    \begin{align}\gN_{\Pi_h^{(T)}}(\rho/2H) \leq \prod_{i=1}^{K}\gN_{\gF_h^{(t_i)}}(\rho^2/16\eta K H^2 T).\end{align}

    The bound for $\gN_{\gA}(\rho/2H)$ follows from discretizing the action space $\gA$ via a covering number argument, and observing that the covering number of an $\gN_{\gA}(\rho)$-dimensional probability distribution is on the order of $\gN_{\gA}(\rho)$.




\subsubsection{Closure under truncated sums limits policy class growth (Proof of Lemma \ref{lem:policy-class-growth-sums})}
\label{app:closure-truncated-sums}


    Let $\gF_{\rho^2/16\eta TH^2}$ be a minimal $\rho^2/16\eta TH^2$-net of $\gF.$ Then, for all $f_h \in \gF$, there exists some $\widehat{f}_h \in \gF_{\rho^2/16\eta TH^2} $ such that
    $$\sup_{s,a} |f_h(s,a) - \widehat{f}_h(s,a)| \leq \frac{\rho^2}{16\eta TH^2}.$$

    As $\gF$ is closed under truncated sums, it then holds that there exists some $f' \in \gF$ such that
    \begin{align}
        &\sup_{s,a} \left|\eta \sum_{t=1}^T f_h^{(t)}(s,a) - \eta \sum_{t=1}^T \widehat{f}_h(s,a)\right| 
        = \sup_{s,a} \left|\eta \sum_{t=1}^T \min\{\max\{f_h^{(t)}(s,a),0\}, H\} - \eta \sum_{t=1}^T \widehat{f}_h(s,a)\right| \nonumber\\
        &\qquad = \sup_{s,a} \left|\eta T f'_h(s,a) - \eta T \widehat{f}_h(s,a)\right| 
        \leq \frac{\rho^2}{16H^2}.
    \end{align} 

    The result then follows from Lemma \ref{lem:b1-covering-number-value} and Lemma \ref{lem:b3-zhong-policy-q-ineq}. That is, Lemma \ref{lem:b3-zhong-policy-q-ineq} shows that
    \begin{align}\sup_{s} ||\pi(\cdot | s) - \pi'(\cdot | s)||_1 \leq 2 \sqrt{\sup_{s,a} \left|\eta \sum_{t=1}^T f_h^{(t)}(s,a) - \eta\sum_{t=1}^T\widehat{f}_h^{(t)}(s,a)\right|} \leq \frac{\rho}{2H},\end{align}
it then holds that
    \begin{align}\gN_{\Pi_h^{(T)}}(\rho/2H) \leq \gN_{\gF}(\rho^2/16\eta TH^2),\end{align}
and the result directly follows from Lemma \ref{lem:b1-covering-number-value}.

\section{Proofs for Theorem \ref{thm:nora-regret_bound}}
\label{app:proof-nora}

Recall the motivation for this solution:
\begin{enumerate}
\itemsep0em 
    \item \textbf{Optimism} allows one to perform strategic exploration, addressing the issue of exploration vs. exploitation, and allowing us to avoid making reachability or coverage assumptions.
    \item \textbf{Off-policy learning} avoids throwing away any samples, ensuring that no samples are wasted.
    \item \textbf{Do rare-switching critic updates work?} A-priori, introducing rare-switching critic updates as in \citet{xiong2023generalframeworksequentialdecisionmaking} offers an appealing solution to the covering number issue. However, the Bellman operator with respect to $\pi^{(t)}$ has a very limited form of optimism. Furthermore, it is difficult to control the number of rare-switching updates in the context of general function approximation, where we make an update when the Bellman error with respect to the current policy is large, as the current policy keeps changing and so we track a moving target.
    \item \textbf{Letting the critic target $Q^*$ and not $Q^\pi$} ensures sufficient optimism, as the Bellman operator is now the same at every iteration. Further, as the critic targets $Q^*$, we do not need to control $\log \gN_{(\gT^\Pi)^T \gF}(\rho)$, as the Bellman operator for the greedy policy is a contraction. 
    \item \textbf{Rare-switching critic updates now work.} However, this introduces an additional term, where we need to bound the deviation of the current policy from its target, the greedy policy with regard to the current critic. Re-introducing rare-switching critic updates resolves this, as now we allow the actor sufficient time to catch up to the critic updates. Controlling the number of critic updates is not an issue when the critic targets $Q^*$, as the Bellman operator is now the same at every iteration. This is reminiscent of the delayed target Q-function trick common in deep RL \citep{lillicrap2019continuouscontroldeepreinforcement, fujimoto2018addressingfunctionapproximationerror}.
    \item \textbf{Policy resets.} However, going through the mirror descent proof to control the additional error results in an additional term bounded by $\sum_{k=1}^{N_{\text{updates}}} \log(1/\pi^{(t_k)})$. This term can be controlled by resetting the policy to the uniform policy upon every critic update. As critic updates are rare, on the order of $dH\log T$, the additional error incurred is very small. This trick was adopted independently by \citet{cassel2024warmupfreepolicyoptimization} for the same reason.
    \item \textbf{Increased learning rate.} We increase the learning rate accordingly by a factor of $\sqrt{dH\log T}$, exactly the square root of the number of critic updates/policy resets. This is done to mitigate the increase in regret incurred by the policy resets by a factor of $\sqrt{dH\log T}$. This can be seen as making more aggressive updates to make up for the lost ground due to policy resets when the critic makes a rare but large update. 
\end{enumerate} 
Given the above, we are now in a position to continue our analysis below.

\subsection{Regret decomposition}
We employ the following regret decomposition below. This is a slightly different regret decomposition than that of \citet{cai2024provablyefficientexplorationpolicy} and \citet{zhong2023theoreticalanalysisoptimisticproximal}, as our critic targets $Q^*$. 

\begin{lem}[Regret Decomposition For $Q^*$-Targeting Actor-Critics]
\begin{align*}
    \operatorname{Reg}(T)
    = & \sum_{t=1}^T \sum_{h=1}^H \mathbb{E}_{\pi^*}\left[\left\langle f_h^{(t)}\left(s_h, \cdot\right), \pi_h^*\left(\cdot \mid s_h\right)-\pi_h^{(t)}\left(\cdot \mid s_h\right)\right\rangle\right]-\sum_{t=1}^T \sum_{h=1}^H\mathbb{E}_{\pi^*}\left[\left(f_h^{(t)}-\gT_h^{\pi^{*}}f_{h+1}^{(t)}\right)\left(s_h, a_h\right)\right] \\
    &\quad+ \sum_{t=1}^T \sum_{h=1}^H\mathbb{E}_{\pi^{(t)}}\left[\left(f_h^{(t)}-\gT_h f_{h+1}^{(t)}\right)\left(s_h, a_h\right)\right]+ \sum_{t=1}^T \sum_{h=1}^H\mathbb{E}_{\pi^{(t)}}\left[\left\langle f_{h+1}^{(t)}(s_{h+1},\cdot), \pi_{h+1}^{f^{(t)}} (\cdot | s_{h+1}) -\pi_{h+1}^{(t)}(\cdot | s_{h+1}) \right\rangle\right].
\end{align*}
\label{lem:regret-decomp-ac}
\end{lem}
\begin{proof}

    By adding and subtracting $f_1^{(t)}(s_1^{(t)}, \pi_1^{(t)}(s_1^{(t)}))$, we obtain
    \begin{align}\label{eqn:decom-qstar-1}
        \text{Reg}(T) &= \sum_{t=1}^T\left(V_1^{*}(s_1^{(t)})-V_1^{\pi^{(t)}}(s_1^{(t)})\right) \nonumber\\
        &= \sum_{t=1}^T \left(V_1^{*}(s_1^{(t)})-f_1^{(t)}(s_1^{(t)}, \pi_1^{(t)}(s_1^{(t)}))\right) + \sum_{t=1}^T\left(f_1^{(t)}(s_1^{(t)}, \pi_1^{(t)}(s_1^{(t)}))-V_1^{\pi^{(t)}}(s_1^{(t)})\right) .
    \end{align}
    
    To further decompose the two terms in (\ref{eqn:decom-qstar-1}), we apply the value difference lemma/generalized policy difference lemma in Lemma \ref{lem:policy-value-difference-lemma} \citep{cai2024provablyefficientexplorationpolicy, efroni2020optimisticpolicyoptimizationbandit} with $f^{(t)}$ as the Q-function, $\pi'=\pi^*$, and $\pi = \pi^{(t)}$ to find that the first term can written as
    \begin{align}\label{eqn:decom-qstar-2}
        &\sum_{t=1}^T \left(V_1^{*}(s_1^{(t)})-f_1^{(t)}(s_1^{(t)}, \pi_1^{(t)}(s_1^{(t)}))\right) \nonumber\\
        &= \sum_{t=1}^T \sum_{h=1}^H \mathbb{E}_{\pi^*}\left[\left\langle f_h^{(t)}\left(s_h, \cdot\right), \pi_h^*\left(\cdot \mid s_h\right)-\pi_h^{(t)}\left(\cdot \mid s_h\right)\right\rangle\right] -\sum_{t=1}^T \sum_{h=1}^H\mathbb{E}_{\pi^*}\left[\left(f_h^{(t)}-\gT_h^{\pi^*}f_{h+1}^{(t)}\right)\left(s_h, a_h\right)\right].
    \end{align}

    Another application of Lemma \ref{lem:policy-value-difference-lemma} with $\pi = \pi' = \pi^{(t)}$ and with $f^{(t)}$ as the Q-function yields
    \begin{align}
        &\sum_{t=1}^T\left(f_1^{(t)}(s_1^{(t)}, \pi_1^{(t)}(s_1^{(t)}))-V_1^{\pi^{(t)}}(s_1^{(t)})\right) \nonumber\\
        &= \sum_{t=1}^T \sum_{h=1}^H \mathbb{E}_{\pi^{(t)}}\left[\left\langle f_h^{(t)}\left(s_h, \cdot\right), \pi_h^{(t)}\left(\cdot \mid s_h\right)-\pi_h^{(t)}\left(\cdot \mid s_h\right)\right\rangle\right] + \sum_{t=1}^T \sum_{h=1}^H\mathbb{E}_{\pi^{(t)}}\left[\left(f_h^{(t)}-\gT_h^{\pi^{(t)}}f_{h+1}^{(t)}\right)\left(s_h, a_h\right)\right] \nonumber\\
        &=  \sum_{t=1}^T \sum_{h=1}^H\mathbb{E}_{\pi^{(t)}}\left[\left(f_h^{(t)}-\gT_h^{\pi^{(t)}}f_{h+1}^{(t)}\right)\left(s_h, a_h\right)\right].
    \end{align}

    We now need to relate the Bellman operator with respect to the current policy $\gT_h^{\pi^{(t)}}$ to the Bellman operator $\gT_h$. Observe that
    \begin{align}\label{eqn:decom-qstar-3}
        &\sum_{t=1}^T \sum_{h=1}^H\mathbb{E}_{\pi^{(t)}}\left[\left(f_h^{(t)}-\gT_h^{\pi^{(t)}}f_{h+1}^{(t)}\right)\left(s_h, a_h\right)\right] \nonumber\\
        &= \sum_{t=1}^T \sum_{h=1}^H\mathbb{E}_{\pi^{(t)}}\left[\left(f_h^{(t)} - \gT_h f_{h+1}^{(t)} \right)\left(s_h, a_h\right)\right]  + \sum_{t=1}^T \sum_{h=1}^H\mathbb{E}_{\pi^{(t)}}\left[\left(\gT_h f_{h+1}^{(t)} -\gT_h^{\pi^{(t)}}f_{h+1}^{(t)}\right)\left(s_h, a_h\right)\right]. 
\end{align}
We further decompose the second term in the following way
\begin{align}\label{eqn:decom-qstar-4}
& \mathbb{E}_{\pi^{(t)}}\left[\left(\gT_h f_{h+1}^{(t)} -\gT_h^{\pi^{(t)}}f_{h+1}^{(t)}\right)\left(s_h, a_h\right)\right]\nonumber\\
        &= \mathbb{E}_{\pi^{(t)}}\left[r_h(s_h, a_h) + f_{h+1}^{(t)}\left(s_{h+1}, \pi_{h+1}^{f^{(t)}}(s_{h+1})\right) - r_h(s_h, a_h) - f_{h+1}^{(t)}\left(s_{h+1}, \pi_{h+1}^{(t)}(s_{h+1})\right) \right] \nonumber\\
        &= \mathbb{E}_{\pi^{(t)}}\left[f_{h+1}^{(t)}\left(s_{h+1}, \pi_{h+1}^{f^{(t)}}(s_{h+1})\right) - f_{h+1}^{(t)}\left(s_{h+1}, \pi_{h+1}^{(t)}(s_{h+1})\right) \right] \nonumber\\
        &=  \mathbb{E}_{\pi^{(t)}}\left[\left\langle f_{h+1}^{(t)}(s_{h+1},\cdot), \pi_{h+1}^{f^{(t)}} (\cdot | s_{h+1}) -\pi_{h+1}^{(t)}(\cdot | s_{h+1}) \right\rangle\right].
    \end{align}
    Replacing the original terms in (\ref{eqn:decom-qstar-1}) by (\ref{eqn:decom-qstar-2}), (\ref{eqn:decom-qstar-3}) and (\ref{eqn:decom-qstar-4}) concludes the proof.
\end{proof}

Each term within the regret decomposition is dealt with differently. We bound the first term via the standard mirror descent analysis, the second term with optimism, the third term with the GOLF regret decomposition and the SEC of \citet{xie2022role}, and the fourth term via a modified mirror descent analysis. 

The attentive reader will note that the Bellman error is defined as $f - \gT f$ in our setup and that of \citet{liu2023optimisticnaturalpolicygradient}, and $\gT f - f$ in that of \citet{zhong2023theoreticalanalysisoptimisticproximal}.

\subsection{Bounding the tracking error}
\begin{lem}[Mirror Descent Tracking Error for Algorithm \ref{alg:NORA}]
    Let $t_k$ and $t_{k+1}$ be switch times within Algorithm \ref{alg:NORA}, where we use the convention that $\pi^{(t_k)} \propto 1$ is post-policy reset and $\pi^{(t_{k+1})} \not\propto 1$ is pre-policy reset. The tracking error with respect to the optimal policy is then bounded by:
    $$\sum_{t=t_k+1}^{t_{k+1}-1} \sum_{h=1}^H \mathbb{E}_{\pi^*}\left[\left\langle f_h^{(t)}\left(s_h, \cdot\right), \pi_h^*\left(\cdot \mid s_h\right)-\pi_h^{(t)}\left(\cdot \mid s_h\right)\right\rangle\right] \leq \eta H^3 (t_{k+1} - t_k)/2 + \frac{H \log |\gA|}{\eta} + H^2.$$
    \label{lem:mirror-descent-nora}
\end{lem}
\begin{proof}

Note that for any $t$ such that $t_k + 1 \leq t \leq t_{k+1}-1$, as we do not reset the policy during these timesteps as the critic does not change, we have 
\begin{align}\pi_{h+1}^{(t+1)}(\cdot | s') = \frac{\pi_{h+1}^{(t)}(\cdot | s') \exp(\eta f_{h+1}^{(t)}(s', \cdot))}{\sum_{a \in \gA} \pi_{h+1}^{(t)}(a | s') \exp(\eta f_{h+1}^{(t)}(s', a))} = Z^{-1}\pi_{h+1}^{(t)}(\cdot | s') \exp(\eta f_{h+1}^{(t)}(s', \cdot)).\end{align}
Rearranging this yields
$$\eta f_{h+1}^{(t)}(s', \cdot) = \log Z + \log \pi_{h+1}^{(t+1)}(\cdot | s') - \log \pi_{h+1}^{(t)}(\cdot | s'),$$
where $\log Z$ is
$$\log Z = \log\left(\sum_{a \in \gA} \pi_{h+1}^{(t)}(a | s') \exp(\eta f_{h+1}^{(t)}(s', a))\right).$$

We can now bound, noting that $\sum_{a \in \gA} \left(\pi_{h+1}^{*}(\cdot | s') - \pi_{h+1}^{(t+1)}(\cdot | s')\right) = 0$, that
\begin{align}\label{eqn:mirror-alg2-1}
    &\left\langle \eta f_{h+1}^{(t)}(s', \cdot), \pi_{h+1}^{*}(\cdot | s') - \pi_{h+1}^{(t+1)}(\cdot | s')\right\rangle \nonumber\\
    &= \left\langle \log Z + \log \pi_{h+1}^{{(t+1)}}(\cdot | s') - \log \pi_{h+1}^{(t)}(\cdot | s'), \pi_{h+1}^{*}(\cdot | s') - \pi_{h+1}^{(t+1)}(\cdot | s')\right\rangle \nonumber\\ 
    &=\left\langle \log \pi_{h+1}^{{(t+1)}}(\cdot | s') - \log \pi_{h+1}^{(t)}(\cdot | s'), \pi_{h+1}^{*}(\cdot | s') - \pi_{h+1}^{(t+1)}(\cdot | s')\right\rangle \nonumber\\
    &= \text{KL}\left(\pi_{h+1}^{*}(\cdot | s')\; || \;\pi_{h+1}^{(t)}(\cdot | s')\right) - \text{KL}\left(\pi_{h+1}^{*}(\cdot | s')\; || \;\pi_{h+1}^{(t+1)}(\cdot | s')\right) - \text{KL}\left(\pi_{h+1}^{(t+1)}(\cdot | s')\; || \;\pi_{h+1}^{(t)}(\cdot | s')\right),
\end{align}
where the last line follows from Lemma~\ref{lem: KL} with $\pi_1 = \pi_{h+1}^{*}(\cdot | s')$, $\pi_2 = \pi_{h+1}^{(t)}(\cdot | s')$ and $\pi= \pi_{h+1}^{(t+1)}(\cdot | s')$. So it must hold that
\begin{align}\label{eqn:mirror-alg2-2}
    &\left\langle \eta f_{h+1}^{(t)}(s', \cdot), \pi_{h+1}^{*}(\cdot | s') - \pi_{h+1}^{(t)}(\cdot | s')\right\rangle  \nonumber\\
    &= \left\langle \eta f_{h+1}^{(t)}(s', \cdot), \pi_{h+1}^{*}(\cdot | s') - \pi_{h+1}^{(t+1)}(\cdot | s')\right\rangle - \left\langle \eta f_{h+1}^{(t)}(s', \cdot), \pi_{h+1}^{(t)}(\cdot | s') - \pi_{h+1}^{(t+1)}(\cdot | s')\right\rangle \nonumber\\
    &\leq \text{KL}\left(\pi_{h+1}^{*}(\cdot | s')\; || \;\pi_{h+1}^{(t)}(\cdot | s')\right) - \text{KL}\left(\pi_{h+1}^{*}(\cdot | s')\; || \;\pi_{h+1}^{(t+1)}(\cdot | s')\right) - \text{KL}\left(\pi_{h+1}^{(t+1)}(\cdot | s')\; || \;\pi_{h+1}^{(t)}(\cdot | s')\right) \nonumber\\
    &\qquad + \eta H ||\pi_{h+1}^{(t+1)}(\cdot | s') - \pi_{h+1}^{(t)}(\cdot | s')||_1.
\end{align}

Summing up $t$ and $h$, we can then derive that
\begin{align}\label{eqn:mirror-alg2-3}
    &\sum_{t=t_k+1}^{t_{k+1}-1} \sum_{h=1}^H \mathbb{E}_{\pi^*}\left[\left\langle f_{h+1}^{(t)}(s', \cdot), \pi_{h+1}^{*}(\cdot | s') - \pi_{h+1}^{(t)}(\cdot | s')\right\rangle \right]\nonumber\\
    &= \frac{1}{\eta} \sum_{t=t_k+1}^{t_{k+1}-1} \sum_{h=1}^H \mathbb{E}_{\pi^*}\left[ \left\langle \eta f_{h+1}^{(t)}(s', \cdot), \pi_{h+1}^{*}(\cdot | s') - \pi_{h+1}^{(t)}(\cdot | s')\right\rangle\right] \nonumber\\
    &\leq \frac{1}{\eta} \sum_{t=t_k+1}^{t_{k+1}-1} \sum_{h=1}^H \mathbb{E}_{\pi^*}\left[\text{KL}\left(\pi_{h+1}^{*}(\cdot | s')\; || \;\pi_{h+1}^{(t)}(\cdot | s')\right) - \text{KL}\left(\pi_{h+1}^{*}(\cdot | s')\; || \;\pi_{h+1}^{(t+1)}(\cdot | s')\right) \right.\nonumber\\
    &\qquad\qquad\qquad \left.- \text{KL}\left(\pi_{h+1}^{(t+1)}(\cdot | s')\; || \;\pi_{h+1}^{(t)}(\cdot | s')\right) + \eta H ||\pi_{h+1}^{(t+1)}(\cdot | s') - \pi_{h+1}^{(t)}(\cdot | s')||_1\right].
\end{align}
Here, we apply Pinsker's inequality on the last line of (\ref{eqn:mirror-alg2-3}), it follows that
\begin{align}\label{eqn:mirror-alg2-4}
    &\sum_{t=t_k+1}^{t_{k+1}-1} \sum_{h=1}^H \mathbb{E}_{\pi^*}\left[\left\langle f_{h+1}^{(t)}(s', \cdot), \pi_{h+1}^{*}(\cdot | s') - \pi_{h+1}^{(t)}(\cdot | s')\right\rangle \right] \nonumber\\
    &\leq \frac{1}{\eta} \sum_{t=t_k+1}^{t_{k+1}-1}  \sum_{h=1}^H \mathbb{E}_{\pi^*}\left[ \text{KL}\left(\pi_{h+1}^{*}(\cdot | s')\; || \;\pi_{h+1}^{(t)}(\cdot | s')\right) - \text{KL}\left(\pi_{h+1}^{*}(\cdot | s')\; || \;\pi_{h+1}^{(t+1)}(\cdot | s')\right) \right. \nonumber\\
    &\qquad\qquad\qquad\qquad\qquad\qquad  \left. - ||\pi_{h+1}^{(t+1)}(\cdot | s') - \pi_{h+1}^{(t)}(\cdot | s')||_1^2/2 + \eta H ||\pi_{h+1}^{(t+1)}(\cdot | s') - \pi_{h+1}^{(t)}(\cdot | s')||_1 \right] \nonumber\\
    &\leq \frac{1}{\eta} \sum_{t=t_k+1}^{t_{k+1}-1} \sum_{h=1}^H \mathbb{E}_{\pi^*}\left[ \text{KL}\left(\pi_{h+1}^{*}(\cdot | s')\; || \;\pi_{h+1}^{(t)}(\cdot | s')\right) - \text{KL}\left(\pi_{h+1}^{*}(\cdot | s')\; || \;\pi_{h+1}^{(t+1)}(\cdot | s')\right) + \eta^2H^2/2\right],
\end{align}
where we use the fact that $\max_{x\in \R}\left\{-x^2/2 + \eta H x\right\} = \eta^2H^2/2$ in the last line. Continuing to simplify (\ref{eqn:mirror-alg2-4}) yields
\begin{align}\label{eqn:mirror-alg2-5}
    &\sum_{t=t_k+1}^{t_{k+1}-1} \sum_{h=1}^H \mathbb{E}_{\pi^*}\left[\left\langle f_{h+1}^{(t)}(s', \cdot), \pi_{h+1}^{*}(\cdot | s') - \pi_{h+1}^{(t)}(\cdot | s')\right\rangle \right] \nonumber\\
    &\leq \sum_{t=t_k+1}^{t_{k+1}-1} \sum_{h=1}^H \left(\eta H^2/2 + \mathbb{E}_{\pi^*}\left[ \frac{\text{KL}\left(\pi_{h+1}^{*}(\cdot | s')\; || \;\pi_{h+1}^{(t)}(\cdot | s')\right) - \text{KL}\left(\pi_{h+1}^{*}(\cdot | s')\; || \;\pi_{h+1}^{(t+1)}(\cdot | s')\right)}{\eta} \right]\right) \nonumber\\
    &\leq \eta H^3 (t_{k+1} - t_k)/2 + \sum_{h=1}^H \frac{\text{KL}\left(\pi_{h+1}^{*}(\cdot | s')\; || \;\pi_{h+1}^{(t_k+1)}(\cdot | s')\right) - \text{KL}\left(\pi_{h+1}^{*}(\cdot | s')\; || \;\pi_{h+1}^{(t_{k+1})}(\cdot | s')\right)}{\eta}\nonumber\\
    &\leq \eta H^3 (t_{k+1} - t_k)/2 + \frac{H \log |\gA|}{\eta} + H^2,
\end{align}
where the last inequality follows from the fact that the KL-divergence is non-negative as well as the policy reset of setting $\pi_h^{(t_k)}$ back to the uniform policy after the critic update. Note that we use the convention that $\pi_h^{(t_k)}$ is after the reset, and $\pi_h^{(t_{k+1})}$ is before the reset. This means that
$$\frac{1}{|\gA|\exp(\eta H)} \leq \frac{\exp(0)}{\sum_{a \in \gA}\exp(\eta H)} \leq \pi_{h+1}^{(t_k+1)}(\cdot | s') = \frac{ \exp(\eta f_{h+1}^{(t)}(s', \cdot))}{\sum_{a \in \gA}\exp(\eta f_{h+1}^{(t)}(s', a))} \leq \frac{\exp(\eta H)}{\sum_{a \in \gA} \exp(0)} = \frac{\exp(\eta H)}{|\gA|},$$
and so we can conclude that
\begin{align}\text{KL}\left(\pi_{h+1}^{*}(\cdot | s')\; || \;\pi_{h+1}^{(t_k+1)}(\cdot | s')\right) \leq \log\left(\frac{1}{\pi_{h+1}^{(t_k+1)}(\cdot | s')}\right) \leq \log|\gA| + \eta H.\end{align}

\end{proof}

\subsection{Asserting optimism}

\begin{lem}[Negative Bellman Error For Algorithm \ref{alg:NORA}]
    Within Algorithm \ref{alg:NORA}, we have that
    $$-\sum_{t=1}^T \sum_{h=1}^H\mathbb{E}_{\pi^*}\left[f_h^{(t)} - \gT_h^{\pi^{*}} f_{h+1}^{(t)} \right]  \leq 0.$$
    \label{lem:negative-bellman-decomp}
\end{lem}
\begin{proof}

Applying Lemma \ref{lem:ac-loss-bounded-optimism}, we note that $f_h^{(t)} \geq \gT_h^{\pi^{*}} f_{h+1}^{(t)}$. The result then follows.

\end{proof}

\subsection{Bounding the Bellman error under the learned policies}

We now turn our attention to the Bellman error with respect to $\gT_h$ under the current policy's occupancy measure.

\begin{lem}[Sum of Bellman Errors Under Algorithm \ref{alg:NORA}]
Within Algorithm \ref{alg:NORA}, the sum of Bellman errors with respect to $\gT$ under the occupancy measure of $\pi^{(t)}$ can be bounded by: 
    $$\sum_{t=1}^T \sum_{h=1}^H\mathbb{E}_{\pi^{(t)}}\left[\left(f_h^{(t)}-\gT_h f_{h+1}^{(t)}\right)\left(s_h, a_h\right)\right] \leq \sqrt{\beta H^4 T\mathsf{SEC}(\gF, \Pi, T)},$$
    where $\beta = \Theta\left(\log \left(HT^2 \mathcal{N}_{\gF,\gT \gF}(1/T) / \delta\right) \right)$.
    \label{lem:occ-measure-regret-nora}
\end{lem}
\begin{proof}

We can now perform the same Cauchy-Schwarz and change of measure argument as in \citet{xie2022role} to find that
\begin{align}\label{eqn:bellman-alg2}
\sum_{t=1}^T \sum_{h=1}^H\mathbb{E}_{\pi^{(t)}}\left[\left(f_h^{(t)}-\gT_h f_{h+1}^{(t)}\right)\left(s_h, a_h\right)\right] &=
    \sum_{t=1}^T \sum_{h=1}^H \E_{d_h^{(t)}}[\delta_h^{(t)}]\nonumber\\
    &= \sum_{t=1}^T \sum_{h=1}^H \E_{d_h^{(t)}}[\delta_h^{(t)}] \left(\frac{H^2 \vee \sum_{i=1}^{t-1}\E_{d_h^{(i)}}[(\delta_h^{(t)})^2]}{H^2 \vee \sum_{i=1}^{t-1}\E_{d_h^{(i)}}[(\delta_h^{(t)})^2]}\right)^{1/2} \nonumber\\
    &\leq \sqrt{\sum_{t=1}^T\sum_{h=1}^H \frac{\E_{d_h^{(i)}}[\delta_h^{(t)}]^2}{H^2 \vee \sum_{i=1}^{t-1}\E_{d_h^{(i)}}[(\delta_h^{(t)})^2]}}\sqrt{\sum_{t=1}^T\sum_{h=1}^H H^2 \vee \sum_{i=1}^{t-1}\E_{d_h^{(i)}}[(\delta_h^{(t)})^2]}. 
\end{align}
Within the last inequality, the first term can be bounded by $H$ times the SEC of \citet{xie2022role}, by Definition~\ref{defn:SEC}. The second term is bounded by Lemma \ref{lem:ac-loss-bounded-optimism}. Therefore, (\ref{eqn:bellman-alg2}) can be bounded as
\begin{align}
    \sum_{t=1}^T \sum_{h=1}^H\mathbb{E}_{\pi^{(t)}}\left[\left(f_h^{(t)}-\gT_h f_{h+1}^{(t)}\right)\left(s_h, a_h\right)\right] &\leq \sqrt{H\mathsf{SEC}(\gF, \Pi, T)}\sqrt{\beta H^3 T}\leq \sqrt{\beta H^4 T \mathsf{SEC}(\gF, \Pi, T)}.
\end{align}

\end{proof}


\begin{lem}[Greedy Policy Tracking Error For Algorithm \ref{alg:NORA}]
Let $t_k$ and $t_{k+1}$ be switch times within Algorithm \ref{alg:NORA}, where we use the convention that $\pi^{(t_k)} \propto 1$ is post-policy reset and $\pi^{(t_{k+1})} \not\propto 1$ is pre-policy reset. The tracking error with respect to the greedy policy corresponding to the current critic is then bounded by:
    \begin{align*}
        \sum_{t=t_k+1}^{t_{k+1}-1} \sum_{h=1}^H \E_{d_h^{(t)}}\left[\left\langle f_{h+1}^{(t)}(s', \cdot) , \pi_{h+1}^{f^{(t)}}(\cdot | s') - \pi_{h+1}^{(t)} (\cdot | s')\right\rangle\right] \leq \eta H^3 (t_{k+1} - t_k)/2 + \frac{H\log|\gA|}{\eta} + H^2.
    \end{align*}
    \label{lem:policy-tracking-ft-t-nora}
\end{lem}
\begin{proof}
    Again note that for any $t$ such that $t_k+1 \leq t \leq t_{k+1}$, we do not reset the policy during these timesteps as the critic does not change. We therefore have
\begin{align}\pi_{h+1}^{(t+1)}(\cdot | s') = \frac{\pi_{h+1}^{(t)}(\cdot | s') \exp(\eta f_{h+1}^{(t)}(s', \cdot))}{\sum_{a \in \gA} \pi_{h+1}^{(t)}(a | s') \exp(\eta f_{h+1}^{(t)}(s', a))} = Z_t^{-1}\pi_{h+1}^{(t)}(\cdot | s') \exp(\eta f_{h+1}^{(t)}(s', \cdot)),\end{align}
and rearranging this yields
$$\eta f_{h+1}^{(t)}(s', \cdot) = \log Z_t + \log \pi_{h+1}^{(t+1)}(\cdot | s') - \log \pi_{h+1}^{(t)}(\cdot | s'),$$
where $\log Z_t$ is
$$\log Z_t = \log\left(\sum_{a \in \gA} \pi_{h+1}^{(t)}(a | s') \exp(\eta f_{h+1}^{(t)}(s', a))\right) = \log \pi_{h+1}^{(t)}(\cdot | s') - \log \pi_{h+1}^{(t+1)}(\cdot | s') + \eta f_{h+1}^{(t)}(s', \cdot).$$

Noting that $\sum_{a \in \gA} \left(\pi_{h+1}(\cdot | s') - \pi_{h+1}'(\cdot | s')\right) = 0$ for any two policies $\pi,\pi' \in \Pi$, we can now bound 
\begin{align}\label{eqn:decom-nora-1}
    &\left\langle \eta f_{h+1}^{(t)}(s', \cdot), \pi_{h+1}^{f^{(t)}}(\cdot | s') - \pi_{h+1}^{(t+1)}(\cdot | s')\right\rangle \nonumber\\
    &= \left\langle \log Z + \log \pi_{h+1}^{{(t+1)}}(\cdot | s') - \log \pi_{h+1}^{(t)}(\cdot | s'), \pi_{h+1}^{f^{(t)}}(\cdot | s') - \pi_{h+1}^{(t+1)}(\cdot | s')\right\rangle \nonumber\\ 
    &=\left\langle \log \pi_{h+1}^{{(t+1)}}(\cdot | s') - \log \pi_{h+1}^{(t)}(\cdot | s'), \pi_{h+1}^{f^{(t)}}(\cdot | s') - \pi_{h+1}^{(t+1)}(\cdot | s')\right\rangle \nonumber\\
    &= \text{KL}\left(\pi_{h+1}^{f^{(t)}}(\cdot | s')\; || \;\pi_{h+1}^{(t)}(\cdot | s')\right) - \text{KL}\left(\pi_{h+1}^{f^{(t)}}(\cdot | s')\; || \;\pi_{h+1}^{(t+1)}(\cdot | s')\right) - \text{KL}\left(\pi_{h+1}^{(t+1)}(\cdot | s')\; || \;\pi_{h+1}^{(t)}(\cdot | s')\right),
\end{align}
where the last line follows from Lemma~\ref{lem: KL}. So it satisfies that 
\begin{align}\label{eqn:decom-nora-2}
    &\left\langle \eta f_{h+1}^{(t)}(s', \cdot), \pi_{h+1}^{f^{(t)}}(\cdot | s') - \pi_{h+1}^{(t)}(\cdot | s')\right\rangle \nonumber \\
    &= \left\langle \eta f_{h+1}^{(t)}(s', \cdot), \pi_{h+1}^{f^{(t)}}(\cdot | s') - \pi_{h+1}^{(t+1)}(\cdot | s')\right\rangle - \left\langle \eta f_{h+1}^{(t)}(s', \cdot), \pi_{h+1}^{(t)}(\cdot | s') - \pi_{h+1}^{(t+1)}(\cdot | s')\right\rangle \nonumber\\
    &\leq \text{KL}\left(\pi_{h+1}^{f^{(t)}}(\cdot | s')\; || \;\pi_{h+1}^{(t)}(\cdot | s')\right) - \text{KL}\left(\pi_{h+1}^{f^{(t)}}(\cdot | s')\; || \;\pi_{h+1}^{(t+1)}(\cdot | s')\right) - \text{KL}\left(\pi_{h+1}^{(t+1)}(\cdot | s')\; || \;\pi_{h+1}^{(t)}(\cdot | s')\right)\nonumber \\
    &\qquad + \eta H ||\pi_{h+1}^{(t+1)}(\cdot | s') - \pi_{h+1}^{(t)}(\cdot | s')||_1.
\end{align}

Sum up with $t$ and $h$, we can then derive
\begin{align}\label{eqn:decom-nora-3}
    &\sum_{t=t_k+1}^{t_{k+1}-1} \sum_{h=1}^H \mathbb{E}_{d_h^{(t)}}\left[\left\langle f_{h+1}^{(t)}(s', \cdot), \pi_{h+1}^{f^{(t)}}(\cdot | s') - \pi_{h+1}^{(t)}(\cdot | s')\right\rangle \right]\nonumber\\
    &= \frac{1}{\eta} \sum_{t=t_k+1}^{t_{k+1}-1} \sum_{h=1}^H \mathbb{E}_{d_h^{(t)}}\left[ \left\langle \eta f_{h+1}^{(t)}(s', \cdot), \pi_{h+1}^{f^{(t)}}(\cdot | s') - \pi_{h+1}^{(t)}(\cdot | s')\right\rangle\right] \nonumber\\
    &\leq \frac{1}{\eta} \sum_{t=t_k+1}^{t_{k+1}-1} \sum_{h=1}^H \mathbb{E}_{d_h^{(t)}}\left[\text{KL}\left(\pi_{h+1}^{f^{(t)}}(\cdot | s')\; || \;\pi_{h+1}^{(t)}(\cdot | s')\right) - \text{KL}\left(\pi_{h+1}^{f^{(t)}}(\cdot | s')\; || \;\pi_{h+1}^{(t+1)}(\cdot | s')\right) \right.\nonumber\\
    &\qquad\qquad\qquad \left.- \text{KL}\left(\pi_{h+1}^{(t+1)}(\cdot | s')\; || \;\pi_{h+1}^{(t)}(\cdot | s')\right) + \eta H ||\pi_{h+1}^{(t+1)}(\cdot | s') - \pi_{h+1}^{(t)}(\cdot | s')||_1\right].
\end{align}
When applying Pinsker's inequality in the last line of (\ref{eqn:decom-nora-3}), we can show that
\begin{align}\label{eqn:decom-nora-4}
    &\sum_{t=t_k+1}^{t_{k+1}-1} \sum_{h=1}^H \mathbb{E}_{d_h^{(t)}}\left[\left\langle f_{h+1}^{(t)}(s', \cdot), \pi_{h+1}^{f^{(t)}}(\cdot | s') - \pi_{h+1}^{(t)}(\cdot | s')\right\rangle \right] \nonumber\\
    &\leq \frac{1}{\eta} \sum_{t=t_k+1}^{t_{k+1}-1} \sum_{h=1}^H \mathbb{E}_{d_h^{(t)}}\left[ \text{KL}\left(\pi_{h+1}^{f^{(t)}}(\cdot | s')\; || \;\pi_{h+1}^{(t)}(\cdot | s')\right) - \text{KL}\left(\pi_{h+1}^{f^{(t)}}(\cdot | s')\; || \;\pi_{h+1}^{(t+1)}(\cdot | s')\right) \right. \nonumber\\
    &\qquad\qquad\qquad\qquad\qquad\qquad  \left. - ||\pi_{h+1}^{(t+1)}(\cdot | s') - \pi_{h+1}^{(t)}(\cdot | s')||_1^2/2 + \eta H ||\pi_{h+1}^{(t+1)}(\cdot | s') - \pi_{h+1}^{(t)}(\cdot | s')||_1 \right].
\end{align}
Here, we use the fact that $\max_{x\in \R}\left\{-x^2/2 + \eta H x\right\} = \eta^2H^2/2$, and obtain that 
\begin{align*}
    - ||\pi_{h+1}^{(t+1)}(\cdot | s') - \pi_{h+1}^{(t)}(\cdot | s')||_1^2/2 + \eta H ||\pi_{h+1}^{(t+1)}(\cdot | s') - \pi_{h+1}^{(t)}(\cdot | s')||_1\leq\eta^2H^2/2.
\end{align*}

We now can continue through the following. Note that $f^{(t_k+1)} = ... f^{(t)} = ... = f^{(t_{k+1})}$ due to rare-switching. With a direct calculation, we obtain that
\begin{align}\label{eqn:decom-nora-5}
    & \sum_{t=t_k+1}^{t_{k+1}-1} \mathbb{E}_{d_h^{(t)}}\left[ \frac{\text{KL}\left(\pi_{h+1}^{f^{(t)}}(\cdot | s')\; || \;\pi_{h+1}^{(t)}(\cdot | s')\right) - \text{KL}\left(\pi_{h+1}^{f^{(t)}}(\cdot | s')\; || \;\pi_{h+1}^{(t+1)}(\cdot | s')\right)}{\eta} \right]\nonumber\\
    & = \sum_{t=t_k+1}^{t_{k+1}-1} \mathbb{E}_{d_h^{(t)}}\left[ \frac{\text{KL}\left(\pi_{h+1}^{f^{(t_k+1)}}(\cdot | s')\; || \;\pi_{h+1}^{(t)}(\cdot | s')\right) - \text{KL}\left(\pi_{h+1}^{f^{(t_k+1)}}(\cdot | s')\; || \;\pi_{h+1}^{(t+1)}(\cdot | s')\right)}{\eta} \right]\nonumber\\
    & = \mathbb{E}_{d_h^{(t)}}\left[ \frac{\text{KL}\left(\pi_{h+1}^{f^{(t_k+1)}}(\cdot | s')\; || \;\pi_{h+1}^{(t_k+1)}(\cdot | s')\right) - \text{KL}\left(\pi_{h+1}^{f^{(t_k+1)}}(\cdot | s')\; || \;\pi_{h+1}^{(t_{k+1})}(\cdot | s')\right)}{\eta} \right].
\end{align} 
Furthermore, there exists some $a^*(s')$ for each $s'$ such that
$$a^*(s) = \argmax_{a' \in \gA} f_{h+1}^{(t)}(s',a') \text{ for all } t_k+1 \leq t \leq t_{k+1},$$
hence that $\pi_h^f(a' | s) = \mathbbm{1}\left(a' = \argmax_{a \in \gA} f_h(s,a)\right)$. This yields
\begin{align}\label{eqn:decom-nora-6}
    &\mathbb{E}_{d_h^{(t)}}\left[ \frac{\text{KL}\left(\pi_{h+1}^{f^{(t_k+1)}}(\cdot | s')\; || \;\pi_{h+1}^{(t_k+1)}(\cdot | s')\right) - \text{KL}\left(\pi_{h+1}^{f^{(t_k+1)}}(\cdot | s')\; || \;\pi_{h+1}^{(t_{k+1})}(\cdot | s')\right)}{\eta} \right]\nonumber\\ 
    &= \frac{1}{\eta}\mathbb{E}_{d_h^{(t)}}\left[1 \cdot \log\left(\frac{1}{\pi_{h+1}^{(t_k+1)} ( a^*(s') | s')}\right) - 1 \cdot \log\left(\frac{1}{\pi_{h+1}^{(t_{k+1})} ( a^*(s') | s')}\right)\right]\nonumber\\
    &=\frac{1}{\eta}\mathbb{E}_{d_h^{(t)}}\left[\log\left(\frac{\pi_{h+1}^{(t_{k+1})} ( a^*(s') | s')}{\pi_{h+1}^{(t_k+1)} ( a^*(s') | s')}\right)\right].
\end{align}
For this term, we note that $\pi_h^{(t_k)}$ is back to the uniform policy after the critic update. Note that we use the convention that $\pi_h^{(t_k)}$ is before the reset, and $\pi_h^{(t_k+1)}$ is after the reset. This means that
$$\frac{1}{|\gA|\exp(\eta H)} \leq \frac{\exp(0)}{\sum_{a \in \gA}\exp(\eta H)} \leq \pi_{h+1}^{(t_k+1)}(\cdot | s') = \frac{ \exp(\eta f_{h+1}^{(t)}(s', \cdot))}{\sum_{a \in \gA}\exp(\eta f_{h+1}^{(t)}(s', a))} \leq \frac{\exp(\eta H)}{\sum_{a \in \gA} \exp(0)} = \frac{\exp(\eta H)}{|\gA|},$$
and so we can conclude that
\begin{align}
    &\sum_{t=t_k+1}^{t_{k+1}-1} \mathbb{E}_{d_h^{(t)}}\left[ \frac{\text{KL}\left(\pi_{h+1}^{f^{(t)}}(\cdot | s')\; || \;\pi_{h+1}^{(t)}(\cdot | s')\right) - \text{KL}\left(\pi_{h+1}^{f^{(t)}}(\cdot | s')\; || \;\pi_{h+1}^{(t+1)}(\cdot | s')\right)}{\eta} \right]\nonumber\\
    & = \frac{1}{\eta}\log\left(\frac{\pi_{h+1}^{(t_{k+1})} ( a^*(s') | s')}{\pi_{h+1}^{(t_k+1)}(a^*(s') | s')}\right) \leq \frac{1}{\eta}\log\left(\frac{1}{\pi_{h+1}^{(t_k+1)}(a^*(s') | s')}\right) \leq \frac{1}{\eta}\left(\log|\gA| + \eta H\right).
\end{align}
Therefore, we obtain that
\begin{align}
    \sum_{t=t_k+1}^{t_{k+1}-1} \sum_{h=1}^H \mathbb{E}_{d_h^{(t)}}\left[\left\langle f_{h+1}^{(t)}(s', \cdot), \pi_{h+1}^{f^{(t)}}(\cdot | s') - \pi_{h+1}^{(t)}(\cdot | s')\right\rangle \right] 
  &\leq \frac{1}{\eta} \sum_{t=t_k+1}^{t_{k+1}-1} \sum_{h=1}^H(\eta^2H^2/2) + \frac{1}{\eta}\sum_{h=1}^{H}\big(\log|\mathcal{A}| + \eta H\big) \nonumber\\
    &\leq \eta H^3 (t_{k+1} - t_k)/2 + \frac{H\log|\gA|}{\eta} + H^2.
\end{align}


\end{proof}

\subsection{Auxiliary lemmas}

\subsubsection{Bound on rare-switching update frequency}

We bound the rare-switching update frequency via an argument similar to that of \citet{xiong2023generalframeworksequentialdecisionmaking}. 
\begin{lem}[Switching Costs]
    Consider a procedure where the critic is updated only when there exists some $h$ such that $$\mathcal{L}_h^{(t)}(f_h^{(t)}, f_{h+1}^{(t)})-\min _{f_h^{\prime} \in \mathcal{F}_h} \mathcal{L}_h^{(t)}(f_h^{\prime}, f_{h+1}^{(t)}) \geq 5H^2\beta.$$ 
    This performs no more than $N_{\text{updates}, h}(T) \leq O(d\log(T))$ Q-function updates for each $h \in [H]$, and no more than $N_{\text{updates}}(T) \leq O(dH\log(T))$ Q-function updates in total.
    \label{lem:switch-cost}
\end{lem}
\begin{proof}
    To show this result, we control the number of switches induced by the Q-function class targeting $\pi^*$ by upper and lower bounding the cumulative squared Bellman error under the observed states and actions. Fix some $h \in [H]$ for now. For simplicity, write $K_h = N_{\text{updates},h}(T)$ for the total number of updates, and $t_{1,h},...,t_{K_h,h}$ the update times for $f_h^{(t)}$, with $t_{0,h} = 0$. By definition, at every $t_{k,h}$,
    \begin{align}\mathcal{L}_h^{(t_{k,h})}\left(f_h^{(t_{k,h})}, f_{h+1}^{(t_{k,h})}\right)-\min _{f_h^{\prime} \in \mathcal{F}_h} \mathcal{L}_h^{(t_{k,h})}\left(f_h^{\prime}, f_{h+1}^{(t_{k,h})}\right) \geq 5 H^2 \beta.\end{align}

    An application of Lemma \ref{lem:g1-ac-loss-bounded} yields 
    $$0 \leq \mathcal{L}_h^{\left(t_{k,h}\right)}\left(\mathcal{T}_h f_{h+1}^{(t_{k,h})}, f_{h+1}^{(t_{k,h})}\right)-\min _{f_h^{\prime} \in \mathcal{F}_h} \mathcal{L}_h^{\left(t_{k,h}\right)}\left(f_h^{\prime}, f_{h+1}^{(t_{k,h})}\right) \leq H^2 \beta,$$
    $$\min _{f_h^{\prime} \in \mathcal{F}_h} \mathcal{L}_h^{\left(t_{k,h}\right)}\left(f_h^{\prime}, f_{h+1}^{(t_{k,h})}\right)  \geq \mathcal{L}_h^{\left(t_{k,h}\right)}\left(\mathcal{T}_h f_{h+1}^{(t_{k,h})}, f_{h+1}^{(t_{k,h})}\right)- H^2 \beta,$$
    $$\mathcal{L}_h^{(t_{k,h})}\left(f_h^{(t_{k,h})}, f_{h+1}^{(t_{k,h})}\right)-\mathcal{L}_h^{\left(t_{k,h}\right)}\left(\mathcal{T}_h f_{h+1}^{(t_{k,h})}, f_{h+1}^{(t_{k,h})}\right) \geq 4H^2\beta.
    $$
    From the above, one can obtain 
    \begin{align}
        &\mathcal{L}_h^{(t_{k-1,h}+1:t_{k,h})}\left(f_h^{(t_{k,h})}, f_{h+1}^{(t_{k,h})}\right)-\mathcal{L}_h^{\left(t_{k-1,h}+1:t_{k,h}\right)}\left(\mathcal{T}_h f_{h+1}^{(t_{k,h})}, f_{h+1}^{(t_{k,h})}\right) \nonumber\\
        &=\mathcal{L}_h^{(t_{k-1,h}+1:t_{k,h})}\left(f_h^{(t_{k-1,h}+1)}, f_{h+1}^{(t_{k-1,h}+1)}\right)-\mathcal{L}_h^{\left(t_{k-1,h}+1:t_{k,h}\right)}\left(\mathcal{T}_h f_{h+1}^{(t_{k-1,h}+1)}, f_{h+1}^{(t_{k-1,h}+1)}\right) \nonumber\\
        &= \mathcal{L}_h^{(t_{k,h})}\left(f_h^{(t_{k-1,h}+1)}, f_{h+1}^{(t_{k-1,h}+1)}\right)-\mathcal{L}_h^{\left(t_{k,h}\right)}\left(\mathcal{T}_h f_{h+1}^{(t_{k-1,h}+1)}, f_{h+1}^{(t_{k-1,h}+1)}\right) \nonumber\\
        &\qquad- \left(\mathcal{L}_h^{(t_{k-1,h})}\left(f_h^{(t_{k-1,h}+1)}, f_{h+1}^{(t_{k-1,h}+1)}\right) +\mathcal{L}_h^{\left(t_{k-1,h}\right)}\left(\mathcal{T}_h f_{h+1}^{(t_{k-1,h}+1)}, f_{h+1}^{(t_{k-1,h}+1)}\right)\right) \nonumber\\
        &= \mathcal{L}_h^{(t_{k,h})}\left(f_h^{(t_{k,h})}, f_{h+1}^{(t_{k,h})}\right)-\mathcal{L}_h^{\left(t_{k,h}\right)}\left(\mathcal{T}_h f_{h+1}^{(t_{k,h})}, f_{h+1}^{(t_{k,h})}\right) \nonumber\\
        &\qquad- \left(\mathcal{L}_h^{(t_{k-1,h})}\left(f_h^{(t_{k-1,h}+1)}, f_{h+1}^{(t_{k-1,h}+1)}\right) + \mathcal{L}_h^{\left(t_{k-1,h}\right)}\left(\mathcal{T}_h f_{h+1}^{(t_{k-1,h}+1)}, f_{h+1}^{(t_{k-1,h}+1)}\right)\right) \nonumber\\
        &\geq 4H^2\beta - H^2\beta = 3H^2\beta.
    \end{align}
    
    Therefore, for any $t$ such that $t_{k-1,h} < t \leq t_{k,h}$, this argument and noting that $f_h^{(t_{k-1,h}+1)} = ... = f_h^{(t)} = ... = f_h^{(t_{k,h})}$ yields
    \begin{align}\mathcal{L}_h^{(t_{k-1,h}+1:t_{k,h})}\left(f_h^{(t)}, f_{h+1}^{(t)}\right)-\mathcal{L}_h^{\left(t_{k-1,h}+1:t_{k,h}\right)}\left(\mathcal{T}_h f_{h+1}^{(t)}, f_{h+1}^{(t)}\right) \geq 3H^2\beta.
    \end{align}

    An application of Lemma \ref{lem:g3-ac-loss-concentration-ub} while noting that $f_h^{(t_{k-1,h}+1)} = ... = f_h^{(t)} = ... = f_h^{(t_{k,h})}$ leads to 
    \begin{align}\sum_{i=t_{k-1,h}+1}^{t_{k,h}}\left(f_h^{(i)} - \gT_h f_{h+1}^{(i)}\right)^2(s_h^{(i)}, a_h^{(i)}) = \sum_{i=t_{k-1,h}+1}^{t_{k,h}}\left(f_h^{(t_{k,h})} - \gT_h f_{h+1}^{(t_{k,h})}\right)^2(s_h^{(i)}, a_h^{(i)}) \geq 2H^2\beta.\end{align}
    Now summing over all $t_{1,h},...,t_{K,h}$ yields 
    \begin{align}\sum_{t=1}^T \left(f_h^{(t)} - \gT_h f_{h+1}^{(t)}\right)^2(s_h^{(t)}, a_h^{(t)}) = \sum_{k=1}^{K_h} \sum_{i=t_{k-1,h}+1}^{t_{k,h}} \left(f_h^{(i)} - \gT_h f_{h+1}^{(i)}\right)^2(s_h^{(i)}, a_h^{(i)}) \geq 2(K_h-1)H^2\beta.\end{align}
    By Lemma \ref{lem:ac-loss-bounded-optimism}, we have that 
    \begin{align}\sum_{i=1}^{t-1} \left(f_h^{(t)} - \gT_h f_{h+1}^{(t)}\right)^2(s_h^{(i)}, a_h^{(i)}) \leq O(H^2\beta).\end{align}
    Invoking the squared distributional Bellman eluder dimension definition, as in \cite{xiong2023generalframeworksequentialdecisionmaking}, yields 
    \begin{align}\sum_{t=1}^{T} \left(f_h^{(t)} - \gT_h f_{h+1}^{(t)}\right)^2(s_h^{(t)}, a_h^{(t)}) \leq O(dH^2\beta\log T).\end{align}

    So we have established that 
    \begin{align}2(K_h-1)H^2\beta \leq \sum_{t=1}^{T} \left(f_h^{(t)} - \gT_h f_{h+1}^{(t)}\right)^2(s_h^{(t)}, a_h^{(t)}) \leq O(dH^2\beta\log T).\end{align}
    The number of updates for each $h$ must therefore be bounded as
    $$K_h \leq d \log T, \text{ yielding } N_{\text{switch}}(T) \leq dH\log T.$$

\end{proof}

\subsubsection{Showing optimism for critics targeting $Q^{*}$}

The following lemma applies to Algorithms \ref{alg:NORA} and \ref{alg:NOAH-star}. As Algorithm \ref{alg:NORA} is optimistic, both properties apply to it, while only the second applies to Algorithm \ref{alg:NOAH-star}.

\begin{lem}[Optimism and in-sample error control for critics targeting $Q^{*}$]
With probability at least $1-\delta$, for all $t \in\left[T\right]$, we have that for all $h=1, \ldots, H$, an optimistic critic targeting $Q^*$ in the same way as defined in Algorithm \ref{alg:NORA} achieves
\begin{align*}
&\text{(i) $\gT_h^{\pi^{(t)}} f_{h+1}^{{(t)}} \in \mathcal{F}_h^{(t)}$, and $f_h^{(t)} \geq \gT_h^{\pi}f_{h+1}^{(t)}$ for all $\pi$,}
\end{align*}
Similarly, a critic targeting $Q^*$ as in Algorithms \ref{alg:NORA} and \ref{alg:NOAH-star} achieves
\begin{align*}
&\text{(ii) $\sum_{i=1}^{t-1} \mathbb{E}_{d^{\pi^{(i)}}}\left[\left(f_h^{(t)} - \gT_h f_{h+1}^{(t)}\right)^2\right] \leq O\left(H^2 \beta\right)$,}
\end{align*}
by choosing $\beta=c_1\left(\log \left[H T\mathcal{N}_{\gF,\gT \gF}(\rho) / \delta\right]\right)$ for some constant $c_1$.
\label{lem:ac-loss-bounded-optimism}
\end{lem}

\begin{proof}

    By Lemma \ref{lem:g1-ac-loss-bounded} applied to the greedy policy $\pi^{f^{(t)}}$, for any $h \in [H]$ and $t \in [T]$, we have with probability $1-\delta$ that
    \begin{align}0 \leq \gL_h^{(t)}(\gT_h f_{h+1}^{(t)}, f_{h+1}^{(t)}) - \min_{f_h' \in \gF_h} \gL_h^{(t)}(f_h', f_{h+1}^{(t)}) \leq H^2\beta.\end{align}    
    We construct the confidence sets only at timesteps $t_k$ where switches occur:
    $$\gF_h^{(t_{\text{last}})} := \left\{f \in \gF : \gL_h^{(t_{\text{last}})}(f_h, f_{h+1}) - \min_{f_h' \in \gF_h} \gL_h^{(t_{\text{last}})}(f_h', f_{h+1}) \leq H^2\beta \;\; \forall h \in [H] \right\}.$$
    
     We can now show the first result. As we defined $f_h^{(t)}(s,a) :=\operatorname{argmax}_{f_h \in \mathcal{F}_h^{(t_{\text{last}})}} f_h\left(s,a\right)$, for all $\pi$:
    \begin{align}
        \gT_h^{\pi}f_{h+1}^{(t)}(s,a) = r_h(s,a) + \E_{s'} \left[f_{h+1}^{(t)}(s', \pi_{h+1}(s'))\right] \leq r_h(s,a) + \E_{s'} \left[\max_{a' \in \gA} f_{h+1}^{(t)}(s', a')\right] = \gT_h f_{h+1}^{(t)}(s,a).
    \end{align}
    We further note that $\gT_h f_{h+1}^{(t)}(s,a)\leq f_h^{(t)}(s,a)$, as $\gT_h f_{h+1}^{(t)}\in\mathcal{F}_h^{(t_{\text{last}})}$ and by the definition of $f_h^{(t)}(s,a)$.
    





    The second result can now be shown using by Lemma \ref{lem:g2-ac-loss-concentration-ub} using a similar argument to the proof of Theorem 1 in \citet{xiong2023generalframeworksequentialdecisionmaking}, which in turn takes inspiration from the proofs of Lemmas 39 and 40 in \citet{jin2021bellman}. We elaborate accordingly.
    
    Consider two cases, one where we perform an update at episode $t-1$ and one where we do not. If we perform an update at episode $t-1$, then by the choice of $f^{(t)}$ to be near-optimal (in fact, with Algorithm \ref{alg:NOAH-star}, it is optimal and this is zero) it must hold that
    $$\mathcal{L}_h^{(t-1)}\left(f_h^{(t)}, f_{h+1}^{(t)}\right)-\min _{f_h^{\prime} \in \mathcal{F}_h} \mathcal{L}_h^{(t-1)}\left(f_h^{\prime}, f_{h+1}^{(t)}\right) \leq 5H^2\beta, $$
    and by Lemma \ref{lem:g1-ac-loss-bounded} it must also hold that
    $$0 \leq \mathcal{L}_h^{(t-1)}\left(\gT_h f_h^{(t)}, f_{h+1}^{(t)}\right)-\min _{f_h^{\prime} \in \mathcal{F}_h} \mathcal{L}_h^{(t-1)}\left(f_h^{\prime}, f_{h+1}^{(t)}\right) \leq H^2\beta.$$
    One can then see that
    \begin{align}\mathcal{L}_h^{(t-1)}\left(f_h^{(t)}, f_{h+1}^{(t)}\right) - \mathcal{L}_h^{(t-1)}\left(\gT_h f_h^{(t)}, f_{h+1}^{(t)}\right) \leq 6H^2\beta.\end{align}

    The same holds for the other case where we do not perform an update at episode $t-1$. Observe that because we did not perform an update,  
    $$\mathcal{L}_h^{(t-1)}\left(f_h^{(t-1)}, f_{h+1}^{(t-1)}\right)-\min_{f_h^{\prime} \in \mathcal{F}_h} \mathcal{L}_h^{(t-1)}\left(f_h^{\prime}, f_{h+1}^{(t-1)}\right) \leq 5 H^2\beta.$$
    From Lemma \ref{lem:g1-ac-loss-bounded}, it also holds that
    $$\mathcal{L}_h^{(t-1)}\left(\gT_h f_h^{(t-1)}, f_{h+1}^{(t-1)}\right)-\min_{f_h^{\prime} \in \mathcal{F}_h} \mathcal{L}_h^{(t-1)}\left(f_h^{\prime}, f_{h+1}^{(t-1)}\right) \leq H^2\beta.$$
    Putting the above two statements together and using that $f^{(t)} = f^{(t-1)}$ again yields 
    \begin{align}
        &\mathcal{L}_h^{(t-1)}\left(f_h^{(t)}, f_{h+1}^{(t)}\right) - \mathcal{L}_h^{(t-1)}\left(\gT_h f_h^{(t)}, f_{h+1}^{(t)}\right) = \mathcal{L}_h^{(t-1)}\left(f_h^{(t-1)}, f_{h+1}^{(t-1)}\right) - \mathcal{L}_h^{(t-1)}\left(\gT_h f_h^{(t-1)}, f_{h+1}^{(t-1)}\right) \leq 6H^2\beta.
    \end{align}

    An application of Lemma \ref{lem:g2-ac-loss-concentration-ub} to both cases then yields in either case that 
\begin{align}\sum_{i=1}^{t-1}\left(f_h^{(t)} - \gT_h f_{h+1}^{(t)}\right)^2(s_h^{(i)}, a_h^{(i)}) \leq 7H^2\beta,\end{align}
    and also that 
    \begin{align}\sum_{i=1}^{t-1}\E_{d_h^{(i)}}\left[\left(f_h^{(t)} - \gT_h f_{h+1}^{(t)}\right)^2\right] \leq 7H^2\beta.\end{align}
\end{proof}




\section{Proofs for Regret Guarantees of Hybrid RL}
\subsection{Proofs for Theorem \ref{thm:hybrid-noah-pi}}
\label{app:noah-pi-proofs}

We start with the same regret decomposition as in Lemma \ref{lem:regret-decomp-douhua-ac}:
\begin{align}\label{eqn:reg-decom-alg3}
    \operatorname{Reg}(T)= & \sum_{t=1}^T\left(V_1^{*}(s_1^{(t)})-V_1^{\pi^{(t)}}(s_1^{(t)})\right) \nonumber\\
    = & \sum_{t=1}^T \sum_{h=1}^H \mathbb{E}_{\pi^*}\left[\left\langle f_h^{(t)}\left(s_h, \cdot\right), \pi_h^*\left(\cdot \mid s_h\right)-\pi_h^{(t)}\left(\cdot \mid s_h\right)\right\rangle\right] -\sum_{t=1}^T \sum_{h=1}^H\mathbb{E}_{\pi^*}\left[\left(f_h^{(t)}-\gT_h^{\pi^{*}}f_{h+1}^{(t)}\right)\left(s_h, a_h\right)\right] \nonumber\\
    &\qquad+ \sum_{t=1}^T \sum_{h=1}^H\mathbb{E}_{\pi^{(t)}}\left[\left(f_h^{(t)}-\gT_h^{\pi^{(t)}} f_{h+1}^{(t)}\right)\left(s_h, a_h\right)\right].
\end{align}

We control the first term with the same argument as Theorem \ref{thm:regret-bound-douhua}, by using Lemma \ref{lem:mirror-descent-douhua}. Controlling the third term follows by the same argument as in Lemma \ref{lem:occ-measure-regret-douhua}. It remains to tackle the second term, which we bound with the offline data. 

    First, we decompose the last term of (\ref{eqn:reg-decom-alg3}) as
    \begin{align}\label{eqn:reg-decom-alg3-1}
        -\sum_{t=1}^T \sum_{h=1}^H \mathbb{E}_{\pi^*}\left[f_h^{(t)}-\mathcal{T}_h^{\pi^{*}} f_{h+1}^{(t)}\right] 
        = \sum_{t=1}^T \sum_{h=1}^H \mathbb{E}_{\pi^*}\left[\mathcal{T}_h^{\pi^{(t)}} f_{h+1}^{(t)}- f_h^{(t)}\right] + \sum_{t=1}^T \sum_{h=1}^H \mathbb{E}_{\pi^*}\left[\mathcal{T}_h^{\pi^{*}} f_{h+1}^{(t)}- \mathcal{T}_h^{\pi^{(t)}} f_{h+1}^{(t)}\right].
    \end{align}

The latter term of (\ref{eqn:reg-decom-alg3-1}) is bounded as:
\begin{align}
& \sum_{t=1}^T \sum_{h=1}^H \mathbb{E}_{\pi^*}\left[\mathcal{T}_h^{\pi^{*}} f_{h+1}^{(t)}- \mathcal{T}_h^{\pi^{(t)}} f_{h+1}^{(t)}\right]\nonumber\\
        &= \sum_{t=1}^T \sum_{h=1}^H \mathbb{E}_{\pi^*}\left[r_h(s_h, a_h) + f_{h+1}^{(t)}\left(s_{h+1}, \pi_{h+1}^{*}(s_{h+1})\right) - r_h(s_h, a_h) - f_{h+1}^{(t)}\left(s_{h+1}, \pi_{h+1}^{(t)}(s_{h+1})\right) \right] \nonumber\\
        &= \sum_{t=1}^T \sum_{h=1}^H \mathbb{E}_{\pi^{*}}\left[f_{h+1}^{(t)}\left(s_{h+1}, \pi_{h+1}^{*}(s_{h+1})\right) - f_{h+1}^{(t)}\left(s_{h+1}, \pi_{h+1}^{(t)}(s_{h+1})\right) \right] \nonumber\\
        &=  \sum_{t=1}^T \sum_{h=1}^H \mathbb{E}_{\pi^{*}}\left[\left\langle f_{h+1}^{(t)}(s_{h+1},\cdot), \pi_{h+1}^{*} (\cdot | s_{h+1}) -\pi_{h+1}^{(t)}(\cdot | s_{h+1}) \right\rangle\right].
    \end{align}
    This yields what is essentially a copy of the first term of (\ref{eqn:reg-decom-alg3}). For the former term, we use a similar argument to that of Theorem \ref{thm:hybrid-nora-regret_bound}. Concretely, as long as $-f \in \gF$ for all $f \in \gF$,
\begin{align}
    &\sum_{t=1}^T \sum_{h=1}^H\mathbb{E}_{\pi^*}\left[\left(\gT_h^{\pi^{(t)}}f_{h+1}^{(t)} - f_h^{(t)}\right)\left(s_h, a_h\right)\right] \nonumber\\ 
    &\leq \sum_{t=1}^T \sum_{h=1}^H \E_{\pi^*}\left[\left(\gT_h^{\pi^{(t)}}f_{h+1}^{(t)} - f_h^{(t)}\right)\right] \left(\frac{N_{\off}\E_{\mu}\left[\left(\gT_h^{\pi^{(t)}}f_{h+1}^{(t)} - f_h^{(t)}\right)^2\right] + \sum_{i=1}^{t-1}\E_{d_h^{(i)}}\left[\left(\gT_h^{\pi^{(t)}}f_{h+1}^{(t)} - f_h^{(t)}\right)^2\right]}{N_{\off}\E_{\mu}\left[\left(\gT_h^{\pi^{(t)}}f_{h+1}^{(t)} - f_h^{(t)}\right)^2\right] + \sum_{i=1}^{t-1}\E_{d_h^{(i)}}\left[\left(\gT_h^{\pi^{(t)}}f_{h+1}^{(t)} - f_h^{(t)}\right)^2\right]}\right)^{1/2}\nonumber \\
    &\leq \sqrt{\sum_{t=1}^T\sum_{h=1}^H \frac{\E_{\pi^*}\left[\gT_h^{\pi^{(t)}}f_{h+1}^{(t)} - f_h^{(t)}\right]^2}{N_{\off}\E_{\mu}\left[\left(\gT_h^{\pi^{(t)}}f_{h+1}^{(t)} - f_h^{(t)}\right)^2\right] + \sum_{i=1}^{t-1}\E_{d_h^{(i)}}\left[\left(\gT_h^{\pi^{(t)}}f_{h+1}^{(t)} - f_h^{(t)}\right)^2\right]}}\nonumber\\
    &\qquad \cdot\sqrt{\sum_{t=1}^T\sum_{h=1}^H \left(N_{\off}\E_{\mu}\left[\left(\gT_h^{\pi^{(t)}}f_{h+1}^{(t)} - f_h^{(t)}\right)^2\right] + \sum_{i=1}^{t-1}\E_{d_h^{(i)}}\left[\left(\gT_h^{\pi^{(t)}}f_{h+1}^{(t)} - f_h^{(t)}\right)^2\right]\right)} \nonumber\\
    &\leq \sqrt{\sum_{t=1}^T\sum_{h=1}^H \frac{\E_{\pi^*}\left[\gT_h^{\pi^{(t)}}f_{h+1}^{(t)} - f_h^{(t)}\right]^2}{N_{\off}\E_{\mu}\left[\left(\gT_h^{\pi^{(t)}}f_{h+1}^{(t)} - f_h^{(t)}\right)^2\right]}}\nonumber\\
    &\qquad \cdot\sqrt{\sum_{t=1}^T\sum_{h=1}^H \left(N_{\off}\E_{\mu}\left[\left(\gT_h^{\pi^{(t)}}f_{h+1}^{(t)} - f_h^{(t)}\right)^2\right] + \sum_{i=1}^{t-1}\E_{d_h^{(i)}}\left[\left(\gT_h^{\pi^{(t)}}f_{h+1}^{(t)} - f_h^{(t)}\right)^2\right]\right)} \nonumber\\
    &\leq \sqrt{c_{\off}^*(\gF,\Pi)HT/N_{\off}}\sqrt{\beta H^3T} \nonumber\\
    &\leq \sqrt{H^4\beta c_{\off}^*(\gF,\Pi) T^2/N_{\off}},
\end{align}
where for the penultimate line, the bound for the first term follows directly from the Definition \ref{defn:single-policy-concentrability} on single-policy conentrability coefficient, and the second bound follows directly from Lemma \ref{lem:optimism-Q-pi-t}. Note that the argument is similar to that of Theorem \ref{thm:hybrid-nora-regret_bound}, with the exception that we can use the single-policy concentrability coefficient as the density ratio we need to bound is \begin{align}\frac{\E_{\pi^*}\left[\gT_h^{\pi^{(t)}}f_{h+1}^{(t)} - f_h^{(t)}\right]^2}{\E_{\mu}\left[\left(\gT_h^{\pi^{(t)}}f_{h+1}^{(t)} - f_h^{(t)}\right)^2\right]} \leq \max_{h\in[H]}\sup_{f \in \gF} \sup_{\pi \in \Pi} \frac{\E_{\pi^*}\left[f_h-\gT_h^\pi f_{h+1} \right]^2}{\E_{\mu}[\left(f_h-\gT_h f_{h+1}\right)^2]} = c_{\off}^*(\gF, \Pi),\end{align}
where again the first inequality holds as long as $-f \in \gF$ for all $f \in \gF$.

\subsection{Proofs for Theorem \ref{thm:hybrid-noah-star}}
\label{app:noah-star-proofs}

We start with the same regret decomposition in Lemma \ref{lem:regret-decomp-ac}:
\begin{align}
    \operatorname{Reg}(T)= & \sum_{t=1}^T\left(V_1^{*}(s_1^{(t)})-V_1^{\pi^{(t)}}(s_1^{(t)})\right) \nonumber\\
    = & \sum_{t=1}^T \sum_{h=1}^H \mathbb{E}_{\pi^*}\left[\left\langle f_h^{(t)}\left(s_h, \cdot\right), \pi_h^*\left(\cdot \mid s_h\right)-\pi_h^{(t)}\left(\cdot \mid s_h\right)\right\rangle\right]-\sum_{t=1}^T \sum_{h=1}^H\mathbb{E}_{\pi^*}\left[\left(f_h^{(t)}-\gT_h^{\pi^{*}}f_{h+1}^{(t)}\right)\left(s_h, a_h\right)\right] \nonumber\\
    &\quad+ \sum_{t=1}^T \sum_{h=1}^H\mathbb{E}_{\pi^{(t)}}\left[\left(f_h^{(t)}-\gT_h f_{h+1}^{(t)}\right)\left(s_h, a_h\right)\right]+ \sum_{t=1}^T \sum_{h=1}^H\mathbb{E}_{\pi^{(t)}}\left[\left\langle f_{h+1}^{(t)}(s_{h+1},\cdot), \pi_{h+1}^{f^{(t)}} (\cdot | s_{h+1}) -\pi_{h+1}^{(t)}(\cdot | s_{h+1}) \right\rangle\right] \nonumber\\
    = & \sum_{t=1}^T \sum_{h=1}^H \mathbb{E}_{\pi^*}\left[\left\langle f_h^{(t)}\left(s_h, \cdot\right), \pi_h^*\left(\cdot \mid s_h\right)-\pi_h^{(t)}\left(\cdot \mid s_h\right)\right\rangle\right] \nonumber\\
    &\quad-\sum_{t=1}^T \sum_{h=1}^H\mathbb{E}_{\pi^*}\left[\left(f_h^{(t)}-\gT_h f_{h+1}^{(t)}\right)\left(s_h, a_h\right)\right] + \sum_{t=1}^T \sum_{h=1}^H\mathbb{E}_{\pi^{*}}\left[\left\langle f_{h+1}^{(t)}(s_{h+1},\cdot), \pi_{h+1}^{*}(\cdot | s_{h+1}) - \pi_{h+1}^{f^{(t)}} (\cdot | s_{h+1}) \right\rangle\right]\nonumber\\
    &\quad+ \sum_{t=1}^T \sum_{h=1}^H\mathbb{E}_{\pi^{(t)}}\left[\left(f_h^{(t)}-\gT_h f_{h+1}^{(t)}\right)\left(s_h, a_h\right)\right]+ \sum_{t=1}^T \sum_{h=1}^H\mathbb{E}_{\pi^{(t)}}\left[\left\langle f_{h+1}^{(t)}(s_{h+1},\cdot), \pi_{h+1}^{f^{(t)}} (\cdot | s_{h+1}) -\pi_{h+1}^{(t)}(\cdot | s_{h+1}) \right\rangle\right] \nonumber\\
    = & \sum_{t=1}^T \sum_{h=1}^H \mathbb{E}_{\pi^*}\left[\left\langle f_h^{(t)}\left(s_h, \cdot\right), \pi_h^*\left(\cdot \mid s_h\right)-\pi_h^{(t)}\left(\cdot \mid s_h\right)\right\rangle\right] \nonumber\\
    &\quad-\sum_{t=1}^T \sum_{h=1}^H\mathbb{E}_{\pi^*}\left[\left(f_h^{(t)}-\gT_h f_{h+1}^{(t)}\right)\left(s_h, a_h\right)\right] + \sum_{t=1}^T \sum_{h=1}^H\mathbb{E}_{\pi^{*}}\left[\left\langle f_{h+1}^{(t)}(s_{h+1},\cdot), \pi_{h+1}^{*}(\cdot | s_{h+1}) - \pi_{h+1}^{(t)} (\cdot | s_{h+1}) \right\rangle\right]\nonumber\\
    &\quad + \sum_{t=1}^T \sum_{h=1}^H\mathbb{E}_{\pi^{*}}\left[\left\langle f_{h+1}^{(t)}(s_{h+1},\cdot), \pi_{h+1}^{(t)}(\cdot | s_{h+1}) - \pi_{h+1}^{f^{(t)}} (\cdot | s_{h+1}) \right\rangle\right]\nonumber\\
    &\quad+ \sum_{t=1}^T \sum_{h=1}^H\mathbb{E}_{\pi^{(t)}}\left[\left(f_h^{(t)}-\gT_h f_{h+1}^{(t)}\right)\left(s_h, a_h\right)\right]+ \sum_{t=1}^T \sum_{h=1}^H\mathbb{E}_{\pi^{(t)}}\left[\left\langle f_{h+1}^{(t)}(s_{h+1},\cdot), \pi_{h+1}^{f^{(t)}} (\cdot | s_{h+1}) -\pi_{h+1}^{(t)}(\cdot | s_{h+1}) \right\rangle\right],
\end{align}
where we break up the original second term corresponding to the negative Bellman error under $\pi^*$ with a similar decomposition as in the proof of Lemma \ref{lem:regret-decomp-ac}. 
Now, observe that
\begin{align}
    \sum_{t=1}^T \sum_{h=2}^H\mathbb{E}_{\pi^{*}}\left[\left\langle f_{h}^{(t)}(s_{h},\cdot), \pi_{h}^{(t)}(\cdot | s_{h}) - \pi_{h}^{f^{(t)}} (\cdot | s_{h}) \right\rangle\right] 
    = \sum_{t=1}^T \sum_{h=2}^H\mathbb{E}_{\pi^{*}}\left[f_{h}^{(t)}(s_{h}, \pi_{h}^{(t)}(s_{h})) - \max_{a \in \gA} f_{h}^{(t)}(s_{h}, a) \right] \leq 0.
\end{align}


So the remaining regret decomposition is:
\begin{align}
    \operatorname{Reg}(T)= & \sum_{t=1}^T\left(V_1^{*}(s_1^{(t)})-V_1^{\pi^{(t)}}(s_1^{(t)})\right) \nonumber\\
    = & \sum_{t=1}^T \sum_{h=1}^H \mathbb{E}_{\pi^*}\left[\left\langle f_h^{(t)}\left(s_h, \cdot\right), \pi_h^*\left(\cdot \mid s_h\right)-\pi_h^{(t)}\left(\cdot \mid s_h\right)\right\rangle\right]-\sum_{t=1}^T \sum_{h=1}^H\mathbb{E}_{\pi^*}\left[\left(f_h^{(t)}-\gT_h f_{h+1}^{(t)}\right)\left(s_h, a_h\right)\right] \nonumber\\
    &\quad +  \sum_{t=1}^T \sum_{h=1}^H\mathbb{E}_{\pi^{*}}\left[\left\langle f_{h+1}^{(t)}(s_{h+1},\cdot), \pi_{h+1}^{*}(\cdot | s_{h+1}) - \pi_{h+1}^{(t)} (\cdot | s_{h+1}) \right\rangle\right] \nonumber\\
    &\quad+ \sum_{t=1}^T \sum_{h=1}^H\mathbb{E}_{\pi^{(t)}}\left[\left(f_h^{(t)}-\gT_h f_{h+1}^{(t)}\right)\left(s_h, a_h\right)\right]+ \sum_{t=1}^T \sum_{h=1}^H\mathbb{E}_{\pi^{(t)}}\left[\left\langle f_{h+1}^{(t)}(s_{h+1},\cdot), \pi_{h+1}^{f^{(t)}} (\cdot | s_{h+1}) -\pi_{h+1}^{(t)}(\cdot | s_{h+1}) \right\rangle\right]
\end{align}

The proof then follows analogously to Theorem \ref{thm:nora-regret_bound} for the first, third, fourth and fifth terms. Note that the proofs for Lemmas \ref{lem:g1-ac-loss-bounded}, \ref{lem:g2-ac-loss-concentration-ub}, \ref{lem:g3-ac-loss-concentration-ub}, and Lemma \ref{lem:switch-cost} still hold. Intuitively, this is because the first three lemmas deal with the generalization error of the empirical TD loss under the occupancy measure of the current policy, and the switch cost proof depends only on these lemmas and choosing some $f^{(t)}$ with low enough training error. Choosing the minimizer as in Algorithm \ref{alg:NOAH-star} fulfills this condition.

Unlike the analysis for Algorithm \ref{alg:NORA}, we bound the second term with the offline data here. We use a similar argument to that of Theorem \ref{thm:hybrid-nora-regret_bound}. Concretely, as long as $-f \in \gF$ for all $f \in \gF$,
\begin{align}
    &\sum_{t=1}^T \sum_{h=1}^H\mathbb{E}_{\pi^*}\left[\left(\gT_hf_{h+1}^{(t)} - f_h^{(t)}\right)\left(s_h, a_h\right)\right] \nonumber\\ 
    &\leq \sum_{t=1}^T \sum_{h=1}^H \E_{\pi^*}\left[\left(\gT_hf_{h+1}^{(t)} - f_h^{(t)}\right)\right] \left(\frac{N_{\off}\E_{\mu}\left[\left(\gT_hf_{h+1}^{(t)} - f_h^{(t)}\right)^2\right] + \sum_{i=1}^{t-1}\E_{d_h^{(i)}}\left[\left(\gT_hf_{h+1}^{(t)} - f_h^{(t)}\right)^2\right]}{N_{\off}\E_{\mu}\left[\left(\gT_hf_{h+1}^{(t)} - f_h^{(t)}\right)^2\right] + \sum_{i=1}^{t-1}\E_{d_h^{(i)}}\left[\left(\gT_hf_{h+1}^{(t)} - f_h^{(t)}\right)^2\right]}\right)^{1/2} \nonumber\\
    &\leq \sqrt{\sum_{t=1}^T\sum_{h=1}^H \frac{\E_{\pi^*}\left[\gT_hf_{h+1}^{(t)} - f_h^{(t)}\right]^2}{N_{\off}\E_{\mu}\left[\left(\gT_hf_{h+1}^{(t)} - f_h^{(t)}\right)^2\right] + \sum_{i=1}^{t-1}\E_{d_h^{(i)}}\left[\left(\gT_hf_{h+1}^{(t)} - f_h^{(t)}\right)^2\right]}}\nonumber\\
    &\qquad \cdot\sqrt{\sum_{t=1}^T\sum_{h=1}^H \left(N_{\off}\E_{\mu}\left[\left(\gT_hf_{h+1}^{(t)} - f_h^{(t)}\right)^2\right] + \sum_{i=1}^{t-1}\E_{d_h^{(i)}}\left[\left(\gT_hf_{h+1}^{(t)} - f_h^{(t)}\right)^2\right]\right)} \nonumber\\
    &\leq \sqrt{\sum_{t=1}^T\sum_{h=1}^H \frac{\E_{\pi^*}\left[\gT_hf_{h+1}^{(t)} - f_h^{(t)}\right]^2}{N_{\off}\E_{\mu}\left[\left(\gT_hf_{h+1}^{(t)} - f_h^{(t)}\right)^2\right]}}\nonumber\\
    &\qquad \cdot\sqrt{\sum_{t=1}^T\sum_{h=1}^H \left(N_{\off}\E_{\mu}\left[\left(\gT_hf_{h+1}^{(t)} - f_h^{(t)}\right)^2\right] + \sum_{i=1}^{t-1}\E_{d_h^{(i)}}\left[\left(\gT_hf_{h+1}^{(t)} - f_h^{(t)}\right)^2\right]\right)} \nonumber\\
    &\leq \sqrt{c_{\off}^*(\gF)HT/N_{\off}}\sqrt{\beta H^3T} \nonumber\\
    &\leq \sqrt{H^4\beta c_{\off}^*(\gF) T^2/N_{\off}},
\end{align}
where for the penultimate line, the bound for the first term follows directly from the Definition \ref{defn:single-policy-concentrability} on single-policy conentrability coefficient, and the second bound follows directly from Lemma \ref{lem:optimism-Q-pi-star}. Note that the argument is similar to that of Theorem \ref{thm:hybrid-nora-regret_bound}, with the exception that we can use the single-policy concentrability coefficient as the density ratio we need to bound is \begin{align}\frac{\E_{\pi^*}\left[\gT_hf_{h+1}^{(t)} - f_h^{(t)}\right]^2}{\E_{\mu}\left[\left(\gT_hf_{h+1}^{(t)} - f_h^{(t)}\right)^2\right]} \leq \max_{h\in[H]}\sup_{f \in \gF} \frac{\E_{\pi^*}\left[f_h-\gT_hf_{h+1} \right]^2}{\E_{\mu}[\left(f_h-\gT_h f_{h+1}\right)^2]} = c_{\off}^*(\gF),\end{align}
where again the first inequality holds as long as $-f \in \gF$ for all $f \in \gF$.

We further note that it is possible to have the critic target $Q^{\pi^{(t)}}$ with this framework, if the sum of truncated critics is a critic. Fortunately, there is no need to deal with a rarely-updating bonus class like in \citep{sherman2024rateoptimalpolicyoptimizationlinear, cassel2024warmupfreepolicyoptimization}, as we do not need optimism.  

\subsection{Proofs for Theorem \ref{thm:hybrid-nora-regret_bound}} 
\label{app:hybrid-rl-nora}

The proof follows analogously to that of the original online case, in line with the analysis and observations of \citet{tan2024natural}. The only difference is that we will now bound the Bellman error under the current policy's occupancy measure as
\begin{align}\sum_{t=1}^T \sum_{h=1}^H \E_{d_h^{(t)}}[\delta_h^{(t)}]  = \sum_{t=1}^T \sum_{h=1}^H \E_{d_h^{(t)}}[\delta_h^{(t)}\mathbbm{1}_{\gX_{\off}}] + \sum_{t=1}^T \sum_{h=1}^H \E_{d_h^{(t)}}[\delta_h^{(t)}\mathbbm{1}_{\gX_{\on}}].\end{align}

For the online term, this is bounded with the same Cauchy-Schwarz and change of measure argument as in \citet{xie2022role}:
\begin{align}
    \sum_{t=1}^T \sum_{h=1}^H \E_{d_h^{(t)}}[\delta_h^{(t)}\mathbbm{1}_{\gX_{\on}}]
    &= \sum_{t=1}^T \sum_{h=1}^H \E_{d_h^{(t)}}[\delta_h^{(t)}\mathbbm{1}_{\gX_{\on}}] \left(\frac{H^2 \vee \sum_{i=1}^{t-1}\E_{d_h^{(i)}}[(\delta_h^{(t)})^2\mathbbm{1}_{\gX_{\on}}]}{H^2 \vee \sum_{i=1}^{t-1}\E_{d_h^{(i)}}[(\delta_h^{(t)})^2\mathbbm{1}_{\gX_{\on}}]}\right)^{1/2}\nonumber \\
    &\leq \sqrt{\sum_{t=1}^T\sum_{h=1}^H \frac{\E_{d_h^{(i)}}[\delta_h^{(t)}\mathbbm{1}_{\gX_{\on}}]^2}{H^2 \vee \sum_{i=1}^{t-1}\E_{d_h^{(i)}}[(\delta_h^{(t)})^2\mathbbm{1}_{\gX_{\on}}]}}\sqrt{\sum_{t=1}^T\sum_{h=1}^H H^2 \vee \sum_{i=1}^{t-1}\E_{d_h^{(i)}}[(\delta_h^{(t)})^2\mathbbm{1}_{\gX_{\on}}]}\nonumber \\
    &\leq \sqrt{H\mathsf{SEC}(\gF_{\on}, \Pi, T)}\sqrt{\beta H^3 T}\nonumber\\
    &\leq \sqrt{\beta H^4T \mathsf{SEC}(\gF_{\on}, \Pi, T)}.
\end{align}
Within the third-last inequality, the first term can be bounded by $H$ times the SEC of \citet{xie2022role}, almost by definition of the SEC. The second term is bounded by Lemma \ref{lem:ac-loss-bounded-optimism}.

The offline term is bounded by the offline data. We perform a similar Cauchy-Schwarz and change of measure argument as \citet{tan2024natural} to see that:
\begin{align}
    \sum_{t=1}^T \sum_{h=1}^H \E_{d_h^{(t)}}[\delta_h^{(t)}\mathbbm{1}_{\gX_{\off}}]
    &= \sum_{t=1}^T \sum_{h=1}^H \E_{d_h^{(t)}}[\delta_h^{(t)}\mathbbm{1}_{\gX_{\off}}] \left(\frac{N_{\off}\E_{\mu}[(\delta_h^{(t)})^2\mathbbm{1}_{\gX_{\off}}] + \sum_{i=1}^{t-1}\E_{d_h^{(i)}}[(\delta_h^{(t)})^2\mathbbm{1}_{\gX_{\off}}]}{N_{\off}\E_{\mu}[(\delta_h^{(t)})^2\mathbbm{1}_{\gX_{\off}}] + \sum_{i=1}^{t-1}\E_{d_h^{(i)}}[(\delta_h^{(t)})^2\mathbbm{1}_{\gX_{\off}}]}\right)^{1/2}\nonumber \\
    &\leq \sqrt{\sum_{t=1}^T\sum_{h=1}^H \frac{\E_{d_h^{(i)}}[\delta_h^{(t)}\mathbbm{1}_{\gX_{\off}}]^2}{N_{\off}\E_{\mu}[(\delta_h^{(t)})^2\mathbbm{1}_{\gX_{\off}}] + \sum_{i=1}^{t-1}\E_{d_h^{(i)}}[(\delta_h^{(t)})^2\mathbbm{1}_{\gX_{\off}}]}}\nonumber\\
    &\qquad \cdot\sqrt{\sum_{t=1}^T\sum_{h=1}^H \left(N_{\off}\E_{\mu}[(\delta_h^{(t)})^2\mathbbm{1}_{\gX_{\off}}] + \sum_{i=1}^{t-1}\E_{d_h^{(i)}}[(\delta_h^{(t)})^2\mathbbm{1}_{\gX_{\off}}]\right)} \nonumber\\
    &\leq \sqrt{\sum_{t=1}^T\sum_{h=1}^H \frac{\E_{d_h^{(i)}}[\delta_h^{(t)}\mathbbm{1}_{\gX_{\off}}]^2}{N_{\off}\E_{\mu}[(\delta_h^{(t)})^2\mathbbm{1}_{\gX_{\off}}]}}\nonumber\\
    &\qquad \cdot\sqrt{\sum_{t=1}^T\sum_{h=1}^H \left(N_{\off}\E_{\mu}[(\delta_h^{(t)})^2\mathbbm{1}_{\gX_{\off}}] + \sum_{i=1}^{t-1}\E_{d_h^{(i)}}[(\delta_h^{(t)})^2\mathbbm{1}_{\gX_{\off}}]\right)}\nonumber \\
    &\leq \sqrt{c_{\off}(\gF \mathbbm{1}_{\gX_{\off}})HT / N_{\off}}\sqrt{\beta H^3 T}\nonumber\\
    &\leq \sqrt{\beta H^4 c_{\off}(\gF \mathbbm{1}_{\gX_{\off}})T^2/N_{\off}}.
\end{align}
Note that the first term of the penultimate line follows directly from Definition \ref{defn:partial-all-policy-concentrability} on the partial all-policy concentrability, and the second term term follows directly from Lemma \ref{lem:optimism-Q-pi-star}.

\subsection{Auxiliary lemmas}
We also require the following helper lemma:

\begin{lem}[Optimism and in-sample error control for critics targeting $Q^{\pi^{(t)}}$ in hybrid RL]
\label{lem:optimism-Q-pi-t}
With probability at least $1-\delta$, for all $t \in\left[T\right]$, we have that for all $h=1, \ldots, H$, a critic targeting $Q^{\pi^{(t)}}$ in the same way as defined in Algorithm \ref{alg:NOAH-pi} achieves
\begin{align*}
N_{\off}\mathbb{E}_{\mu_h}\left[\left(f_h^{(t)} - \gT_h f_{h+1}^{(t)}\right)^2\right] + \sum_{i=1}^{t-1} \mathbb{E}_{d^{\pi^{(i)}}}\left[\left(f_h^{(t)} - \gT_h f_{h+1}^{(t)}\right)^2\right] \leq O\left(H^2 \beta\right)
\end{align*}
by choosing $\beta=c_1\left(\log [NH \mathcal{N}_{\mathcal{F},\left(\mathcal{T}^{\Pi}\right)^T \mathcal{F}}(\rho) / \delta]+N \rho\right)$ for some constant $c_1$, where $N = N_{\off} + T$.
\end{lem}
\begin{proof}
    The proof is analogous to that of Lemma 1 in \citet{tan2024natural}. We apply property (ii) of Lemma \ref{lem:ac-loss-bounded-optimism-pi} in the following way: we append a sequence of functions generated from offline samples $1,...,N_{\off}$ to the start of the sequence of $T$ online samples.
    
    Let $f^{(t)}$ be a sequence of critics in $\gF$, defined as follows. Arrange the offline samples in any order. For each $n \in [N_{\off}]$, define $f^{(n)}$ to be any function in the confidence sets constructed by the first $n$ offline episodes.
    
    Now for each $t = N_{\off}+1, ..., N$, define $f^{(t)} := f^{(t - N_{\off})} \in \gF^{(t)}$. As Lemma \ref{lem:ac-loss-bounded-optimism-pi} shows that property (ii) holds for all $\tau \in [N]$, it must also hold for all $\tau = N_{\off}+1, \dots, N$.
\end{proof}

\begin{lem}[Optimism and in-sample error control for critics targeting $Q^{*}$ in hybrid RL]\label{lem:optimism-Q-pi-star}
With probability at least $1-\delta$, for all $t \in\left[T\right]$, we have that for all $h=1, \ldots, H$, a critic targeting $Q^*$ in the same way as defined in Algorithm \ref{alg:NORA-long} achieves
\begin{align*}
&\text{(i) $\gT_h^{\pi^{(t)}} f_{h+1}^{{(t)}} \in \mathcal{F}_h^{(t)}$, and $f_h^{(t)} \geq \gT_h^{\pi^{(t)}}f_{h+1}^{(t)}$,} \\
&\text{(ii) $N_{\off}\mathbb{E}_{\mu_h}\left[\left(f_h^{(t)} - \gT_h f_{h+1}^{(t)}\right)^2\right] + \sum_{i=1}^{t-1} \mathbb{E}_{d^{\pi^{(i)}}}\left[\left(f_h^{(t)} - \gT_h f_{h+1}^{(t)}\right)^2\right] \leq O\left(H^2 \beta\right)$,}
\end{align*}
by choosing $\beta=c_1\left(\log \left[N H \mathcal{N}_{\gF,\gT \gF}(\rho) / \delta\right]+N \rho\right)$ for some constant $c_1$, where $N = N_{\off} + T$. If the critic is non-optimistic as in Algorithm \ref{alg:NOAH-star}, only (ii) holds.
\label{lem:ac-loss-bounded-optimism-hybrid}
\end{lem}
\begin{proof}
    The proof is analogous to that of Lemma 1 in \citet{tan2024natural}. We apply Lemma \ref{lem:ac-loss-bounded-optimism} in the following way: we append a sequence of functions generated from offline samples $1,...,N_{\off}$ to the start of the sequence of $T$ online samples.
    
    Let $f^{(t)}$ be a sequence of critics in $\gF$, defined as follows. Arrange the offline samples in any order. For each $n \in [N_{\off}]$, define $f^{(n)}$ to be any function in the confidence sets constructed by the first $n$ offline episodes.
    
    Now for each $t = N_{\off}+1, ..., N$, define $f^{(t)} := f^{(t - N_{\off})} \in \gF^{(t)}$. As Lemma \ref{lem:ac-loss-bounded-optimism} shows that (i) and (ii) hold for all $\tau \in [N]$, they must also hold for all $\tau = N_{\off}+1, \dots, N$.
\end{proof}
\section{Concentration of the Empirical Loss}

\begin{lem}[Modified Lemma G.1, \citet{xiong2023generalframeworksequentialdecisionmaking}]
    For any $f \in \gF$, $\pi \in \Pi$, $h \in [H]$ and $t_1 \leq t_2 \in [T]$, we have with probability $1-\delta$ that
    $$0 \leq \gL_h^{(t_1:t_2, \pi)}(\gT_h^{\pi} f_{h+1}, f_{h+1}) - \min_{f_h' \in \gF_h} \gL_h^{(t_1:t_2, \pi)}(f_h', f_{h+1}) \leq H^2\beta,$$
    where $\beta=c_1\left(\log \left[N H \mathcal{N}_{\gF,\gG}(\rho) / \delta\right]+N \rho\right)$ for some constant $c_1$. Note that $\gG = \gT \gF$ if $\pi$ is the greedy policy with respect to $f$, and $\gG = (\gT^\Pi)^T \gF$ otherwise.
    \label{lem:g1-ac-loss-bounded}
\end{lem}

\begin{proof}
    The proof is for the most part similar to the proof of Lemma G.1 for \citet{xiong2023generalframeworksequentialdecisionmaking}, which is in turn an analogue of Lemma 40 in \citet{jin2021bellman}. However, we take into account the fact that the Bellman operator we are concerned is the Bellman operator for policy $\pi$. That is, we are concerned with $\gT^{\pi}$, not $\gT$. For completeness and thoroughness, we provide the full proof below. 

    Write $\ell_h^{(i)}(f,g,\pi) = g_h(s_h^{(i)}, a_h^{(i)}) - r_h(s_h^{(i)}, a_h^{(i)}) - f_{h+1}(s_{h+1}^{(i)}, \pi_{h+1}(s_{h+1}^{(i)}))$ for the TD error at timestep $h$ and trajectory $i$ for policy $\pi$, and $\delta_h^{(f,g,\pi)}(s_h^{(i)}, a_h^{(i)}) = \E_{s_{h+1}^{(i)}\sim d^{\pi^{(i)}}}[\ell_h^{(i)}(f,g,\pi)]$ for its expectation over $s_{h+1}^{(i)}\sim d^{\pi^{(i)}}$, known as the Bellman error between $f$ and $g$.

    Now write $X_h^{(i)}(f,g,\pi) = \ell_h^{(i)}(f,g, \pi)^2 - \ell_h^{(i)}(f,\gT^{\pi} f, \pi)^2$ for the difference in the squared TD error between $f_h$ and $g_h$ and the squared TD error between $\gT_h^{\pi} f_{h+1}$ and $f_h$. Note that this is bounded by $O(H^2)$. We can then show that
    \begin{align}
        &\E_{s_{h+1}^{(i)} \sim d^{\pi^{(i)}}}[X_h^{(i)}(f,g,\pi)] \nonumber\\
        &= \E_{s_{h+1}^{(i)} \sim d^{\pi^{(i)}}}[(\ell_h^{(i)}(f,g, \pi)-\ell_h^{(i)}(f,\gT_h^{\pi} f, \pi))\cdot (\ell_h^{(i)}(f,g, \pi)+\ell_h^{(i)}(f,\gT_h^{\pi} f, \pi))]\nonumber\\
        &= \E_{s_{h+1}^{(i)} \sim d^{\pi^{(i)}}}[\E_{s_{h+1}^{(i)} \sim d^{\pi^{(i)}}}[\ell_h^{(i)}(f,g, \pi)]\cdot \ell_h^{(i)}(f,g, \pi)] \nonumber\\
        &= \delta_h^{(f,g,\pi)}(s_h^{(i)}, a_h^{(i)})^2,
    \end{align}
    where the second-last equality follows from noting that \citep{chen2022generalframeworksampleefficientfunction}
    \begin{align}
        \ell_h^{(i)}(f,g, \pi) - \E_{s_{h+1}^{(i)} \sim d^{\pi^{(i)}}}[\ell_h^{(i)}(f,g, \pi)] 
        &= g_h(s_h^{(i)}, a_h^{(i)}) - r_h(s_h^{(i)}, a_h^{(i)}) - f_{h+1}(s_{h+1}^{(i)}, \pi_{h+1}(s_{h+1}^{(i)})) \nonumber\\
        &\qquad- g_h(s_h^{(i)}, a_h^{(i)}) + \gT_h^{\pi} f_{h+1}(s_h^{(i)}, a_h^{(i)})\nonumber\\
        &= \gT_h^{\pi} f_{h+1}(s_h^{(i)}, a_h^{(i)}) - r_h(s_h^{(i)}, a_h^{(i)}) - f_{h+1}(s_{h+1}^{(i)}, \pi_{h+1}(s_{h+1}^{(i)}))\nonumber\\
        &= \ell_h^{(i)}(f,\gT_h^{\pi} f, \pi).
    \end{align}

    We use this property again in the fourth equality below to show that
    \begin{align}
        \E_{s_{h+1}^{(i)}\sim d^{\pi^{(i)}}}[X_h^{(i)}(f,g,\pi)^2] 
        &= \E_{s_{h+1}^{(i)} \sim d^{\pi^{(i)}}}[(\ell_h^{(i)}(f,g, \pi)-\ell_h^{(i)}(f,\gT_h^{\pi} f, \pi))^2\cdot (\ell_h^{(i)}(f,g, \pi)+\ell_h^{(i)}(f,\gT_h^{\pi} f, \pi))^2] \nonumber\\
        &\leq 4H^2\E_{s_{h+1}^{(i)} \sim d^{\pi^{(i)}}}[(\ell_h^{(i)}(f,g, \pi)-\ell_h^{(i)}(f,\gT_h^{\pi} f, \pi))^2]\nonumber\\
        &= 4H^2 \delta_h^{(f,g,\pi)}(s_h^{(i)}, a_h^{(i)})^2 \nonumber\\
        &= \E_{s_{h+1}^{(i)} \sim d^{\pi^{(i)}}}[X_h^{(i)}(f,g,\pi)].
    \end{align}

    We can then apply Freedman's inequality \citep{jin2021bellman, chen2022generalframeworksampleefficientfunction} and a union bound over the value function class $\gV$ to see that for any $f \in \gF$ and $g \in \gG$,
    \begin{align}
        &\left|\sum_{i=t_1}^{t_2} X_h^{(i)}(f,g,\pi) - \sum_{i=t_1}^{t_2} \E_{s_{h+1}^{(i)} \sim d^{\pi^{(i)}}}[X_h^{(i)}(f,g,\pi)]\right| \lesssim H\sqrt{\iota\sum_{i=t_1}^{t_2} \E_{s_{h+1}^{(i)} \sim d^{\pi^{(i)}}}[X_h^{(i)}(f,g,\pi)]} + H^2\iota
    \end{align}
    holds with probability at least $1-\delta$, where $\iota = \log(HT^2\gN_{\gF, \gG}(1/T)/\delta)$. 

    The result now holds by observing that 
    $$\E_{s_{h+1}^{(i)} \sim d^{\pi^{(i)}}}[X_h^{(i)}(f,g,\pi)] \geq 0$$
    and that
    $$\sum_{i=t_1}^{t_2} X_h^{(i)}(f,g,\pi) \leq O(H^2\iota + H) \leq H^2\beta$$
    for any $f \in \gF$ and $g \in \gG$ with probability at least $1-\delta$. 
    
\end{proof}

\begin{lem}[Modified Lemma G.2, \citet{xiong2023generalframeworksequentialdecisionmaking}]

    Let $t_1 \leq t_2 \in [T]$, $\pi \in \Pi$, and $f \in \gF$. If it holds with probability at least $1-2\delta$ that
    $$\mathcal{L}_h^{\left(t_1:t_2,\pi\right)}\left(f_h, f_{h+1}\right) - \mathcal{L}_h^{\left(t_1:t_2, \pi\right)}\left(\gT_h^{\pi}f_{h+1}, f_{h+1}\right) \leq CH^2\beta,$$
    then it also holds with probability at least $1-2\delta$ that:
    \begin{align*}
        \sum_{i=t_1}^{t_2}\delta_h^{(f,f,\pi)}(s_h^{(i)}, a_h^{(i)})^2 \leq C_1H^2\beta, \qquad
        \sum_{i=t_1}^{t_2} \E_{d_h^{\pi^{(i)}}}\left[\left(\delta_h^{(f,f,\pi)}\right)^2\right] \leq C_1H^2\beta,
    \end{align*}
    where $C_1 = C+1$ if $C \in [1,100]$, or $2C$ for all constants $C\geq 2$, and  $\beta=c_1\left(\log \left[N H \mathcal{N}_{\gF,\gG}(\rho) / \delta\right]+N \rho\right)$ for some constant $c_1$. Note that $\gG = \gT \gF$ if $\pi$ is the greedy policy with respect to $f$, and $\gG = (\gT^\Pi)^T \gF$ otherwise.

\label{lem:g2-ac-loss-concentration-ub}
\end{lem}
\begin{proof}
    The proof is similar to that of Lemma G.2 in \citet{xiong2023generalframeworksequentialdecisionmaking}. By the same argument as in Lemma \ref{lem:g1-ac-loss-bounded}, we apply Freedman's inequality \citep{jin2021bellman, chen2022generalframeworksampleefficientfunction} and a union bound over a $1/T$-net (or $1/N$ net in the hybrid case) of $(\gF, \gG)$ to see that for any $f \in \gF$ and $g \in \gG$,
    \begin{align}
        &\left|\sum_{i=t_1}^{t_2} X_h^{(i)}(f,g,\pi) - \sum_{i=t_1}^{t_2} \E_{s_{h+1}^{(i)} \sim d^{\pi^{(i)}}}[X_h^{(i)}(f,g,\pi)]\right| \lesssim H\sqrt{\iota\sum_{i=t_1}^{t_2} \E_{s_{h+1}^{(i)} \sim d^{\pi^{(i)}}}[X_h^{(i)}(f,g,\pi)]} + H^2\iota
    \end{align}
    holds with probability at least $1-\delta$. 

    By assumption, we have that
    \begin{align}\sum_{i=t_1}^{t_2} X_h^{(i)}(f,f,\pi) = \sum_{i=t_1}^{t_2} \left(\ell_h^{(i)}(f,f,\pi)^2 - \ell_h^{(i)}(f,\gT_h^\pi f, \pi)^2\right) \leq C H^2 \beta.\end{align}
    
    If $C \in [1,100]$ and therefore is a relatively small bounded constant, we can choose $c_1$ large enough in the definition of $\beta$ so that there is some $\widetilde{f}, \widetilde{g}$ in the covering of $(\gF, \gG)$ such that
    \begin{align}\left|\sum_{i=t_1}^{t_2} X_h^{(i)}(f,f,\pi) - \sum_{i=t_1}^{t_2} X_h^{(i)}(\widetilde f,\widetilde g,\pi)\right| \leq O(H), \text{ and } \sum_{i=t_1}^{t_2} X_h^{(i)}(\widetilde f, \widetilde g,\pi) \leq CH^2\beta + O(H),\end{align}
    and consequently that
    $$\sum_{i=t_1}^{t_2} X_h^{(i)}(\widetilde f,\widetilde g,\pi) \leq (C+1/2)H^2\beta, \text{ and } \sum_{i=t_1}^{t_2} \E_{s_{h+1}^{(i)} \sim d^{\pi^{(i)}}}[X_h^{(i)}(f,f,\pi)] \leq (C+1/2)H^2\beta + O(H) \leq (C+1)H^2\beta.$$

    As in \citet{xiong2023generalframeworksequentialdecisionmaking} and \citet{jin2021bellman}, the argument for the case where $C > 100$ follows by accepting a $2$-approximation where $C_1 = 2C$, and the argument for $\sum_{i=1}^{t-1} \E_{d_h^{\pi^{(i)}}}\left[\left(\delta_h^{(t)}\right)^2\right]$ also follows analogously by taking expectations.

    Note that taking $t_1 = 1, t_2 = t-1, \pi = \pi^{(t-1)}, f = f^{(t)}$ or $t_1 = 1, t_2 = t-1, \pi = \pi^f, f = f^{(t)}$ gives results of direct importance to us. 
\end{proof}

\begin{lem}[Modified Lemma G.3, \citet{xiong2023generalframeworksequentialdecisionmaking}]
    Let $t_1 \leq t_2 \in [T]$, $\pi \in \Pi$, and $f \in \gF$. If it holds with probability at least $1-2\delta$ that
    $$\mathcal{L}_h^{\left(t_1:t_2,\pi\right)}\left(f_h, f_{h+1}\right) - \mathcal{L}_h^{\left(t_1:t_2, \pi\right)}\left(\gT_h^{\pi}f_{h+1}, f_{h+1}\right) \geq CH^2\beta,$$
    then it also holds with probability at least $1-2\delta$ that:
    \begin{align*}
        \sum_{i=t_1}^{t_2}\delta_h^{(f,f,\pi)}(s_h^{(i)}, a_h^{(i)})^2 \geq C_1H^2\beta, \quad
        \sum_{i=t_1}^{t_2} \E_{d_h^{\pi^{(i)}}}\left[\left(\delta_h^{(f,f,\pi)}\right)^2\right] \geq C_1H^2\beta,
    \end{align*}
    where $C_1 = C+1$ if $C \in [1,100]$, or $2C$ for all constants $C\geq 2$. , and  $\beta=c_1\left(\log \left[N H \mathcal{N}_{\gF,\gG}(\rho) / \delta\right]+N \rho\right)$ for some constant $c_1$. Note that $\gG = \gT \gF$ if $\pi$ is the greedy policy with respect to $f$, and $\gG = (\gT^\Pi)^T \gF$ otherwise.

\label{lem:g3-ac-loss-concentration-ub}
\end{lem}
\begin{proof}
    The proof is similar to that of Lemma G.3 in \citet{xiong2023generalframeworksequentialdecisionmaking}. By the same argument as in Lemma \ref{lem:g1-ac-loss-bounded}, we apply Freedman's inequality \citep{jin2021bellman, chen2022generalframeworksampleefficientfunction} and a union bound over a $1/T$-net (or $1/N$ net in the hybrid case) of $(\gF, \gG)$ to see that for any $f \in \gF$ and $g \in \gT^{\pi}\gF$,
    \begin{align}
        &\left|\sum_{i=t_1}^{t_2} X_h^{(i)}(f,g,\pi) - \sum_{i=t_1}^{t_2} \E_{s_{h+1}^{(i)} \sim d^{\pi^{(i)}}}[X_h^{(i)}(f,g,\pi)]\right| \lesssim H\sqrt{\iota\sum_{i=t_1}^{t_2} \E_{s_{h+1}^{(i)} \sim d^{\pi^{(i)}}}[X_h^{(i)}(f,g,\pi)]} + H^2\iota
    \end{align}
    holds with probability at least $1-\delta$. 

    By assumption, we have that
    $$\sum_{i=t_1}^{t_2} X_h^{(i)}(f,f,\pi) = \sum_{i=t_1}^{t_2} \left(\ell_h^{(i)}(f,f,\pi)^2 - \ell_h^{(i)}(f,\gT_h^\pi f, \pi)^2\right) \geq C H^2 \beta.$$
    
    If $C \in [1,100]$ and therefore is a relatively small bounded constant, we can choose $c_1$ large enough in the definition of $\beta$ so that there is some $\widetilde{f}, \widetilde{g}$ in the covering of $(\gF, \gG)$ such that
    \begin{align}\left|\sum_{i=t_1}^{t_2} X_h^{(i)}(f,f,\pi) - \sum_{i=t_1}^{t_2} X_h^{(i)}(\widetilde f,\widetilde g,\pi)\right| \leq O(H), \text{ and } \sum_{i=t_1}^{t_2} X_h^{(i)}(\widetilde f, \widetilde g,\pi) \geq CH^2\beta - O(H),\end{align}
    and consequently that
    \begin{align}\sum_{i=t_1}^{t_2} X_h^{(i)}(\widetilde f,\widetilde g,\pi) \geq (C-1/2)H^2\beta, \text{ and } \sum_{i=t_1}^{t_2} \E_{s_{h+1}^{(i)} \sim d^{\pi^{(i)}}}[X_h^{(i)}(f,f,\pi)] \geq (C-1/2)\beta - O(H)\geq (C-1)H^2\beta.\end{align}

    As in \citet{xiong2023generalframeworksequentialdecisionmaking} and \citet{jin2021bellman}, the argument for the case where $C > 100$ follows by accepting a $2$-approximation where $C_1 = C/2$, and the argument for $\sum_{i=1}^{t-1} \E_{d_h^{\pi^{(i)}}}\left[\left(\delta_h^{(t)}\right)^2\right]$ also follows analogously by taking expectations.

    Note that taking $t_1 = 1, t_2 = t-1, \pi = \pi^{(t-1)}, f = f^{(t)}$ or $t_1 = 1, t_2 = t-1, \pi = \pi^f, f = f^{(t)}$ gives results of direct importance to us. 
\end{proof}

\section{What If We Target $Q^{\pi^{(t)}}$ Instead of $Q^*$?}
\label{app:problem-target-q-pi}

\begin{algorithm}[h]
    \caption{NORA-$\pi$}
    \begin{algorithmic}[1]
            \STATE {\bfseries Input:} Offline dataset $\gD_{\off}$, samples sizes $T$, $N_{\off}$, function class $\gF$ and confidence width $\beta > 0$ 
            \STATE {\bfseries Initialize:} $\gF^{(0)} \leftarrow \gF$, $\gD_h^{(0)}\leftarrow \emptyset, \forall h \in [H]$, $\eta = \Theta(\sqrt{\log|\gA| H^{-2}T^{-1}})$, $\pi^{(1)} \; \propto \; 1$.
            \FOR{episode $t = 1, 2, \dots, T$}
            \STATE Select critic $f_h^{(t)}(s,a) :=\operatorname{argmax}_{f_h \in \mathcal{F}_h^{(t_{\text{last}})}} f_h\left(s,a\right)$ for all $s,a$.
            \STATE Play policy $\pi^{(t)}$ for one episode and obtain trajectory $(s_1^{(t)}, a_1^{(t)}, r_1^{(t)}), \dots, (s_H^{(t)}, a_H^{(t)}, r_H^{(t)})$.
                \STATE Update dataset $\mathcal{D}_h^{(t)} \leftarrow \mathcal{D}_h^{(t-1)} \cup\{(s_h^{(t)}, a_h^{(t)}, r_h^{(t)}, s_{h+1}^{(t)})\}, \forall h \in[H]$. 
            \IF{there exists some $h$ such that $\mathcal{L}_h^{(t, \pi^{(t)})}(f_h^{(t)}, f_{h+1}^{(t)})-\min _{f_h^{\prime} \in \mathcal{F}_h} \mathcal{L}_h^{(t, \pi^{(t)})}(f_h^{\prime}, f_{h+1}^{(t)}) \geq 5H^2\beta$} 
            \STATE Compute confidence set for Bellman operator $\gT^{\pi^{(t)}}$, set $\gF^{(t)} \gets \gF^{(t,\pi^{(t)})}$.
            $$
                \mathcal{F}^{(t, \pi^{(t)})} \leftarrow\left\{f \in \mathcal{F}: \mathcal{L}_h^{(t, \pi^{(t)})}\left(f_h, f_{h+1}\right)-\min _{f_h^{\prime} \in \mathcal{F}_h} \mathcal{L}_h^{(t, \pi^{(t)})}\left(f_h^{\prime}, f_{h+1}\right) \leq H^2\beta \quad \forall h \in[H]\right\},
            $$
            $$ \qquad \text{ where } \mathcal{L}_h^{(t, \pi^{(t)})}\left(f, f'\right):=\sum_{\left(s, a, r, s^{\prime}\right) \in \mathcal{D}_h^{(t)} \cup \mathcal{D}_{\off, h}}\left(f(s, a)-r-f^{\prime}(s^{\prime}, \pi^{(t)}_{h+1}(s'))\right)^2, \forall f \in \gF_{h}, f' \in \mathcal{F}_{h+1}.
            $$
            \vspace{-3mm}
            \STATE Set $t_{\text{last}} := t$, increment number of updates $N_{\text{updates}}^{(t)} := N_{\text{updates}}^{(t-1)}+ 1$.
            \ELSE
            \STATE Set $N_{\text{updates}}^{(t)} := N_{\text{updates}}^{(t-1)}$, $\gF^{(t)} := \gF^{(t-1)}$.
            \ENDIF
           
        \STATE Select policy $\pi_h^{(t+1)} \; \propto \; \pi_h^{(t)}\exp(\eta f_h^{(t)})$.
        \ENDFOR
    \end{algorithmic}
\label{alg:NORA-pi-long}
\end{algorithm}

\subsection{But can we bound the number of critic updates?}

    We attempt to show a similar result to Lemma \ref{lem:switch-cost} for the Q-function confidence set $\gF^{(t, \pi^{(t)})}$ targeting $\pi^{(t)}$. However, we will later see that we run into an issue.

    Fix some $h \in [H]$ for now. For simplicity, write $K_h = N_{\text{updates},h}(T)$ for the total number of updates induced by updating the Q-function class targeting $\pi^{(t)}$, and $t_{1,h},...,t_{K_h,h}$ the update times for $f_h^{(t)}$, with $t_{0,h} = 0$. By definition, at every $t_{k,h}$, we have
    \begin{align}\mathcal{L}_h^{(t_{k,h}, \pi^{(t_{k,h})})}\left(f_h^{(t_{k,h})}, f_{h+1}^{(t_{k,h})}\right)-\min _{f_h^{\prime} \in \mathcal{F}_h} \mathcal{L}_h^{(t_{k,h}, \pi^{(t_{k,h})})}\left(f_h^{\prime}, f_{h+1}^{(t_{k,h})}\right) \geq 5 H^2 \beta\end{align}

    An application of Lemma \ref{lem:g1-ac-loss-bounded} yields 
    $$0 \leq \mathcal{L}_h^{\left(t_{k,h}, \pi^{(t_{k,h})}\right)}\left(\mathcal{T}_h^{\pi^{(t_{k,h})}} f_{h+1}^{(t_{k,h})}, f_{h+1}^{(t_{k,h})}\right)-\min _{f_h^{\prime} \in \mathcal{F}_h} \mathcal{L}_h^{\left(t_{k,h}, \pi^{(t_{k,h})}\right)}\left(f_h^{\prime}, f_{h+1}^{(t_{k,h})}\right) \leq H^2 \beta,$$
    $$\min _{f_h^{\prime} \in \mathcal{F}_h} \mathcal{L}_h^{\left(t_{k,h}, \pi^{(t_{k,h})}\right)}\left(f_h^{\prime}, f_{h+1}^{(t_{k,h})}\right)  \geq \mathcal{L}_h^{\left(t_{k,h}, \pi^{(t_{k,h})}\right)}\left(\mathcal{T}_h^{\pi^{(\pi^{(t_{k,h})})}} f_{h+1}^{(t_{k,h})}, f_{h+1}^{(t_{k,h})}\right)- H^2 \beta,$$
    $$\mathcal{L}_h^{(t_{k,h}, \pi^{(t_{k,h})})}\left(f_h^{(t_{k,h})}, f_{h+1}^{(t_{k,h})}\right)-\mathcal{L}_h^{\left(t_{k,h}, \pi^{(t_{k,h})}\right)}\left(\mathcal{T}_h^{\pi^{(\pi^{(t_{k,h})})}} f_{h+1}^{(t_{k,h})}, f_{h+1}^{(t_{k,h})}\right) \geq 4H^2\beta.
    $$
    From the above, we can now establish that
    \begin{align}
        &\mathcal{L}_h^{(t_{k-1,h}+1:t_{k,h}, \pi^{(t_{k,h})})}\left(f_h^{(t_{k,h})}, f_{h+1}^{(t_{k,h})}\right)-\mathcal{L}_h^{\left(t_{k-1,h}+1:t_{k,h}, \pi^{(t_{k,h})}\right)}\left(\mathcal{T}_h^{\pi^{(t_{k,h})}} f_{h+1}^{(t_{k,h})}, f_{h+1}^{(t_{k,h})}\right) \nonumber\\
        &=\mathcal{L}_h^{(t_{k-1,h}+1:t_{k,h}, \pi^{(t_{k,h})})}\left(f_h^{(t_{k-1,h}+1)}, f_{h+1}^{(t_{k-1,h}+1)}\right)-\mathcal{L}_h^{\left(t_{k-1,h}+1:t_{k,h}, \pi^{(t_{k,h})}\right)}\left(\mathcal{T}_h^{\pi^{(t_{k,h})}} f_{h+1}^{(t_{k-1,h}+1)}, f_{h+1}^{(t_{k-1,h}+1)}\right) \nonumber\\
        &= \mathcal{L}_h^{(t_{k,h}, \pi^{(t_{k,h})})}\left(f_h^{(t_{k-1,h}+1)}, f_{h+1}^{(t_{k-1,h}+1)}\right)-\mathcal{L}_h^{\left(t_{k,h}, \pi^{(t_{k,h})}\right)}\left(\mathcal{T}_h^{\pi^{(t_{k,h})}} f_{h+1}^{(t_{k-1,h}+1)}, f_{h+1}^{(t_{k-1,h}+1)}\right) \nonumber \\
        &\qquad- \left(\mathcal{L}_h^{(t_{k-1,h}, \pi^{(t_{k,h})})}\left(f_h^{(t_{k-1,h}+1)}, f_{h+1}^{(t_{k-1,h}+1)}\right) +\mathcal{L}_h^{\left(t_{k-1,h}, \pi^{(t_{k,h})}\right)}\left(\mathcal{T}_h^{\pi^{(t_{k,h})}} f_{h+1}^{(t_{k-1,h}+1)}, f_{h+1}^{(t_{k-1,h}+1)}\right)\right) \nonumber\\
        &= \mathcal{L}_h^{(t_{k,h}, \pi^{(t_{k,h})})}\left(f_h^{(t_{k,h})}, f_{h+1}^{(t_{k,h})}\right)-\mathcal{L}_h^{\left(t_{k,h}, \pi^{(t_{k,h})}\right)}\left(\mathcal{T}_h^{\pi^{(t_{k,h})}} f_{h+1}^{(t_{k,h})}, f_{h+1}^{(t_{k,h})}\right) \nonumber\\
        &\qquad- \left(\mathcal{L}_h^{(t_{k-1,h}, \pi^{(t_{k,h})})}\left(f_h^{(t_{k-1,h}+1)}, f_{h+1}^{(t_{k-1,h}+1)}\right) + \mathcal{L}_h^{\left(t_{k-1,h}, \pi^{(t_{k,h})}\right)}\left(\mathcal{T}_h^{\pi^{(t_{k,h})}} f_{h+1}^{(t_{k-1,h}+1)}, f_{h+1}^{(t_{k-1,h}+1)}\right)\right) \nonumber\\
        &\geq 4H^2\beta - H^2\beta = 3H^2\beta
    \end{align}
    We substitute $f_h^{(t_{k,h})}$ for $f_h^{(t_{k-1,h}+1)}$ in the second equality, and obtain the last inequality via the inequality established in the previous argument and an application of Lemma \ref{lem:g1-ac-loss-bounded} on $t_2 = t_{k-1,h}, \pi = \pi^{(t_{k,h})}, f = f^{(t_{k-1,h}+1)}$.

    Therefore, for any $t$ such that $t_{k-1,h} < t \leq t_{k,h}$, this argument and noting that $f_h^{(t_{k-1,h}+1)} = ... = f_h^{(t)} = ... = f_h^{(t_{k,h})}$ yields
    \begin{align}\mathcal{L}_h^{(t_{k-1,h}+1:t_{k,h}, \pi^{(t_{k,h})})}\left(f_h^{(t)}, f_{h+1}^{(t)}\right)-\mathcal{L}_h^{\left(t_{k-1,h}+1:t_{k,h}, \pi^{(t_{k,h})}\right)}\left(\mathcal{T}_h^{\pi^{(t_{k,h})}} f_{h+1}^{(t)}, f_{h+1}^{(t)}\right) \geq 3H^2\beta.
    \end{align}
    An application of Lemma \ref{lem:g3-ac-loss-concentration-ub} while noting that $f_h^{(t_{k-1,h}+1)} = ... = f_h^{(t)} = ... = f_h^{(t_{k,h})}$ yields
    \begin{align}\sum_{i=t_{k-1,h}+1}^{t_{k,h}}\left(f_h^{(i)} - \gT_h^{\pi^{(t_{k,h})}}f_{h+1}^{(i)}\right)^2(s_h^{(i)}, a_h^{(i)}) = \sum_{i=t_{k-1,h}+1}^{t_{k,h}}\left(f_h^{(t_{k,h})} - \gT_h^{\pi^{(t_{k,h})}}f_{h+1}^{(t_{k,h})}\right)^2(s_h^{(i)}, a_h^{(i)}) \geq 2H^2\beta.\end{align}
    Summing over all $t_{1,h},...,t_{K,h}$ yields 
    \begin{align} \sum_{t=1}^{T} \left(f_h^{(t)} - \gT_h^{\pi^{(t_{\text{next}})}} f_{h+1}^{(t)}\right)^2(s_h^{(t)}, a_h^{(t)}) = \sum_{k=1}^{K_h} \sum_{i=t_{k-1,h}+1}^{t_{k,h}} \left(f_h^{(i)} - \gT_h^{\pi^{(t_{k,h})}}f_{h+1}^{(i)}\right)^2(s_h^{(i)}, a_h^{(i)}) \geq 2(K_h-1)H^2\beta.\end{align}

    By Lemma \ref{lem:ac-loss-bounded-optimism}, we have that 
    \begin{align}\sum_{i=1}^{t-1} \left(f_h^{(t)} - \gT_h^{\pi^{(t-1)}}f_{h+1}^{(t)}\right)^2(s_h^{(i)}, a_h^{(i)}) \leq O(H^2\beta).\end{align}
    Invoking the squared distributional Bellman eluder dimension definition yields
    \begin{align}\sum_{t=1}^{T} \left(f_h^{(t)} - \gT_h^{\pi^{(t-1)}}f_{h+1}^{(t)}\right)^2(s_h^{(t)}, a_h^{(t)}) \leq O(dH^2\beta\log T).\end{align}

    So we have established that 
    $$2(K_h-1)H^2\beta \leq \sum_{t=1}^{T} \left(f_h^{(t)} - \gT_h^{\pi^{(t_{\text{next}})}}f_{h+1}^{(t)}\right)^2(s_h^{(t)}, a_h^{(t)}),$$
    $$\text{and }  \sum_{t=1}^{T} \left(f_h^{(t)} - \gT_h^{\pi^{(t-1)}}f_{h+1}^{(t)}\right)^2(s_h^{(t)}, a_h^{(t)}) \leq O(dH^2\log T).$$

    However, it remains unclear how one can relate 
    $$\sum_{t=1}^{T} \left(f_h^{(t)} - \gT_h^{\pi^{(t_{\text{next}})}}f_{h+1}^{(t)}\right)^2(s_h^{(t)}, a_h^{(t)})$$
    $$\text{to }\sum_{t=1}^{T} \left(f_h^{(t)} - \gT_h^{\pi^{(t-1)}}f_{h+1}^{(t)}\right)^2(s_h^{(t)}, a_h^{(t)}).$$

    We would like the former to be no greater than the latter, but that does not necessarily hold, as $\pi^{(t-1)}$ is closer to the target by which $\gF^{(t)}$ was constructed, $\pi^{(t_{\last})}$, than $\pi^{(t_{\text{next}})}$. So if anything, it is likely that the Bellman error under the Bellman operator for $\pi^{(t_{\text{next}})}$ is greater than that for $\pi^{(t-1)}$. It is therefore difficult to say anything with regard to the number of updates for each $h$.

\subsection{But can we control the negative Bellman error?}

There is another obstacle. It is unclear how to control the negative Bellman error under the occupancy measure of the optimal policy, given the far more limited form of optimism in \ref{lem:ac-loss-bounded-optimism-pi}. This is because the more limited form of optimism only allows us to show:
\begin{align}-\sum_{t=1}^T \sum_{h=1}^H\mathbb{E}_{\pi^*}\left[f_h^{(t)} - \gT_h^{\pi^{(t)}}f_{h+1}^{(t)}\right] \leq  \sum_{t=1}^{T} \sum_{h=1}^H  \mathbb{E}_{\pi^*}\left[ \left\langle f_{h+1}^{(t)}(s', \cdot), \pi_{h+1}^{(t)}(\cdot | s') - \pi_{h+1}^{(t_{\text{last}})}(\cdot | s')\right\rangle\right].\end{align}

To see this, observe that by Lemma \ref{lem:ac-loss-bounded-optimism-pi}, $\mathcal{T}_h^{\pi^{\left(t_{\text {last}}\right)}} f_{h+1}^{\left(t\right)}(s,a) \leq f_h^{(t)}(s, a)$. Therefore,
\begin{align}
    -\sum_{t=1}^T \sum_{h=1}^H\mathbb{E}_{\pi^*}\left[f_h^{(t)} - \gT^{\pi^{(t)}}f_{h+1}^{(t)}\right] 
    &= \sum_{t=1}^T \sum_{h=1}^H\mathbb{E}_{\pi^*}\left[ r_h(s,a) + f_{h+1}^{(t)}(s', \pi_{h+1}^{(t)}(s')) - f_h^{(t)}(s,a )\right] \nonumber\\
    &= \sum_{t=1}^T \sum_{h=1}^H\mathbb{E}_{\pi^*}\left[ \mathcal{T}_h^{\pi^{\left(t\right)}} f_{h+1}^{\left(t\right)}(s,a) - f_h^{(t)}(s,a )\right]\nonumber\\
    &\leq \sum_{t=1}^T \sum_{h=1}^H\mathbb{E}_{\pi^*}\left[ \mathcal{T}_h^{\pi^{\left(t\right)}} f_{h+1}^{\left(t\right)}(s,a) - \gT_h^{\pi^{(t_{\text{last}})}}f_{h+1}^{(t)}(s,a )\right]\nonumber\\
    &= \sum_{t=1}^{T} \sum_{h=1}^H \mathbb{E}_{\pi^*}\left[ r_h(s,a)  + f_{h+1}^{\left(t\right)}(s',\pi_{h+1}^{\left(t\right)}(s')) - r_h(s,a) - f_{h+1}^{(t)}(s',\pi_{h+1}^{(t_{\text{last}})}(s')) \right] \nonumber\\
    &= \sum_{t=1}^{T}\sum_{h=1}^H \mathbb{E}_{\pi^*}\left[ f_{h+1}^{\left(t\right)}(s',\pi_{h+1}^{\left(t\right)}(s')) - f_{h+1}^{(t)}(s',\pi_{h+1}^{(t_{\text{last}})}(s')) \right] \nonumber\\
    &= \sum_{t=1}^{T} \sum_{h=1}^H  \mathbb{E}_{\pi^*}\left[ \left\langle f_{h+1}^{(t)}(s', \cdot), \pi_{h+1}^{(t)}(\cdot | s') - \pi_{h+1}^{(t_{\text{last}})}(\cdot | s')\right\rangle\right].
\end{align}


We can now continue to go through the mirror descent argument. Recall that
$$\pi_{h+1}^{(t+1)}(\cdot | s') = \frac{\pi_{h+1}^{(t)}(\cdot | s') \exp(\eta f_{h+1}^{(t)}(s', \cdot))}{\sum_{a \in \gA} \pi_{h+1}^{(t)}(a | s') \exp(\eta f_{h+1}^{(t)}(s', a))} = Z^{-1}\pi_{h+1}^{(t)}(\cdot | s') \exp(\eta f_{h+1}^{(t)}(s', \cdot)),$$
and rearranging this yields
$$\eta f_{h+1}^{(t)}(s', \cdot) = \log Z_t + \log \pi_{h+1}^{(t+1)}(\cdot | s') - \log \pi_{h+1}^{(t)}(\cdot | s'),$$
where $\log Z_t$ is
$$\log Z_t = \log\left(\sum_{a \in \gA} \pi_{h+1}^{(t)}(a | s') \exp(\eta f_{h+1}^{(t)}(s', a))\right) = \log \pi_{h+1}^{(t)}(\cdot | s') - \log \pi_{h+1}^{(t+1)}(\cdot | s') + \eta f_{h+1}^{(t)}(s', \cdot).$$

Noting that $\sum_{a \in \gA} \left(\pi_{h+1}^{(t)}(\cdot | s') - \pi_{h+1}^{(t_{\text{last}+1})}(\cdot | s')\right) = 0$, we can now bound that
\begin{align}
    &\left\langle \eta f_{h+1}^{(t)}(s', \cdot), \pi_{h+1}^{(t)}(\cdot | s') - \pi_{h+1}^{(t_{\text{last}}+1)}(\cdot | s')\right\rangle \nonumber\\
    &= \left\langle \log Z_t + \log \pi_{h+1}^{(t+1)}(\cdot | s') - \log \pi_{h+1}^{(t)}(\cdot | s'), \pi_{h+1}^{(t)}(\cdot | s') - \pi_{h+1}^{(t_{\text{last}}+1)}(\cdot | s')\right\rangle \nonumber\\ 
    &= \left\langle \log \pi_{h+1}^{(t+1)}(\cdot | s') - \log \pi_{h+1}^{(t)}(\cdot | s'), \pi_{h+1}^{(t)}(\cdot | s') - \pi_{h+1}^{(t_{\text{last}}+1)}(\cdot | s')\right\rangle \nonumber\\
    &= -\text{KL}\left(\pi_{h+1}^{(t)}(\cdot | s')\; || \;\pi_{h+1}^{(t+1)}(\cdot | s')\right) + \text{KL}\left(\pi_{h+1}^{(t_{\text{last}}+1)}(\cdot | s')\; ||\; \pi_{h+1}^{(t+1)}(\cdot | s')\right) - \text{KL}\left(\pi_{h+1}^{(t_{\text{last}}+1)}(\cdot | s')\; ||\; \pi_{h+1}^{(t)}(\cdot | s')\right),
\end{align}
where the last equality follows directly from Lemma~\ref{lem: KL}. This establishes the following telescoping sum
\begin{align}
    &\sum_{t=1}^T\sum_{h=1}^H\mathbb{E}_{\pi^*}\left[ \left\langle f_{h+1}^{(t)}(s', \cdot), \pi_{h+1}^{(t)}(\cdot | s') - \pi_{h+1}^{(t_{\text{last}}+1)}(\cdot | s')\right\rangle\right] \nonumber\\
    &= \frac{1}{\eta}\sum_{t=1}^T\sum_{h=1}^H\mathbb{E}_{\pi^*}\left[ -\text{KL}\left(\pi_{h+1}^{(t)}(\cdot | s')\; || \;\pi_{h+1}^{(t+1)}(\cdot | s')\right) + \text{KL}\left(\pi_{h+1}^{(t_{\text{last}}+1)}(\cdot | s')\; ||\; \pi_{h+1}^{(t+1)}(\cdot | s')\right) - \text{KL}\left(\pi_{h+1}^{(t_{\text{last}}+1)}(\cdot | s')\; ||\; \pi_{h+1}^{(t)}(\cdot | s')\right)\right] \nonumber\\
    &= -\frac{1}{\eta}\sum_{t=1}^T\sum_{h=1}^H\mathbb{E}_{\pi^*}\left[ \text{KL}\left(\pi_{h+1}^{(t)}(\cdot | s')\; || \;\pi_{h+1}^{(t+1)}(\cdot | s')\right) \right]\nonumber\\
    &\qquad+ \frac{1}{\eta}\sum_{h=1}^H\sum_{k=1}^{K_h} \sum_{t = t_k}^{t_{k+1}-1} \mathbb{E}_{\pi^*}\left[ \text{KL}\left(\pi_{h+1}^{(t_k+1)}(\cdot | s')\; ||\; \pi_{h+1}^{(t+1)}(\cdot | s')\right) - \text{KL}\left(\pi_{h+1}^{(t_k+1)}(\cdot | s')\; ||\; \pi_{h+1}^{(t)}(\cdot | s')\right)\right]\nonumber\\
    &= -\frac{1}{\eta}\sum_{t=1}^T\sum_{h=1}^H\mathbb{E}_{\pi^*}\left[ \text{KL}\left(\pi_{h+1}^{(t)}(\cdot | s')\; || \;\pi_{h+1}^{(t+1)}(\cdot | s')\right) \right] + \frac{1}{\eta}\sum_{h=1}^H\sum_{k=1}^{K_h} \mathbb{E}_{\pi^*}\left[ \text{KL}\left(\pi_{h+1}^{(t_k+1)}(\cdot | s')\; ||\; \pi_{h+1}^{(t_{k+1})}(\cdot | s')\right)\right].
    \end{align}
    To see this step, consider the example where we perform switches at step 1, 3, 6, and $T=8$. Note that we adopt the convention that $\pi^{(T+1)} = \pi^{(T)}$. The telescoping sum then becomes
    \begin{align}
        &\sum_{t=1}^8 \text{KL}(t_{\text{last}}+1 || t+1) - \text{KL}(t_{\text{last}}+1 || t) \nonumber\\
        &= \text{KL}(2, 2) - \text{KL}(2,1) + \text{KL}(2,3) - \text{KL}(2,2) 
        + \text{KL}(4,4) - \text{KL}(4, 3) 
        + \text{KL}(4,5) - \text{KL}(4,4) \nonumber\\
        &\qquad + \text{KL}(4,6) - \text{KL}(4,5) + \text{KL}(7,7) - \text{KL}(7,6) + \text{KL}(7,8) - \text{KL}(7,7) + \text{KL}(7,9) - \text{KL}(7,8)\nonumber\\
        &= \text{KL}(2,3)  + \text{KL}(4,6) + \text{KL}(7,9) .
    \end{align}
    So the sum includes every $t_k$ where a switch occurs:
    \begin{align*}
    &\sum_{t=1}^T\sum_{h=1}^H\mathbb{E}_{\pi^*}\left[ \left\langle f_{h+1}^{(t)}(s', \cdot), \pi_{h+1}^{(t)}(\cdot | s') - \pi_{h+1}^{(t_{\text{last}})}(\cdot | s')\right\rangle\right] \\
    &= \frac{1}{\eta}\sum_{h=1}^H\sum_{k=1}^{K_h} \mathbb{E}_{\pi^*}\left[ \text{KL}\left(\pi_{h+1}^{(t_k+1)}(\cdot | s')\; ||\; \pi_{h+1}^{(t_{k+1})}(\cdot | s')\right) \right] -\frac{1}{\eta}\sum_{t=1}^T\sum_{h=1}^H\mathbb{E}_{\pi^*}\left[ \text{KL}\left(\pi_{h+1}^{(t)}(\cdot | s')\; || \;\pi_{h+1}^{(t+1)}(\cdot | s')\right) \right] \\
    &= \gH^*(\pi_h^{(t)}, t_k). 
\end{align*}

The latter two terms cancel in the TV distance, or any distance where the triangle inequality holds. It is harder to see a relation with the KL divergence, but in general, one may not be able to achieve sublinear regret. 

This is because, if we merge this term with the first term in the regret decomposition of Lemma \ref{lem:regret-decomp-ac}, we obtain
$$\sum_{t=1}^T\sum_{h=1}^H\mathbb{E}_{\pi^*}\left[ \left\langle f_{h+1}^{(t)}(s', \cdot), \pi_{h+1}^*(\cdot | s') - \pi_{h+1}^{(t)}(\cdot | s')\right\rangle\right] + \sum_{t=1}^T\sum_{h=1}^H\mathbb{E}_{\pi^*}\left[ \left\langle f_{h+1}^{(t)}(s', \cdot), \pi_{h+1}^{(t)}(\cdot | s') - \pi_{h+1}^{(t_{\text{last}})}(\cdot | s')\right\rangle\right],$$
which evaluates to 
$$\sum_{t=1}^T\sum_{h=1}^H\mathbb{E}_{\pi^*}\left[ \left\langle f_{h+1}^{(t)}(s', \cdot), \pi_{h+1}^*(\cdot | s') - \pi_{h+1}^{(t_{\last})}(\cdot | s')\right\rangle\right].$$

\section{Miscellaneous Lemmas}

In this section, we collect some auxiliary lemmas that are useful in deriving our main results. 

\begin{lem}[Bound on Covering Number of Value Function Class, Lemma B.1 from \citet{zhong2023theoreticalanalysisoptimisticproximal}]
    Consider the value function class induced by a Q-function class $\gF^{(t)}$ and a class of stochastic policies $\Pi^{(t)}$, given by
    $$\gV_h^{(t)} = \left\{\left\langle f_h(s, \cdot), \pi_h(\cdot | s)\right\rangle \mid f_h \in \gF_h^{(t)}, \pi \in \Pi^{(t)}\right\}.$$
    Then, the covering number of the value function class can be bounded by the product of the covering number of its components:
    $$\gN_{\gV_h^{(t)}}(\rho) \leq \gN_{\gF_h^{(t)}}(\rho/2)\cdot \gN_{\Pi_h^{(t)}}(\rho/2H).$$
    \label{lem:b1-covering-number-value}
\end{lem}

\begin{lem}[Lemma B.3, \citet{zhong2023theoreticalanalysisoptimisticproximal}]
For $\pi, \pi^{\prime} \in \Delta(\mathcal{A})$ and $Q, Q^{\prime}: \mathcal{A} \mapsto \mathbb{R}^{+}$, if $\pi(\cdot) \propto \exp (Q(\cdot))$ and $\pi^{\prime}(\cdot) \propto \exp \left(Q^{\prime}(\cdot)\right)$, we have

$$
\left\|\pi-\pi^{\prime}\right\|_1 \leq \sqrt{2 \cdot \text{KL}(\pi || \pi')}\leq 2 \sqrt{\left\|Q-Q^{\prime}\right\|_{\infty}} .
$$
    \label{lem:b3-zhong-policy-q-ineq}
\end{lem}

\begin{lem}[Adapted Version of Lemma D.2 of \citet{xiong2023generalframeworksequentialdecisionmaking}]
    Let $\gF$ be a function class with low $D_\Delta$-type Bellman Eluder dimension. Then, for any policy $\pi \in \Pi$, if we have that $\sum_{i=1}^{t-1} (f_h^{(t)} - \gT_h^\pi f_{h+1}^{(t)})^2(s_{h}^{(i)}, a_h^{(i)}) \leq \beta H^2$ for any $t\in [T]$ and $\beta \geq 9$, then for any $t' \in [T]$ we also have that
    $$\sum_{i=1}^{t} (f_h^{(i)} - \gT_h^\pi f_{h+1}^{(i)})^2\left(s_h^{(i)}, a_h^{(i)}\right) \leq O\left(d_{\text{BE}}(\gF, D_\Delta, 1/\sqrt{T})\beta H^2\log T\right).$$
\end{lem}

\begin{lem}[Value Difference/Generalized Policy Difference Lemma, \citep{cai2024provablyefficientexplorationpolicy, efroni2020optimisticpolicyoptimizationbandit}]
    Let $\pi,\pi^{\prime}$ be two policies and $f \in \gF$ be any Q-function. Then for any $t \in[T]$ we have
$$
\begin{aligned}
& f_1\left(s_1, \pi_1(s_1)\right)-V_1^{\pi^{\prime}}\left(s_1\right) \\
& \quad=\sum_{h=1}^H \mathbb{E}_{\pi^{\prime}}\left[\left\langle f_h\left(s_h, \cdot\right), \pi_h\left(\cdot \mid s_h\right)-\pi_h^{\prime}\left(\cdot \mid s_h\right)\right\rangle\right]+\sum_{h=1}^H \mathbb{E}_{\pi^{\prime}}\left[f_h\left(s_h, a_h\right)-\gT_h^{\pi'}f_{h+1}\left(s_h, a_h\right)\right] .
\end{aligned}
$$
\label{lem:policy-value-difference-lemma}
\end{lem}

\begin{lem}\label{lem: KL}
For any probability distributions $\pi(\cdot), \pi_1(\cdot)$ and $\pi_2(\cdot)$ over space $\mathcal{S}$. We have following relationship holds:
\begin{align*}
\big\langle\pi_1(\cdot) - \pi(\cdot), \log\pi(\cdot)-\log\pi_2(\cdot)\big\rangle =  -\mathrm{KL}(\pi_1||\pi) +\mathrm{KL}(\pi_1||\pi_2) - \mathrm{KL}(\pi||\pi_2).
\end{align*}
\end{lem}
\begin{proof}
  Note that for the first equality, we have
  \begin{align*}
  \big\langle\pi_1(\cdot) - \pi(\cdot), \log\pi(\cdot)-\log\pi_2(\cdot)\big\rangle &= \big\langle\pi_1(\cdot), \log\pi(\cdot)-\log\pi_2(\cdot)\big\rangle - \big\langle \pi(\cdot), \log\pi(\cdot)-\log\pi_2(\cdot)\big\rangle\\
  &=\big\langle\pi_1(\cdot), \log\pi(\cdot)-\log\pi_1(\cdot)\big\rangle + \big\langle\pi_1(\cdot), \log\pi_1(\cdot)-\log\pi_2(\cdot)\big\rangle\\
  &\qquad - \big\langle \pi(\cdot), \log\pi(\cdot)-\log\pi_2(\cdot)\big\rangle\\
  & = -\mathrm{KL}(\pi_1||\pi)+\mathrm{KL}(\pi_1||\pi_2) - \mathrm{KL}(\pi||\pi_2).
  \end{align*}
\end{proof}

\begin{lem}[Policy Optimization Difference]
Let $\pi^{(t)}, t=1,...,T$ be a sequence of policies updated by:
$$\pi_{h+1}^{(t+1)}\left(\cdot \mid s^{\prime}\right)=\frac{\pi_{h+1}^{(t)}\left(\cdot \mid s^{\prime}\right) \exp \left(\eta f_{h+1}^{(t)}\left(s^{\prime}, \cdot\right)\right)}{\sum_{a \in \mathcal{A}} \pi_{h+1}^{(t)}\left(a \mid s^{\prime}\right) \exp \left(\eta f_{h+1}^{(t)}\left(s^{\prime}, a\right)\right)}=Z_t^{-1} \pi_{h+1}^{(t)}\left(\cdot \mid s^{\prime}\right) \exp \left(\eta f_{h+1}^{(t)}\left(s^{\prime}, \cdot\right)\right),$$
where $f^{(t)} \in \gF$. For any $t_1, t_2 \in [T]$, where $\mu^{(t)}$ is an arbitrary set of distributions,
$$\sum_{t=\min\{t_1,t_2\}}^{\max\{t_1,t_2\}} \mathbb{E}_{\mu^{(t)}}\left[\left\langle f_{h+1}^{(t)}(s',\cdot), \pi^{(t_2)}_{h+1}(\cdot|s') - \pi^{(t_1)}_{h+1}(\cdot|s')\right\rangle\right] \leq \eta H^2 (|t_2 - t_1|+1).$$
\label{lem:policy-difference}
\end{lem}

\begin{proof}
    We observe that
    \begin{align*}
        \frac{\pi_{h+1}^{(t+1)}(\cdot \mid s^{\prime})}{\pi_{h+1}^{(t)}(\cdot \mid s^{\prime})}&\geq \frac{\exp (0)}{|\gA|^{-1}\sum_{a \in \mathcal{A}} \exp (\eta H)}\geq \exp(-\eta H),\\
        \frac{\pi_{h+1}^{(t+1)}(\cdot \mid s^{\prime})}{\pi_{h+1}^{(t)}(\cdot \mid s^{\prime})}&=\frac{\exp (\eta f_{h+1}^{(t)}(s^{\prime}, \cdot))}{\sum_{a \in \mathcal{A}} \pi_{h+1}^{(t)}(a \mid s^{\prime})\exp (\eta f_{h+1}^{(t)}(s^{\prime}, a))} \leq \frac{\exp (\eta H)}{|\gA|^{-1}\sum_{a \in \mathcal{A}} \exp (0)}=\exp (\eta H).
    \end{align*}
    So we can establish that since $e^{-x} \geq 1-x$ for all $x \in \R$,
    \begin{align}
        &|\pi_{h+1}^{(t+1)}(\cdot | s') - \pi_{h+1}^{(t)}(\cdot | s')| 
        = \pi_{h+1}^{(t+1)}(\cdot | s') \cdot \left|1-\frac{\pi_{h+1}^{(t)}(\cdot | s')}{\pi_{h+1}^{(t+1)}(\cdot | s')}\right| \leq \pi_{h+1}^{(t+1)}(\cdot | s') \left(1-\exp(-\eta H)\right)\nonumber \\
        &\qquad \leq \pi_{h+1}^{(t+1)}(\cdot | s') \left(1-(1-\eta H)\right) \leq \eta H \pi_{h+1}^{(t+1)}(\cdot | s').
    \end{align}
    By the triangle inequality and the fact that $f \leq H$ for all $f \in \gF$, it then follows that
    \begin{align}
        &\sum_{t=\min\{t_1,t_2\}}^{\max\{t_1,t_2\}} \mathbb{E}_{\mu^{(t)}}\left[\left\langle f_{h+1}^{(t)}(s',\cdot), \pi^{(t_2)}_{h+1}(\cdot|s') - \pi^{(t_1)}_{h+1}(\cdot|s')\right\rangle\right] 
        \leq \sum_{t=\min\{t_1,t_2\}}^{\max\{t_1,t_2\}} \mathbb{E}_{\mu^{(t)}}\left[\left\langle H, \eta H \pi_{h+1}^{(t+1)}(\cdot | s')\right\rangle\right]\nonumber\\
        &\qquad \leq \sum_{t=\min\{t_1,t_2\}}^{\max\{t_1,t_2\}} \eta H^2\leq 2 \eta H^2 (|t_2 - t_1|+1).
    \end{align}
\end{proof}
\section{Further Experiment Details}
\label{app:experiment_details}

Figure \ref{fig:cumulative-regret} can be reproduced by running \texttt{actor\_critic.ipynb} within the following GitHub repository (\href{https://github.com/hetankevin/hybridcov}{https://github.com/hetankevin/hybridcov}). 
Figure \ref{fig:average-return} can be reproduced by running \texttt{scripts/run\_antmaze.sh} within the following GitHub repository (\href{https://github.com/nakamotoo/Cal-QL}{https://github.com/nakamotoo/Cal-QL}). The results for Cal-QL arise from running the script as-is. Algorithm 2H can be reproduced by adding the flags \texttt{--enable\_calql=False}, \texttt{--use\_cql=False}, and \texttt{--online\_use\_cql=False}. Algorithm 1H can be reproduced with the same flags as Algorithm 2H, but additionally setting the \texttt{config.cql\_max\_target\_backup} argument within the \texttt{ConservativeSAC()} object to False. 

To implement the mirror descent update in Figure \ref{fig:cumulative-regret}, we store the sequence of past Q-functions fitted for Algorithm \ref{alg:DOUHUA}, and the last Q-function for Algorithm \ref{alg:NORA}. Upon receiving a query to evaluate or sample from $\pi_h^{(t)}(s,a)$ for a given $h,s,a$ tuple, we compute $\exp(\eta \sum_{t=1}^T f_h^{(t)}(s,a))$ for Algorithm \ref{alg:DOUHUA}, and $\exp(\eta (t-t_{\last}) f_h^{(t)}(s,a))$ for Algorithm \ref{alg:NORA}. To generate a sample from this density, the normalizing constant can be computed exactly in the case where there is a finite number of actions. Otherwise a sample can be generated via MCMC, importance sampling, or rejection sampling.


\bibliography{ref}
\bibliographystyle{apalike}

\end{document}

%% file: macros.tex
\usepackage{amsmath,bbm,bm}
\usepackage{amsfonts}
\usepackage{amsthm}
\usepackage{mathtools}

\usepackage{tabularx}

\newtheorem{thm}{Theorem}
\newtheorem{lem}{Lemma}

\newtheorem{cor}{Corollary}

\newtheorem{aspt}{Assumption}

\newtheorem{defn}{Definition}

\usepackage{xcolor}
\newcount\comments  
\newcommand{\genComment}[2]{\ifnum\comments=1{\textcolor{#1}{\textsf{\footnotesize #2}}}\fi}

\def\eqref#1{equation~\ref{#1}}

\def\1{\bm{1}}

\DeclareMathAlphabet{\mathsfit}{\encodingdefault}{\sfdefault}{m}{sl}
\SetMathAlphabet{\mathsfit}{bold}{\encodingdefault}{\sfdefault}{bx}{n}

\def\gA{{\mathcal{A}}}

\def\gD{{\mathcal{D}}}

\def\gF{{\mathcal{F}}}
\def\gG{{\mathcal{G}}}
\def\gH{{\mathcal{H}}}

\def\gL{{\mathcal{L}}}

\def\gN{{\mathcal{N}}}
\def\gO{{\mathcal{O}}}

\def\gS{{\mathcal{S}}}
\def\gT{{\mathcal{T}}}

\def\gV{{\mathcal{V}}}

\def\gX{{\mathcal{X}}}

\def\prob{{\mathbb{P}}}

\def\R{{\mathbb{R}}}

\newcommand{\E}{\mathbb{E}}

\newcommand{\last}{\text{last}}

\DeclareMathOperator*{\argmax}{arg\,max}
\DeclareMathOperator*{\argmin}{arg\,min}

\newcommand{\off}{\operatorname{off}}
\newcommand{\on}{\operatorname{on}}